\pdfoutput=1
\documentclass{article}
\usepackage{arxiv}
\renewcommand{\shorttitle}{Kernel Autoresearch for Open-Ended Model Discovery}
\renewcommand{\headeright}{}
\renewcommand{\undertitle}{}
\usepackage[utf8]{inputenc}
\usepackage[T1]{fontenc}
\usepackage{amsmath,amssymb,amsthm}
\usepackage{booktabs}
\usepackage{array}
\usepackage{enumitem}
\usepackage{xspace}
\usepackage{algorithm}
\usepackage{algorithmic}
\makeatletter
\newcommand{\FUNCTION}[1]{\ALC@it\textbf{function} #1\begin{ALC@g}}
\newcommand{\ENDFUNCTION}{\end{ALC@g}\ALC@it\textbf{end function}}
\makeatother
\usepackage{graphicx}
\usepackage{fontawesome5}
\usepackage{pgfplots}
\pgfplotsset{compat=1.17}
\usetikzlibrary{positioning,arrows.meta}
\usepackage{microtype}
\usepackage{colortbl}
\usepackage{subcaption}
\usepackage{float}
\usepackage{fvextra}

\fvset{breaksymbolleft={},breaksymbolright={},breakindent=0pt,breakautoindent=false}
\usepackage[most]{tcolorbox}
\usepackage{listings}
\definecolor{oiVerm}{HTML}{D55E00}
\definecolor{oiOrange}{HTML}{E69F00}
\definecolor{oiSky}{HTML}{56B4E9}
\definecolor{oiGreen}{HTML}{009E73}
\definecolor{oiPurple}{HTML}{CC79A7}
\definecolor{oiBlue}{HTML}{0072B2}
\definecolor{promptbg}{gray}{0.97}
\definecolor{codebg}{gray}{0.97}
\definecolor{codegray}{gray}{0.45}
\lstdefinestyle{kcode}{
  language=Python, basicstyle=\fontencoding{T1}\fontfamily{lmtt}\selectfont\footnotesize, columns=fullflexible, keepspaces=true,
  keywordstyle=\bfseries, commentstyle=\color{codegray}\itshape,
  emph={float,abs,int,len,range,sum,min,max}, emphstyle=\itshape,
  morekeywords={as,with,lambda,yield}, identifierstyle=\relax,
  stringstyle=\relax, showstringspaces=false, breaklines=true,
  numbers=left, numberstyle=\tiny\color{codegray}, numbersep=7pt, xleftmargin=1.6em,
  upquote=true,
  literate={é}{{\'e}}1 {—}{{---}}1 {"}{{\char34}}1,
}
\lstdefinelanguage{kprompt}{
  morecomment=[l]{\#}, morestring=[b]`, moredelim=[s][\itshape]{<}{>},
}
\lstdefinestyle{kprompt}{
  language=kprompt, basicstyle=\ttfamily\fontsize{7.6}{9.4}\selectfont, columns=fullflexible, keepspaces=true,
  breaklines=true, breakatwhitespace=true, showstringspaces=false,
  commentstyle=\bfseries, stringstyle=\relax,
  literate={é}{{\'e}}1 {—}{{---}}1 {"}{{\char34}}1,
}
\lstdefinelanguage{json}{
  literate={é}{{\'e}}1 {"}{{\char34}}1,
}
\lstdefinestyle{kjson}{
  language=json, basicstyle=\ttfamily\footnotesize, columns=fullflexible, keepspaces=true,
  showstringspaces=false, breaklines=true,
}
\tcbset{
  aibox/.style={
    width=\textwidth,
    top=10pt, bottom=4pt,
    colback=promptbg, colframe=black!70, colbacktitle=black!85,
    coltitle=white, fonttitle=\bfseries\small,
    boxrule=0.7pt, arc=2.5mm,
    enhanced, breakable, center,
    attach boxed title to top left={yshift=-0.1in,xshift=0.15in},
    boxed title style={boxrule=0pt,colframe=white,arc=1.5mm},
  },
  codebox/.style={
    colback=codebg, colframe=black!55, colbacktitle=black!85, coltitle=white,
    fonttitle=\bfseries\footnotesize, boxrule=0.6pt, arc=1.5mm, enhanced, breakable,
    left=2pt, right=2pt, top=2pt, bottom=2pt,
  },
}
\newtcolorbox{AIbox}[2][]{aibox,title=#2,#1}
\newtcolorbox{CodeBox}[2][]{codebox,title=#2,#1}
\usepackage[normalem]{ulem}
\usepackage{adjustbox}
\usepackage[hyphens]{url}
\usepackage[round]{natbib}

\newtheorem{definition}{Definition}
\newtheorem{proposition}{Proposition}

\newcommand{\method}{\textsc{Kernaut}\xspace}
\newcommand{\dwf}{DWF\xspace}
\newcommand{\cX}{\mathcal{X}}

\newcommand{\paperlink}[3]{\href{#1}{\textcolor{black!75}{#2\hspace{0.3em}#3}}}
\newcolumntype{L}[1]{>{\raggedright\arraybackslash}p{#1}}

\usepackage[hidelinks]{hyperref}
\hypersetup{
  bookmarksdepth=3,
  pdftitle={Kernel Autoresearch for Open-Ended Model Discovery},
  pdfauthor={Richard Cornelius Suwandi, Feng Yin, Kevin Murphy}
}

\title{Kernel Autoresearch for Open-Ended Model Discovery}

\author{
  Richard Cornelius Suwandi \\
  CUHK-Shenzhen \\
  \And
  Feng Yin\\
  CUHK-Shenzhen \\
  \And
  Kevin Murphy \\
  University of British Columbia \\
}

\date{}

\begin{document}
\maketitle
{\let\thefootnote\relax\footnotetext{Preprint. Under review.}}
\vspace{-5.5pt}
{\centering\small
\paperlink{https://github.com/richardcsuwandi/kernaut}{\faGithub}{Code}%
\hspace{1.6em}%
\paperlink{https://richardcsuwandi.github.io/kernaut/}{\faGlobe}{Website}\par}
\vspace{1.15em}
\suppressfloats[t]

\begin{abstract}
Kernels encode the inductive bias of a wide range of machine learning models, yet automated kernel design faces a fundamental dilemma. A fixed grammar of base kernels and operators guarantees validity but limits the search to structures expressible by those building blocks. Conversely, unrestricted programs remove this limitation but no longer guarantee validity. In our stress tests, 22--58\% of LLM-generated kernels that pass numerical checks on random inputs fail when evaluated at different scales or dimensions. We propose Kernel Autoresearch (\method), which treats kernel design as open-ended model discovery. Coding agents write kernels as programs, while construction contracts ensure that every accepted kernel is valid. A quality-diversity archive retains high-performing kernels with distinct behaviors, and novelty screening steers agents toward functionally new candidates. Our experiments demonstrate that the discovered kernels encode reusable inductive biases that generalize to unseen tasks. On held-out black-box optimization families, a discovered kernel outperforms a meta-learned deep kernel trained on the same episodes. Furthermore, kernels discovered from ten enzyme-kinetic rate laws achieve lower error than tuned ARD and deep kernel baselines on five unseen mechanisms. The discovered kernels are also interpretable programs that human researchers can refine: a human-refined version of one further reduces the held-out predictive error by 5.7\% and optimization regret by 7.8\%.
\end{abstract}  


\section{Introduction}
\label{sec:intro}

Kernel choice largely determines the inductive bias of a wide range of machine learning methods, yet most automated kernel design still searches over a narrow space. Most approaches start
from base kernels and combine them through a fixed set of operators, whether
via greedy selection \citep{duvenaud2013structure}, Bayesian optimization \citep[BO;][]{malkomes2016boms}, neural composition \citep{sun2018neuralkernelnetwork}, or search guided by large language models \citep[LLMs;][]{suwandi2025cake}. Despite their different search strategies, these approaches share the same tradeoff. Restricting search to a predefined grammar guarantees validity by construction, but limits search to structures expressible through the base kernels and operators. We instead formulate kernel design as open-ended program synthesis, allowing agents to propose kernels beyond this grammar. However, this broader search space introduces a verification challenge because an arbitrary program does not necessarily define a mathematically valid kernel.

An unrestricted program need not define a symmetric, positive semidefinite (PSD) kernel \citep{rice1953classes}, and numerical checks on sampled inputs cannot establish these properties over the entire input domain. We therefore propose Kernel Autoresearch (\method), in which coding agents write kernel components and a trusted interpreter constructs valid kernels. The interpreter uses PSD-preserving rules that we call construction contracts. The system evaluates the resulting kernels and stores high-performing programs with distinct kernel behavior in a quality-diversity archive. This archive guides later proposals and preserves diverse discoveries for evaluation.
Our contributions are as follows.
\begin{itemize}[leftmargin=*,itemsep=1pt,topsep=2pt]
\item \textbf{Open-ended kernel discovery with PSD-preserving contracts.}
Agents propose executable feature maps, spectral representations, input
transformations, and kernel compositions beyond a fixed grammar. A trusted
interpreter constructs each accepted Gram matrix through one of the PSD-preserving
construction contracts
(\S\ref{sec:formulation}--\ref{sec:tiers}).
\item \textbf{Verified execution with a quality-diversity archive.}
A deterministic backend executes, verifies, tunes, and scores every proposal.
The archive retains programs that combine strong predictive performance with
distinct kernel behavior. The system then uses accepted examples and
structured rejection reasons to guide subsequent proposals
(\S\ref{sec:loop}--\ref{sec:functional-novelty}).
\item \textbf{Reusable inductive biases across unseen tasks.}
We freeze each discovered program before evaluation and reuse it without
structural modification. We evaluate prediction on held-out black-box
optimization (BBO) functions, time series, enzyme mechanisms, and
glucose patients (\S\ref{sec:experiments}).
\item \textbf{Interpretable discoveries across domains.}
Because the kernel programs are short, we can analyze the structures they encode, such as warps and folds for BBO functions, trend and period banks for time series, rate-law shapes for enzymes, and meal and insulin responses for glucose (Appendix~\ref{app:kernel-descriptions}). As a case study, we analyze the dual warp--fold (\dwf{}) kernel and a human-refined variant with lower held-out predictive error (\S\ref{sec:dwf-analysis}).
\end{itemize}

\section{Related Work}
\label{sec:related}

\paragraph{Deep, spectral, and compositional kernel learning.}
Deep kernel learning jointly trains a Gaussian process (GP) with a learned feature map \citep{wilson2016deepkernel}. \citet{wistuba2021fsbo} meta-learn this deep-kernel model across related BO tasks. Spectral mixture methods learn a stationary spectral density \citep{wilson2013spectralmixture}. Other methods combine base kernels through PSD-preserving operators by using greedy search \citep{duvenaud2013structure,lloyd2014automatic}, BO \citep{malkomes2016boms}, 
differentiable composition \citep{sun2018neuralkernelnetwork}, or LLM-guided evolution \citep{suwandi2025cake}. These methods preserve kernel validity by restricting search to a chosen deep-kernel architecture, spectral family, or compositional grammar. 

\paragraph{LLM-generated kernels with empirical checks.}
The closest prior work, \citet{yun2026automated}, uses LLM-driven evolutionary search to generate kernels beyond standard compositions. Their method screens each candidate by attempting a Cholesky factorization on sampled inputs. Passing this check provides evidence only for the sampled Gram matrix, but it does not establish kernel validity over the entire input domain. In our numerical screening experiment, 22--58\% of unrestricted programs passed the initial checks but failed additional checks at different input scales and dimensions (Appendix~\ref{app:screening-stress}). In contrast, \method constructs each accepted kernel through a PSD-preserving contract. Kernel validity therefore follows from the construction under the contract's stated assumptions (Proposition~\ref{prop:validity}).

\paragraph{Coding agents and autoresearch.}
Coding-agent systems often follow a propose--evaluate--revise loop, as in FunSearch and AlphaEvolve for algorithms and mathematical constructions \citep{romeraparedes2024funsearch,novikov2025alphaevolve} and in autoresearch for training code \citep{karpathy2026autoresearch}. Model Discovery Agent \citep{murphy2026mdagent} and Agentic BO \citep{brunzema2026agenticbo} similarly separate agent proposals from numerical evaluation in scientific and optimization workflows. \method applies this approach to kernel discovery, and the backend verifies each program's construction before the archive accepts it.

\begin{figure*}[t]
\centering
\includegraphics[width=\textwidth]{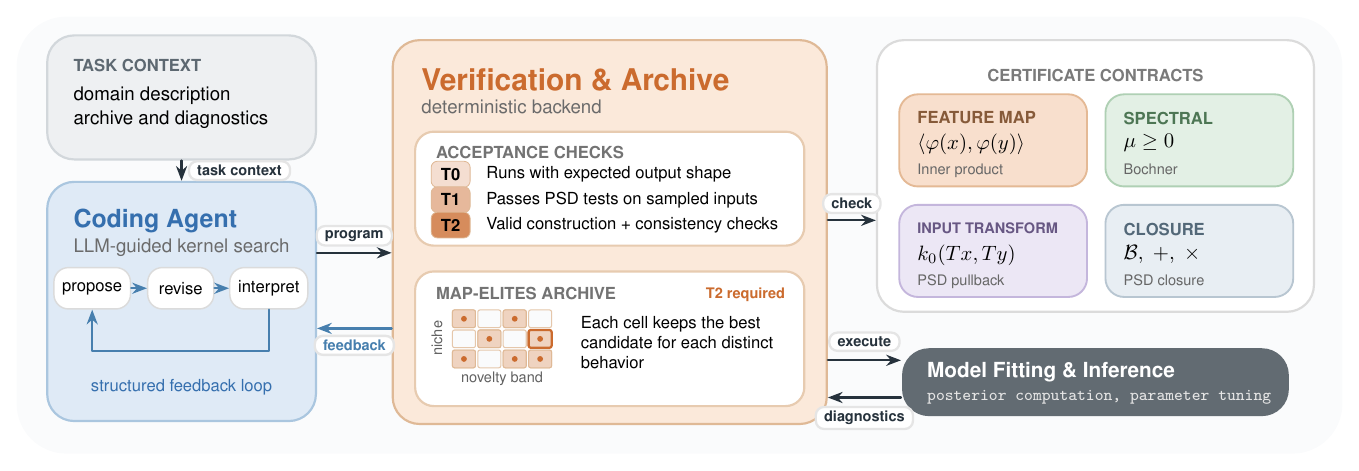}
\caption{\textbf{The \method framework.} The coding agent (blue) writes kernel components under one of four construction contracts. The deterministic backend (orange) applies three acceptance tiers, T0 to T2, and only T2 programs enter the quality-diversity archive, where each cell keeps the best candidate for one behavior. Archived kernels and structured feedback guide later proposals. The grey module computes the posterior and tunes parameters.}
\label{fig:pipeline}
\end{figure*}

\section{Kernel Autoresearch Framework}
\label{sec:framework}

\subsection{Kernels as Programs}
\label{sec:formulation}
\label{sec:contracts}

Let $\cX$ be the input space and $\Theta$ a parameter space. We represent a kernel by an executable program $\kappa:\cX\times\cX\times\Theta\to\mathbb{R}$ and write $\kappa_\theta(x,y)=\kappa(x,y,\theta)$ for fixed parameters. Agents write the components of these programs in Python, and the backend runs them in a sandbox with time and memory limits. The question is which programs define valid kernels.

\begin{definition}[Validity]
\label{def:valid}
A program $\kappa$ is valid if, for every $\theta\in\Theta$, $\kappa_\theta$ is symmetric and every Gram matrix $K$ with entries $K_{ij}=\kappa_\theta(x_i,x_j)$ on finitely many inputs $x_1,\dots,x_n\in\cX$ is positive semidefinite (PSD),
\begin{equation}
  \label{eq:psd}
  c^\top Kc=\sum_{i,j} c_i c_j\,\kappa_\theta(x_i,x_j)\ge 0
  \quad\text{for all } c\in\mathbb{R}^n.
\end{equation}
\end{definition}

For an arbitrary program, validity cannot be guaranteed in general
\citep{rice1953classes}, and numerical tests on sampled inputs cannot cover
the entire domain. We therefore ask the agent to only write a kernel
component. A trusted interpreter then assembles the kernel using a
PSD-preserving rule. We call such component interfaces, together with their assembly rules,
construction contracts.

\begin{definition}[Construction contracts]
\label{def:contract}
For fixed $\theta$, the agent writes one of the following four component types. Each proposed component can expose free parameters in $\theta$, but trusted backend code optimizes those parameters. The interpreter then assembles $\kappa_\theta$ from that component.
\begin{itemize}[leftmargin=*,itemsep=1pt,topsep=2pt,parsep=0pt,before={\normalfont}]
\item \textbf{Feature map}: a function $\varphi_\theta:\cX\to\mathbb{R}^m$, giving $\kappa_\theta(x,y)=\varphi_\theta(x)^\top\varphi_\theta(y)$.
\item \textbf{Spectral}: frequencies $\omega_1,\dots,\omega_M\in\mathbb{R}^d$ and raw weights $r_1,\dots,r_M\in\mathbb{R}$, determined only by the input dimension $d$ and kernel parameters $\theta$, giving $\kappa_\theta(x,y)=\sum_{j=1}^{M} w_j\cos\big(\omega_j^\top(x-y)\big)$ with $w_j=r_j^2/M\ge 0$.
\item \textbf{Pullback}: an input transform $T_\theta:\cX\to\mathbb{R}^m$, giving $\kappa_\theta(x,y)=k_0\big(T_\theta(x),T_\theta(y)\big)$. Here, $k_0$ is a \emph{library kernel}, which is a trusted implementation that the agent can reference but cannot change. Examples include radial basis function (RBF), Mat\'ern, and rational-quadratic kernels.
\item \textbf{Closure}: a tree whose leaves are library kernels or kernels already accepted under the first three contracts, combined by sums, products, and nonnegative scalar multiplication.
\end{itemize}
\end{definition}

The contracts above describe familiar ways to construct a covariance. Feature and spectral programs become explicit inner products, pullbacks apply a valid kernel after changing the input geometry, and closure programs use operations that preserve PSD. The trusted interpreter performs this assembly, rather than accepting a covariance matrix written directly by the agent. 
Appendix~\ref{app:proof} gives the algebra for every case.

\begin{proposition}[PSD preservation by construction]
\label{prop:validity}
For fixed $\theta$, every contract produces a symmetric PSD matrix if its component returns finite outputs with the dimensions its interface requires. If the component is also pure on a domain $\cX_0$, meaning that its output depends only on the input and parameters and not on other batch inputs, random draws, or previous calls, the assembled function is a valid kernel on $\cX_0$. If this holds for every $\theta\in\Theta$, the program is valid.
\end{proposition}

Purity is needed because a batch-dependent component can return PSD matrices on each call without defining one kernel function. The contracts remain expressive, since finite polynomial feature maps uniformly approximate any continuous PSD kernel on a compact set (Appendix~\ref{app:contract-approximation}).

\subsection{Acceptance Checks}
\label{sec:tiers}
The backend checks finiteness, the output dimensions required by each
interface, and purity, which trusted assembly cannot enforce.
Deterministic source inspection rejects constructs that can introduce
batch dependence, randomness, or persistent state. Behavioral tests
check determinism and consistency under batch permutation, subsetting,
extension by new points, and point duplication (Table~\ref{tab:contract-checks}).

Checks proceed in three increasingly expensive tiers
(Table~\ref{tab:tiers}). Tier~0 checks execution and output dimensions.
Tier~1 adds numerical PSD tests on sampled inputs. Tier~2 adds contract
assembly and consistency checks. Only Tier~2 programs can enter the
archive. Appendix~\ref{app:screening-stress} shows why sampled PSD tests
alone are insufficient.

\subsection{Discovery Campaign}
\label{sec:loop}

\paragraph{Campaigns, runs, and the archive.}
A \emph{campaign} is one complete search on one domain. Campaigns have separate archives and can run independently in parallel. Within each campaign, a fixed budget of $B$ \emph{runs} executes sequentially. A run is one fresh coding-agent session that aims to submit one program. Within a run, a \emph{round} is one agent turn of proposing, calling a tool, or revising after feedback. Runs share information only through a persistent \emph{archive}, the collection of every accepted program in the campaign. Each column in Figure~\ref{fig:lineage-trees}a is one campaign, and its markers show accepted programs from its runs. A run either writes a program from scratch, called a \emph{root}, or revises one archived program, called its \emph{parent}. A revision counts only if its source code differs from the parent's, so changing tuned constants alone does not create a descendant.

\paragraph{Niches, cells, and elites.}
We use a multi-dimensional archive of phenotypic elites grid \citep[MAP-Elites;][]{mouret2015mapelites}. A \emph{niche} is a broad label for what a program tries to express (e.g., input geometry, spectral structure, feature projection, additive interaction, local structure). Each accepted program occupies one \emph{cell}, set by its niche, its feature growth with input dimension, and its novelty band (\S\ref{sec:functional-novelty}). The best program in a cell is its \emph{elite}, and a newcomer replaces it only by entering the same cell with a higher score. An elite is a local representative of one behavior. It is not necessarily the best program overall or the one chosen for testing. Figure~\ref{fig:archive-map} shows how the archive fills across successive runs.

\paragraph{One discovery run.}
The scheduler assigns each run a \emph{target niche}, prioritizing
empty niches. With probability $p_{\mathrm{new}}$, the agent writes
a new root. Otherwise, it revises a parent elite. The prompt includes
up to $k_{\mathrm{insp}}$ archived elites as \emph{inspirations},
drawn from other niches when revising a parent. The backend tunes
the program's declared parameters on meta-training episodes using
$N_{\mathrm{tune}}$ Halton settings \citep{halton1960efficiency}
and scores the best setting. Appendix~\ref{app:impl} gives the values of $B$,
$p_{\mathrm{new}}$, $k_{\mathrm{insp}}$, and $N_{\mathrm{tune}}$.

We test an \emph{ensemble} proposer\footnote{The ensemble contains
GLM-5.2, DeepSeek-V4-Pro, and Qwen-3.8 Max.}, which selects one LLM
per run through a Thompson-sampling bandit
(Appendix~\ref{app:metabo}), and a \emph{frontier} proposer, which
uses GPT-6 Astra throughout. After all runs, validation selects one
frozen program from the archive. Algorithm~\ref{alg:loop} summarizes
the campaign. Algorithms~\ref{alg:archive-update}
and~\ref{alg:validation-selection} in Appendix~\ref{app:impl}
detail the agent session and final selection.

\begin{algorithm}[t]
\caption{Overview of \method campaign.}
\label{alg:loop}
\begin{algorithmic}[1]
\STATE \textbf{Input:} run budget $B$, new-root probability $p_{\mathrm{new}}$, number of inspirations $k_{\mathrm{insp}}$, meta-training tasks $D_{\mathrm{tr}}$, meta-validation tasks $D_{\mathrm{val}}$, sealed test tasks $D_{\mathrm{test}}$
\STATE Initialize an empty archive $\mathcal A\gets\emptyset$ of accepted programs
\FOR{$j=1,\ldots,B$}
  \STATE Choose a target niche $\mathrm{niche}^\star$, filling empty niches first
  \STATE With probability $p_{\mathrm{new}}$ set $\mathrm{parent}\gets\varnothing$ to write a new root. Otherwise set $\mathrm{parent}$ to an elite drawn from the archive to revise
  \STATE Sample up to $k_{\mathrm{insp}}$ elites as inspirations $\mathcal I$, from other niches when revising
  \STATE Run one agent session (Algorithm~\ref{alg:archive-update}) to get a program $\kappa$, its search score $s$, and its cell $c$
  \IF{$\kappa\neq\varnothing$}
    \STATE Add $\kappa$ to $\mathcal A$, and make it the elite of cell $c$ if $c$ is empty or $s$ beats the current elite
  \ENDIF
\ENDFOR
\STATE Select $\kappa^\star$ from all of $\mathcal A$ by the validation score on $D_{\mathrm{val}}$ (Algorithm~\ref{alg:validation-selection})
\STATE \textbf{return} $\kappa^\star$ and its score on $D_{\mathrm{test}}$
\end{algorithmic}
\end{algorithm}

\subsection{Novelty Screening}
\label{sec:functional-novelty}

An agent may simply rediscover a known kernel under a new name, notation, or
implementation. Such programs could occupy the archive without adding distinct
kernel behavior. We therefore compare both the proposed formulation and the
resulting behavior. Before submitting code, each run registers its mathematical
form, PSD argument, closest known kernel, and expected behavior
(Table~\ref{tab:formulation-fields}). To compare behavior, we define the
\textit{kernel fingerprint} $\bar K_{\kappa}(X_d)$ as the centered, unit-norm Gram matrix on
$n_{\mathrm{probe}}$ fixed points $X_d\subset[0,1]^d$ for each probe dimension
$d\in\mathcal{D}$ in Equation~\eqref{eq:fingerprint}. We measure novelty by the
smallest average fingerprint distance to a reference kernel:
\begin{equation}
 d_{\mathrm{nov}}(\kappa,\mathcal{R})=
 \min_{r\in\mathcal{R}}\frac{1}{|\mathcal{D}|}\sum_{d\in\mathcal{D}}
 \big\lVert \bar K_{\kappa}(X_d)-\bar K_r(X_d)\big\rVert_F,
 \label{eq:functional-novelty}
\end{equation}
where $\mathcal{R}$ contains standard, multiscale, and previously discovered
kernels. We define $\tau_{\mathrm{nov}}$ as the threshold for identifying
near-duplicate programs. The threshold is soft rather than a hard rejection, so a program with $d_{\mathrm{nov}}<\tau_{\mathrm{nov}}$ can still enter the archive if it passes the acceptance checks and later be reused as an inspiration. The penalty decreases from a maximum of $\lambda_{\mathrm{nov}}$ as prediction improves, as defined in
Equation~\eqref{eq:novelty-penalty}. Appendix~\ref{app:impl} reports the values of $\tau_{\mathrm{nov}}$ and $\lambda_{\mathrm{nov}}$ used in our experiments.
\section{Experiments}
\label{sec:experiments}
We evaluate discovered kernels as GP covariance functions and ask three questions: (1) Do the discovered kernels transfer to tasks the search never saw? (2) Does writing new base kernels matter, or would recombining known ones do as well? (3) Do the discovered kernels encode meaningful domain structures?

\subsection{Meta-Evaluation Framework}
\label{sec:metaeval}

We seek kernel programs that transfer to unseen tasks without
repeating discovery. Following \citet{goldie2026procedural}, we
search for programs and tune their parameters on meta-training
episodes. We select a frozen program on new meta-validation episodes
from the same families, then test transfer to families or domain
structures withheld from search. Search never uses validation
episodes, so choosing among the frozen programs is a finite model-selection
problem. Under the assumptions in
Appendix~\ref{app:selection-guarantee}, the selection bound depends
on the number of frozen candidates, not the size of the program space.
\subsection{Baselines and Metrics}
\label{sec:baselines} \label{sec:metrics}

We compare our method against the following baselines.
\begin{itemize}[leftmargin=*,itemsep=1pt,topsep=2pt]
\item \textbf{Fixed kernels:} linear, periodic, RBF, Mat\'ern, rational quadratic, and automatic relevance determination (ARD) Mat\'ern.
\item \textbf{Learned kernel families:} spectral mixtures \citep{wilson2013spectralmixture}, grid spectral mixtures \citep[GSM;][]{suwandi2025gsmp}, input warping \citep{snoek2014inputwarping}, and random Fourier features \citep[RFF;][]{rahimi2007random}.
\item \textbf{Deep and meta-learned kernels:} deep kernel learning \citep[DKL;][]{wilson2016deepkernel} on ChemBench, and Few-Shot Bayesian Optimization \citep[FSBO;][]{wistuba2021fsbo}, which meta-learns a deep kernel across BBO functions.
\item \textbf{Kernel structure search:} compositional kernel search \citep[CKS;][]{duvenaud2013structure,lloyd2014automatic}, CAKE \citep{suwandi2025cake}, and AutoGP \citep{saad2023sequential} on forecasting.
\item \textbf{Domain references:} a domain-feature library on ChemBench, and relative-time ARD Mat\'ern, a physiological library, and per-patient marginal-likelihood fits on GlucoseBench.
\end{itemize}
Each family keeps its best setting on meta-training episodes (Appendix~\ref{app:metabo}). Appendices~\ref{app:fsbo}, \ref{app:ts-extra}, \ref{app:chembench}, and~\ref{app:glucose} give the detailed protocols for the baselines.

\paragraph{Predictive metrics.}
Each episode is a small regression dataset. Every kernel runs inside the same GP with its discovered structure fixed. The backend fits only covariance amplitude and observation noise by training marginal likelihood on fixed grids in Equation~\eqref{eq:amp-noise}. We report the continuous ranked probability score \citep[CRPS;][]{gneiting2007strictly} in Equation~\eqref{eq:crps}. This score evaluates the predictive distribution against each held-out observation. Lower CRPS is better (Appendix~\ref{app:protocol}). Appendix~\ref{app:metabo} also evaluates the frozen kernels in GP-based Bayesian optimization on the synthetic functions of \S\ref{sec:results-bo}.

\subsection{Results}
\label{sec:results}

\subsubsection{Black-Box Optimization}
\label{sec:results-bo}

\paragraph{Benchmark.}
\label{sec:domains}
Tasks are standard test functions rescaled to $[0,1]^d$. Each episode fits on $n=8+2d$ randomly sampled observations and scores 32 distinct held-out query points. Search and validation draw on five function families (Branin, Ackley, Cosine, Hartmann-6, Levy), and testing uses six other families that the search never sees (Bukin, Drop-Wave, Griewank, H\"older Table, Rosenbrock, Rastrigin). Each episode randomly permutes, reflects, and warps the coordinates and may append irrelevant ones. 
Every campaign consists of twelve discovery runs on meta-training episodes (Appendix~\ref{app:metabo}), and Figure~\ref{fig:lineage-trees} shows how the accepted programs relate. Each program and baseline is then evaluated on 50 validation and 60 test tasks.

\paragraph{Results.}
\label{sec:mainresults}
We repeated discovery seven times with an ensemble of three LLM proposers. We froze each validation-selected kernel before testing. Most campaigns outperformed the best tuned baseline on validation (Figure~\ref{fig:lineage-trees}a). On held-out families, six of seven kernels achieved lower predictive CRPS than Mat\'ern-5/2, the strongest fixed baseline. The average CRPS improvement was 3\%, and each campaign selected a different program (Appendix~\ref{app:revision-saved}). We also ran seven matched campaigns with GPT-6 Astra as the only proposer. Their CRPS gains over Mat\'ern ranged from 5.8\% to 11.2\%, with a mean of 8.3\% and a standard deviation of 2.0\% ($p=0.003$ against the ensemble). Within the fixed twelve-run budget, Astra reached the best tuned baseline after one run, whereas the ensemble required about eight (Figure~\ref{fig:search-progress}). Across campaigns, Astra's average discovery matched the human-refined \dwf{} kernel without human editing (Appendix~\ref{app:proposer-scaling}).

\begin{figure*}[t]
\centering
\includegraphics[width=\textwidth]{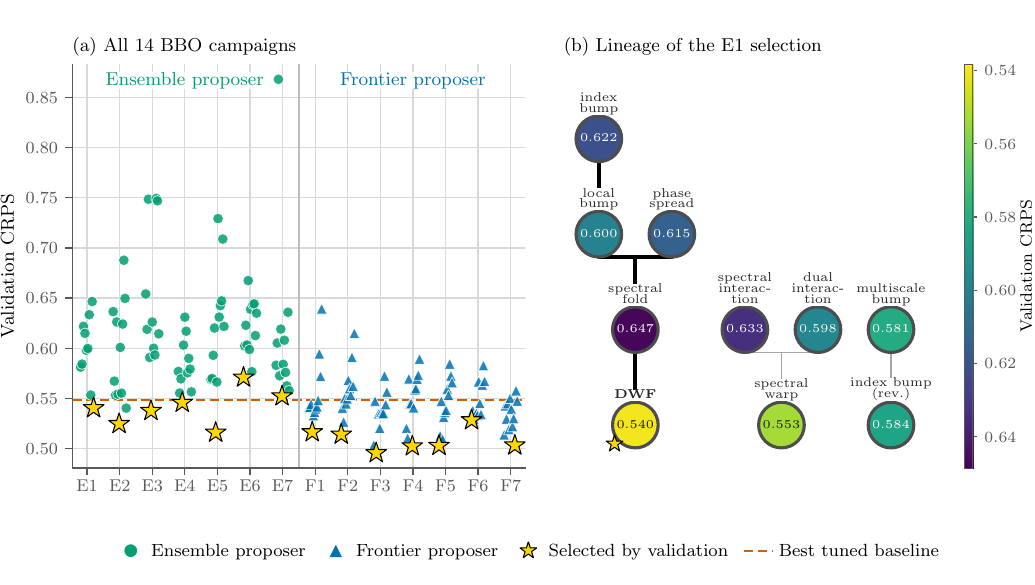}
\caption{\textbf{Synthetic regression: search archives.} (a) Validation continuous ranked probability score (CRPS, Equation~\eqref{eq:crps}, lower is better) of every accepted program in 14 campaigns. E1--E7 use the three-model ensemble, and F1--F7 use GPT-6 Astra as the sole proposer. Stars mark validation-selected programs, and the dashed line marks the best tuned baseline. (b) Lineage of the first ensemble campaign (E1). Edges join each parent to its revisions, and the starred node is the selected program.}
\label{fig:lineage-trees}
\end{figure*}

\begin{figure*}[t]
\centering
\includegraphics[width=\textwidth]{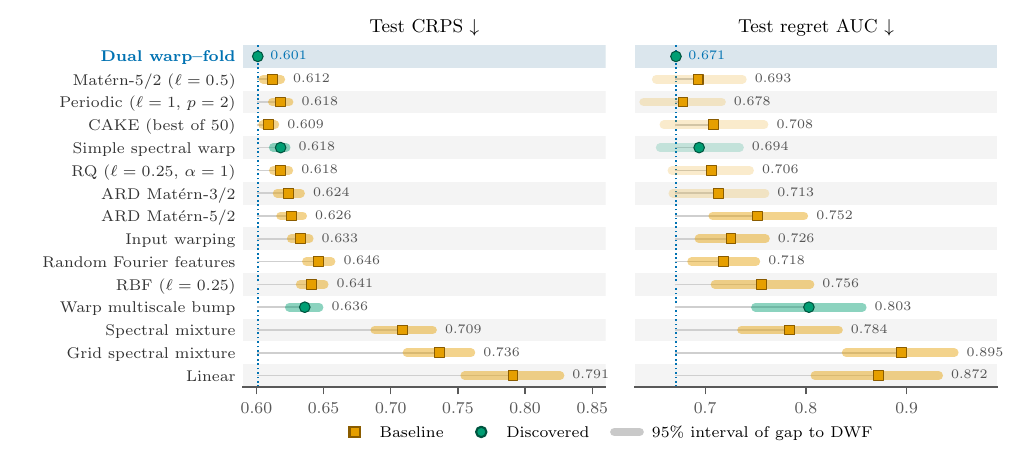}
\caption{Frozen BBO results on 60 held-out test tasks (lower is better). Left: CRPS. Right: regret AUC, the mean normalized simple regret over BO steps. Points are means. Bars are paired 95\% bootstrap intervals for the gap to \dwf{}, where each line ends. Rows are ordered by average rank across both metrics. Fixed kernels use the setting selected on training tasks. Grid spectral mixtures (GSM), input warping, and random Fourier features (RFF) average five fitted variants.}
\label{fig:bbo-results}
\end{figure*}

\paragraph{Open versus closed search.}
We tested whether the improvement came from the search procedure or from access to kernels outside the fixed grammar. The matched ablation used the same agent, budget, score, and search loop, but restricted proposals to compositions of library kernels. The full four-contract search won 11 of 12 paired campaigns and reduced test CRPS by about 8\% relative to this closure-only ablation (Table~\ref{tab:search-ablation}). CKS and CAKE performed similarly to closure-only search. Independent proposals also gave higher CRPS, but this arm changes feedback, cross-iteration novelty screening, and candidate retention together. The comparison therefore tests their combined contribution rather than archive sharing alone (Appendices~\ref{app:revision-saved} and~\ref{app:search-ablation}).

\paragraph{A representative discovery.}
\label{sec:dwf-analysis}
We analyze \emph{dual warp--fold} (\dwf{}), which validation selected in the first ensemble campaign (E1). \dwf{} is a nonstationary kernel that is not a sum or product of standard kernels. It applies a mild monotone \emph{warp} and a triangular \emph{fold} to each coordinate, so values at mirrored points are encouraged to agree (Figure~\ref{fig:dwf-gp-samples}, Proposition~\ref{prop:dwf-cusp}). \dwf{} outperformed every fixed baseline in predictive CRPS, and the paired intervals excluded zero. Keeping only the \emph{warp}, or only the \emph{fold}, did not reproduce that gain (Appendix~\ref{app:dwf-warping}). A human-refined variant then reduced held-out CRPS by a further 5.7\% (Appendix~\ref{app:dwf}). Given the same meta-training data, \dwf{} also reached a lower predictive CRPS than FSBO \citep{wistuba2021fsbo} (Appendix~\ref{app:fsbo}).

\subsubsection{Time-Series Forecasting}
\label{sec:results-ts}

\label{sec:ts-domain}
\paragraph{Benchmark.}
We test whether one frozen kernel transfers across time series without adapting its structure to each record. Search and validation use new windows from one monthly NOAA Global Monitoring Laboratory record\footnote{The records are taken from \url{https://gml.noaa.gov/}.} for each of CO$_2$, CH$_4$, and N$_2$O.
Testing uses SF$_6$, CFC-12, and CFC-11. Each episode samples a training-window length between 216 and $T_{\mathrm{rec}}-48$ months, then holds out the next 48 months. Here, $T_{\mathrm{rec}}$ is the record length. Appendix~\ref{app:ts} lists each record length and the window-sampling rule. Each record has its own univariate GP, and only the frozen kernel is shared. The CFC records rise and then decline, a shape that none of the three meta-training records has, so they test extrapolation. Some windows reverse time, so extrapolation runs in both directions (Appendices~\ref{app:ts} and~\ref{app:ts-extra}).

\paragraph{Results.}
Validation selected a novel residual period-bank kernel, which adds a constant feature and a bank of low-frequency sinusoidal features to a Mat\'ern kernel, as defined in Equation~\eqref{eq:period-bank-frozen}. It had the lowest mean validation CRPS across CO$_2$, CH$_4$, and N$_2$O, and it matched AutoGP \citep{saad2023sequential}, which infers a new structure for every window (Appendix~\ref{app:ts-extra}). Figure~\ref{fig:ts-forecasts-main} shows where the gain comes from. On CFC-12 and CFC-11, the trajectory turns within the held-out interval. RBF extrapolates the training trend with a narrow band, while the period-bank kernel flattens toward the plateau and its wider band covers the observations. It ranked first among 58 frozen candidates on both records (Appendix~\ref{app:ts-extra}). Its sinusoids act as slow trends and not as cycles (Appendix~\ref{app:kernel-descriptions}). Validation selection mattered because the discovery that scored best on SF$_6$ in hindsight lost every CFC window to the validation-selected kernel.

\begin{figure*}[t]
\centering
\includegraphics[width=\textwidth]{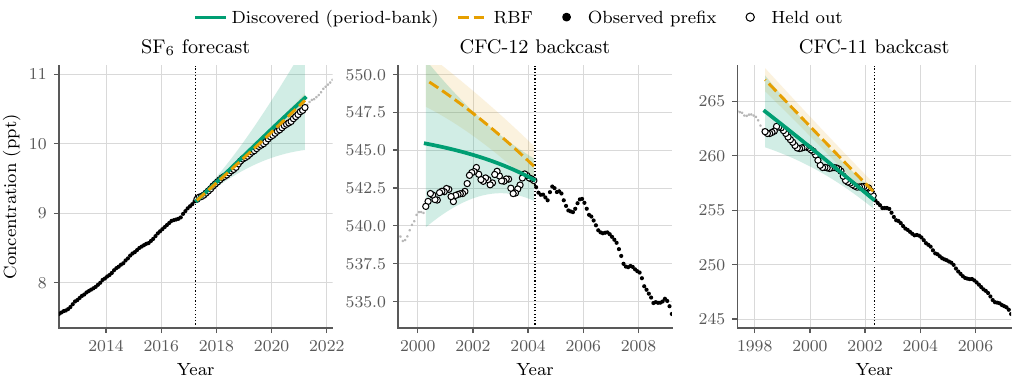}
\caption{\textbf{Forecasting: held-out greenhouse-gas records.} The discovered kernel is the ensemble-selected residual period-bank kernel. Each panel shows, for one record, the window with its median CRPS. Values are in parts per trillion (ppt). SF$_6$ is forecast forward. The chlorofluorocarbon records CFC-12 and CFC-11 are backcast, so their held-out stretch precedes the dotted line. Bands are 95\% intervals. RBF misses the turn in both CFC records. Its smooth posterior mean follows the broader trend of the training window and starts above the last observation.}
\label{fig:ts-forecasts-main}
\end{figure*}

\subsubsection{Enzyme Kinetics}
\label{sec:chembench}

\paragraph{Benchmark.}
ChemBench \citep{kabra2026llm} poses enzyme kinetics as regression from seven experimental variables to a reaction rate $r_0$. The variables are substrate concentration, inhibitor concentration, second-substrate concentration, product concentration, enzyme loading, temperature, and pH. Each episode contains one dataset from one mechanism, with 20 random experiments for fitting and 20 for scoring. Search uses data from ten canonical, single-mechanism rate-law families. Testing uses five structurally distinct mechanisms outside those families, with 15 episodes per mechanism. Appendix~\ref{app:chembench} names both sets and specifies the split.

\paragraph{Results.}
Validation error predicted which campaigns transferred to held-out mechanisms. Five of the nine ensemble campaigns had lower validation error than the other four. All five reduced error on every held-out mechanism (Figure~\ref{fig:chembench}a). The kernel selected by validation achieved about 25\% lower error than Mat\'ern. It performed similarly to the optimized ARD, deep-kernel, and domain-library references (Appendix~\ref{app:chembench}).

GPT-6 Astra produced larger gains. The selected kernel reduced error on every held-out mechanism. Relative to optimized ARD, it reduced error by almost half. It also outperformed all six learned references in paired tests (Table~\ref{tab:chem-paired}). The selected program constructs ten substrate-response shapes from the training rate laws. It also includes local sensitivities to substrate concentration and modulation by enzyme loading. The kernel predicts an unseen mechanism as a combination of these known shapes and local variations. A distance-based kernel cannot express this biochemical prior (Appendices~\ref{app:kernel-descriptions} and~\ref{app:chembench}).

\begin{figure*}[t]
\centering
\includegraphics[width=\textwidth]{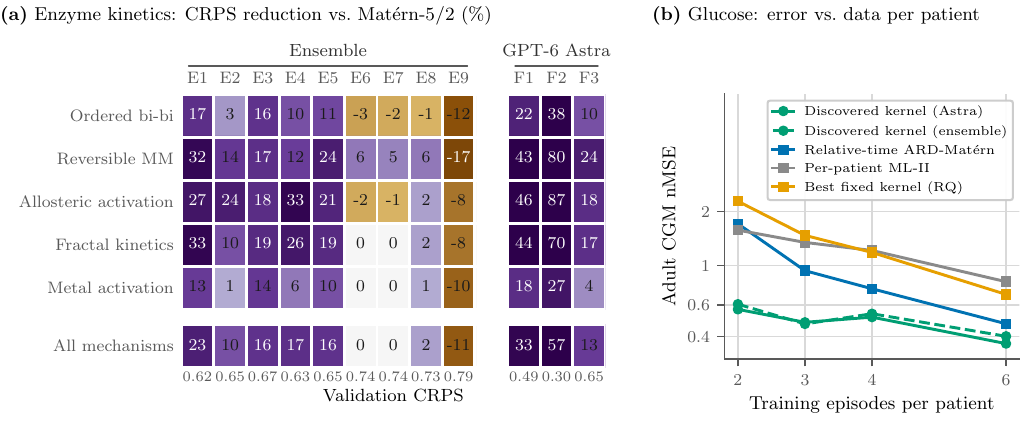}
\caption{\textbf{Enzyme kinetics and glucose dynamics: transfer to unseen mechanisms and patients.} (a) CRPS reduction (\%) relative to Mat\'ern-$5/2$ on held-out enzyme mechanisms, with each campaign's validation CRPS below. Purple means lower error and brown means higher error. Validation selected E1 and F2, shown in bold. (b) Normalized mean squared error (MSE, Equation~\eqref{eq:glucose-nmse} of Appendix~\ref{app:glucose}, log scale) for adult continuous glucose monitor (CGM) forecasts after a 45-minute prefix, against training episodes per patient. Validation selected the discovered kernels. References are tuned ARD-Mat\'ern, per-patient type-II maximum-likelihood (ML-II) fitting, and the best fixed kernel.}
\label{fig:science-transfer}
\label{fig:chembench}
\label{fig:glucose-budget}
\end{figure*}

\subsubsection{Glucose Dynamics}
\label{sec:glucose}

\paragraph{Benchmark.}
GlucoseBench \citep{murphy2026mdagent} builds on the UVA/Padova type-1 diabetes simulator \citep{xie2018simglucose,kovatchev2009silico}. In each episode, a virtual patient eats a meal and may receive an insulin bolus. We forecast continuous glucose monitor (CGM) readings after a 45-minute prefix. We use 10 child profiles for search, 10 adolescents for selection, and 10 adults for testing. Each forecast uses only episodes from its target patient. Thus, the experiment tests whether the kernel structure transfers across patients. 
We measure error using the normalized MSE of the CGM readings, aggregated by the GlucoseBench evaluator. Appendix~\ref{app:glucose} defines normalized MSE in Equation~\eqref{eq:glucose-nmse} and its aggregation, and also reports CRPS.

\paragraph{Results.}
Figure~\ref{fig:glucose-budget}b shows that error decreases as patient-specific data increases. The benefit of each discovered kernel is largest in the low-sample regime. The strongest tuned reference is a relative-time ARD Mat\'ern kernel. It applies one learned lengthscale to each of time, time since the meal, time since the bolus, meal size, and bolus size. Against this reference, the five discoveries reduce geometric-mean normalized MSE by 38--67\% with two training episodes. With six episodes and a 45-minute prefix, the reductions are 0--23\%, and frontier F1, frontier F2, and ensemble c2 reduce it by 15--23\%. The ensemble kernel selected by validation encodes a delayed meal response and an insulin effect that lasts about two hours. Removing either feature eliminates the gain, which links the improvement to these physiological response patterns (Figure~\ref{fig:glucose-learned}).

\subsubsection{Takeaways}
The results answered the three questions of \S\ref{sec:experiments}. Discovered kernels showed predictive transfer to unseen BBO function families, enzyme mechanisms, glucose patients, and CFC records. On BBO benchmarks, open search reduced test CRPS by about 8\% relative to closure-only search, so access to new base kernels mattered. These interpretable programs can also provide useful starting points for human refinement. For example, replacing \dwf{}'s joint construction with a sum of its warp and fold kernels further reduced held-out CRPS by 5.7\% (Appendix~\ref{app:dwf}).

\paragraph{Cost.}
A campaign used about 100 LLM calls and roughly 30 minutes of wall-clock time, and each verification took about 3 seconds (Appendix~\ref{app:ts}). This cost is paid once, because the frozen kernel is reused on every new task with no further LLM calls. 

\section{Conclusion}
\label{sec:conclusion}
We introduced \method, in which coding agents write kernel components and a trusted interpreter assembles each accepted kernel under a PSD-preserving construction contract. These contracts widen the search beyond compositions of a fixed kernel library, and every accepted program remains valid. Our experiments demonstrated that the discovered kernels encode inductive biases that transfer to unseen tasks without further LLM calls.
Because the discovered kernels are short programs, human researchers can also interpret and refine their structures, as shown in our \dwf{} case study (\S\ref{sec:dwf-analysis}). Future construction contracts could support state-space kernels from linear ODEs \citep{loper2021general}, multi-output kernels \citep{alvarez2012kernels}, kernels for structured inputs \citep{gartner2003survey}, and semantic kernels for LLM uncertainty \citep{nikitin2024kernel}. Appendix~\ref{app:extensions} describes several limitations and directions for future work.

\clearpage

\bibliographystyle{apalike}
\bibliography{references}

@misc{hamzi2026algorithmic,
  title = {Toward an Algorithmic Theory of Machine Learning: Unifying Machine Learning and Algorithmic Information Theory via Kernel Methods},
  author = {Hamzi, Boumediene and Hutter, Marcus},
  year = {2026},
  howpublished = {Preliminary research monograph},
  note = {Draft dated August 1, 2026. Preliminary lecture notes}
}

@article{aronszajn1950theory,
  title = {Theory of Reproducing Kernels},
  author = {Aronszajn, Nachman},
  journal = {Transactions of the American Mathematical Society},
  volume = {68},
  number = {3},
  pages = {337--404},
  year = {1950},
  doi = {10.1090/S0002-9947-1950-0051437-7}
}

@article{micchelli2006universal,
  title = {Universal Kernels},
  author = {Micchelli, Charles A. and Xu, Yuesheng and Zhang, Haizhang},
  journal = {Journal of Machine Learning Research},
  volume = {7},
  pages = {2651--2667}, 
  year = {2006},
  url = {https://www.jmlr.org/papers/v7/micchelli06a.html}
}

@article{bach2008consistency,
  title = {Consistency of the Group {L}asso and Multiple Kernel Learning},
  author = {Bach, Francis R.},
  journal = {Journal of Machine Learning Research},
  volume = {9},
  pages = {1179--1225},
  year = {2008},
  url = {https://www.jmlr.org/papers/v9/bach08b.html}
}

@article{cortes2012alignment,
  title = {Algorithms for Learning Kernels Based on Centered Alignment},
  author = {Cortes, Corinna and Mohri, Mehryar and Rostamizadeh, Afshin},
  journal = {Journal of Machine Learning Research},
  volume = {13},
  pages = {795--828},
  year = {2012},
  url = {https://www.jmlr.org/papers/v13/cortes12a.html}
}

@article{halton1960efficiency,
  title   = {On the efficiency of certain quasi-random sequences of points in evaluating multi-dimensional integrals},
  author  = {Halton, John H.},
  journal = {Numerische Mathematik},
  volume  = {2},
  number  = {1},
  pages   = {84--90},
  year    = {1960},
  publisher = {Springer}
}

@inproceedings{suwandi2025cake,
  title     = {Adaptive Kernel Design for {B}ayesian Optimization Is a Piece of {CAKE} with {LLM}s},
  author    = {Suwandi, Richard Cornelius and Yin, Feng and Wang, Juntao and Li, Renjie and Chang, Tsung-Hui and Theodoridis, Sergios},
  booktitle = {Advances in Neural Information Processing Systems},
  volume    = {38},
  year      = {2025}
}

@article{brunzema2026agenticbo,
  title   = {Agentic {B}ayesian Optimization through Surrogate-Augmented Autoresearch},
  author  = {Brunzema, Paul and Tiao, Louis and Le, Nhat and De Angeli, Kevin and Xuan, Yao and Gligorijevic, Djordje},
  journal = {arXiv preprint arXiv:2608.00316},
  year    = {2026}
}

@misc{karpathy2026autoresearch,
  title        = {Autoresearch: {AI} Agents Running Research on Single-{GPU} {Nanochat} Training Automatically},
  author       = {Karpathy, Andrej},
  year         = {2026},
  howpublished = {\url{https://github.com/karpathy/autoresearch}},
  note         = {GitHub repository}
}

@inproceedings{snoek2014inputwarping,
  title     = {Input Warping for {B}ayesian Optimization of Non-Stationary Functions},
  author    = {Snoek, Jasper and Swersky, Kevin and Zemel, Richard S. and Adams, Ryan P.},
  booktitle = {Proceedings of the International Conference on Machine Learning},
  pages     = {1674--1682},
  year      = {2014}
}

@inproceedings{duvenaud2013structure,
  title     = {Structure Discovery in Nonparametric Regression through Compositional Kernel Search},
  author    = {Duvenaud, David and Lloyd, James Robert and Grosse, Roger and Tenenbaum, Joshua B. and Ghahramani, Zoubin},
  booktitle = {Proceedings of the International Conference on Machine Learning},
  year      = {2013}
}

@inproceedings{malkomes2016boms,
  title     = {{B}ayesian Optimization for Automated Model Selection},
  author    = {Malkomes, Gustavo and Schaff, Charles and Garnett, Roman},
  booktitle = {Advances in Neural Information Processing Systems},
  volume    = {29},
  year      = {2016}
}

@inproceedings{sun2018neuralkernelnetwork,
  title     = {Differentiable Compositional Kernel Learning for {G}aussian Processes},
  author    = {Sun, Shengyang and Zhang, Guodong and Wang, Chaoqi and Zeng, Wenyuan and Li, Jiaman and Grosse, Roger},
  booktitle = {Proceedings of the International Conference on Machine Learning},
  year      = {2018}
}

@inproceedings{wilson2013spectralmixture,
  title     = {{G}aussian Process Kernels for Pattern Discovery and Extrapolation},
  author    = {Wilson, Andrew Gordon and Adams, Ryan Prescott},
  booktitle = {Proceedings of the International Conference on Machine Learning},
  year      = {2013}
}

@inproceedings{wilson2016deepkernel,
  title     = {Deep Kernel Learning},
  author    = {Wilson, Andrew Gordon and Hu, Zhiting and Salakhutdinov, Ruslan and Xing, Eric P.},
  booktitle = {Proceedings of the International Conference on Artificial Intelligence and Statistics},
  year      = {2016}
}

@article{suwandi2025gsmp,
  title   = {Sparsity-Aware Distributed Learning for {G}aussian Processes With Linear Multiple Kernel},
  author  = {Suwandi, Richard Cornelius and Lin, Zhidi and Yin, Feng and Wang, Zhiguo and Theodoridis, Sergios},
  journal = {IEEE Transactions on Neural Networks and Learning Systems},
  volume  = {36},
  number  = {8},
  pages   = {14869--14883},
  year    = {2025},
  doi     = {10.1109/TNNLS.2025.3531784}
}

@inproceedings{rahimi2007random,
  title     = {Random Features for Large-Scale Kernel Machines},
  author    = {Rahimi, Ali and Recht, Benjamin},
  booktitle = {Advances in Neural Information Processing Systems},
  volume    = {20},
  year      = {2007}
}

@inproceedings{williams2001nystrom,
  title     = {Using the {N}ystr\"{o}m Method to Speed Up Kernel Machines},
  author    = {Williams, Christopher K. I. and Seeger, Matthias},
  booktitle = {Advances in Neural Information Processing Systems},
  volume    = {13},
  year      = {2001}
}

@book{rasmussen2006gaussian,
  title     = {{G}aussian Processes for Machine Learning},
  author    = {Rasmussen, Carl Edward and Williams, Christopher K. I.},
  publisher = {MIT Press},
  year      = {2006}
}

@article{romeraparedes2024funsearch,
  title   = {Mathematical Discoveries from Program Search with Large Language Models},
  author  = {Romera-Paredes, Bernardino and Barekatain, Mohammadamin and Novikov, Alexander and Balog, Matej and Kumar, M. Pawan and Dupont, Emilien and Ruiz, Francisco J. R. and Ellenberg, Jordan S. and Wang, Pengming and Fawzi, Omar and Kohli, Pushmeet and Fawzi, Alhussein},
  journal = {Nature},
  volume  = {625},
  number  = {7995},
  pages   = {468--475},
  year    = {2024}
}

@article{novikov2025alphaevolve,
  title   = {{A}lpha{E}volve: A Coding Agent for Scientific and Algorithmic Discovery},
  author  = {Novikov, Alexander and V{\~u}, Ng{\^a}n and Eisenberger, Marvin and Dupont, Emilien and Huang, Po-Sen and Wagner, Adam Zsolt and Shirobokov, Sergey and Kozlovskii, Borislav and Ruiz, Francisco J. R. and Mehrabian, Abbas and Kumar, M. Pawan and See, Abigail and Chaudhuri, Swarat and Holland, George and Davies, Alex and Nowozin, Sebastian and Kohli, Pushmeet and Balog, Matej},
  journal = {arXiv preprint arXiv:2506.13131},
  year    = {2025},
  doi     = {10.48550/arXiv.2506.13131},
  url     = {https://arxiv.org/abs/2506.13131}
}

@article{rice1953classes,
  title   = {Classes of Recursively Enumerable Sets and Their Decision Problems},
  author  = {Rice, Henry Gordon},
  journal = {Transactions of the American Mathematical Society},
  volume  = {74},
  number  = {2},
  pages   = {358--366},
  year    = {1953},
  doi     = {10.2307/1990888}
}

@article{yun2026automated,
  title   = {Automated Kernel Discovery Towards Understanding High-dimensional {Bayesian} Optimization},
  author  = {Yun, Taeyoung and Shin, Woocheol and Song, Inhyuck and Lee, Jaewoo and Park, Jinkyoo},
  journal = {arXiv preprint arXiv:2605.20249},
  year    = {2026}
}

@article{goldie2026procedural,
  title   = {{DiscoGen}: Procedural Generation of Algorithm Discovery Tasks in Machine Learning},
  author  = {Goldie, Alexander D and Wang, Zilin and Hayler, Adrian and Nathani, Deepak and Toledo, Edan and Thampiratwong, Ken and Kalisz, Aleksandra and Beukman, Michael and Letcher, Alistair and Reddy, Shashank and others},
  journal = {arXiv preprint arXiv:2603.17863},
  year    = {2026}
}

@article{mouret2015mapelites,
  title   = {Illuminating Search Spaces by Mapping Elites},
  author  = {Mouret, Jean-Baptiste and Clune, Jeff},
  journal = {arXiv preprint arXiv:1504.04909},
  year    = {2015}
}

@article{gneiting2007strictly,
  title     = {Strictly Proper Scoring Rules, Prediction, and Estimation},
  author    = {Gneiting, Tilmann and Raftery, Adrian E},
  journal   = {Journal of the American Statistical Association},
  volume    = {102},
  number    = {477},
  pages     = {359--378},
  year      = {2007}
}

@article{murphy2026mdagent,
  title={Model discovery agent: {LLM}-assisted {Bayesian} experiment design for data-efficient discovery of mechanistic world models},
  author={Murphy, Kevin},
  journal = {arXiv preprint arXiv:2608.09696},
  year={2026}
}

@article{kabra2026llm,
  title={{LLM-AutoSciLab}: Closed-Loop Scientific Discovery via Active Experimentation with {LLMs}},
  author={Kabra, Sanchit and Abhyankar, Nikhil and Desai, Saaketh and Iyer, Prasad and Reddy, Chandan K},
  journal = {arXiv preprint arXiv:2605.24043},
  year={2026}
}

@inproceedings{nikitin2024kernel,
  title     = {Kernel Language Entropy: Fine-grained Uncertainty Quantification for {LLM}s from Semantic Similarities},
  author    = {Nikitin, Alexander and Kossen, Jannik and Gal, Yarin and Marttinen, Pekka},
  booktitle = {Advances in Neural Information Processing Systems},
  volume    = {37},
  year      = {2024}
}

@inproceedings{griffiths2023gauche,
  title     = {{GAUCHE}: A Library for {G}aussian Processes in Chemistry},
  author    = {Griffiths, Ryan-Rhys and Klarner, Leo and Moss, Henry B. and Ravuri, Aditya and Truong, Sang and Du, Yuanqi and Stanton, Samuel and Tom, Gary and Rankovi{\'c}, Bojana and Jamasb, Arian R. and others},
  booktitle = {Advances in Neural Information Processing Systems},
  volume    = {36},
  year      = {2023}
}

@article{chen2021solving,
  title   = {Solving and Learning Nonlinear {PDE}s with {G}aussian Processes},
  author  = {Chen, Yifan and Hosseini, Bamdad and Owhadi, Houman and Stuart, Andrew M.},
  journal = {Journal of Computational Physics},
  volume  = {447},
  pages   = {110668},
  year    = {2021}
}

@inproceedings{wistuba2021fsbo,
  title     = {Few-Shot {B}ayesian Optimization with Deep Kernel Surrogates},
  author    = {Wistuba, Martin and Grabocka, Josif},
  booktitle = {International Conference on Learning Representations},
  year      = {2021}
}

@inproceedings{pineda2021hpob,
  title     = {{HPO-B}: A Large-Scale Reproducible Benchmark for Black-Box {HPO} based on {OpenML}},
  author    = {Pineda-Arango, Sebastian and Jomaa, Hadi S. and Wistuba, Martin and Grabocka, Josif},
  booktitle = {Neural Information Processing Systems Track on Datasets and Benchmarks},
  year      = {2021}
}

@inproceedings{singhal2026llmpriors,
  title     = {{LLM} Priors for {ERM} over Programs},
  author    = {Singhal, Shivam and Mishra, Priyadarsi and Malach, Eran and Galanti, Tomer},
  booktitle = {Proceedings of the 43rd International Conference on Machine Learning},
  series    = {Proceedings of Machine Learning Research},
  volume    = {306},
  year      = {2026},
  note      = {arXiv:2510.14331}
}

@inproceedings{lotfi2023bayesian,
  title        = {{Bayesian} Model Selection, the Marginal Likelihood, and Generalization},
  author       = {Lotfi, Sanae and Izmailov, Pavel and Benton, Gregory and Goldblum, Micah and Wilson, Andrew Gordon},
  booktitle    = {Proceedings of the International Conference on Machine Learning},
  pages        = {14223--14247},
  year         = {2022},
  organization = {PMLR}
}

@inproceedings{saad2023sequential,
  title     = {Sequential {Monte Carlo} Learning for Time Series Structure Discovery},
  author    = {Saad, Feras A. and Patton, Brian J. and Hoffman, Matthew D. and Saurous, Rif A. and Mansinghka, Vikash K.},
  booktitle = {Proceedings of the 40th International Conference on Machine Learning},
  series    = {Proceedings of Machine Learning Research},
  volume    = {202},
  year      = {2023}
}

@article{saad2024scalable,
  title   = {Scalable Spatiotemporal Prediction with {Bayesian} Neural Fields},
  author  = {Saad, Feras and Burnim, Jacob and Carroll, Colin and Patton, Brian and K{\"o}ster, Urs and Saurous, Rif A. and Hoffman, Matthew},
  journal = {Nature Communications},
  volume  = {15},
  pages   = {7942},
  year    = {2024}
}

@inproceedings{lloyd2014automatic,
  title     = {Automatic Construction and Natural-Language Description of Nonparametric Regression Models},
  author    = {Lloyd, James Robert and Duvenaud, David and Grosse, Roger and Tenenbaum, Joshua B. and Ghahramani, Zoubin},
  booktitle = {Proceedings of the AAAI Conference on Artificial Intelligence},
  pages     = {1242--1250},
  year      = {2014}
}

@article{loper2021general,
  title   = {A General Linear-Time Inference Method for {Gaussian} Processes on One Dimension},
  author  = {Loper, Jackson and Blei, David and Cunningham, John P. and Paninski, Liam},
  journal = {Journal of Machine Learning Research},
  volume  = {22},
  pages   = {1--36},
  year    = {2021}
}

@article{sarkka2013spatiotemporal,
  title     = {Spatiotemporal Learning via Infinite-Dimensional {Bayesian} Filtering and Smoothing: A Look at {Gaussian} Process Regression Through {Kalman} Filtering},
  author    = {S{\"a}rkk{\"a}, Simo and Solin, Arno and Hartikainen, Jouni},
  journal   = {IEEE Signal Processing Magazine},
  volume    = {30},
  number    = {4},
  pages     = {51--61},
  year      = {2013},
  publisher = {IEEE},
  doi       = {10.1109/MSP.2013.2246292}
}

@inproceedings{alvarez2009latent,
  title     = {Latent Force Models},
  author    = {{\'A}lvarez, Mauricio and Luengo, David and Lawrence, Neil D.},
  booktitle = {Proceedings of the 12th International Conference on Artificial Intelligence and Statistics},
  series    = {Proceedings of Machine Learning Research},
  volume    = {5},
  pages     = {9--16},
  year      = {2009}
}

@article{kovatchev2009silico,
  title   = {In Silico Preclinical Trials: A Proof of Concept in Closed-Loop Control of Type 1 Diabetes},
  author  = {Kovatchev, Boris P. and Breton, Marc and Dalla Man, Chiara and Cobelli, Claudio},
  journal = {Journal of Diabetes Science and Technology},
  volume  = {3},
  number  = {1},
  pages   = {44--55},
  year    = {2009}
}

@misc{xie2018simglucose,
  title  = {{Simglucose} v0.2.1},
  author = {Xie, Jinyu},
  year   = {2018},
  note   = {\url{https://github.com/jxx123/simglucose}}
}

@book{shalevshwartz2014understanding,
  title     = {Understanding Machine Learning: From Theory to Algorithms},
  author    = {Shalev-Shwartz, Shai and Ben-David, Shai},
  publisher = {Cambridge University Press},
  year      = {2014}
}

@article{alvarez2012kernels,
  title   = {Kernels for Vector-Valued Functions: A Review},
  author  = {{\'A}lvarez, Mauricio A. and Rosasco, Lorenzo and Lawrence, Neil D.},
  journal = {Foundations and Trends in Machine Learning},
  volume  = {4},
  number  = {3},
  pages   = {195--266},
  year    = {2012},
  doi     = {10.1561/2200000036}
}

@article{gartner2003survey,
  title   = {A Survey of Kernels for Structured Data},
  author  = {G{\"a}rtner, Thomas},
  journal = {ACM SIGKDD Explorations Newsletter},
  volume  = {5},
  number  = {1},
  pages   = {49--58},
  year    = {2003},
  doi     = {10.1145/959242.959248}
}
\clearpage
\appendix
\section*{Appendix Guide}
The appendix is organized as follows.
\begin{itemize}[leftmargin=*,itemsep=2pt,topsep=3pt]
\item \textbf{Method (Appendix~\ref{app:math}).} The validity guarantee and its proof, implementation and evaluation details, the soundness of validation selection, and a stress test of numerical validity checks.
\item \textbf{Prediction and optimization on BBO benchmarks (Appendix~\ref{app:bbo}).} The function-family splits, predictive CRPS evaluation, separate GP-based BO protocol and regret results, baselines and controls, ablations, and robustness checks.
\item \textbf{The DWF kernel (Appendix~\ref{app:dwf}).} A detailed analysis of the discovered kernel, covering its geometry, theoretical properties, and controlled comparisons.
\item \textbf{Time-series forecasting (Appendix~\ref{app:forecasting}).} Experimental setup and full results.
\item \textbf{Scientific structure discovery (Appendix~\ref{app:extra-domains}).} Experimental setup and full results for enzyme kinetics and glucose dynamics.
\item \textbf{Transfer and interpretation (Appendix~\ref{app:cross-domain}).} Cross-domain reuse of discovered kernels and plain-language descriptions of what they encode.
\item \textbf{Scope and future extensions (Appendix~\ref{app:extensions}).} The scope of the current experiments and directions for extending the framework.
\item \textbf{LLM prompts (Appendix~\ref{app:prompts}).} The agent interface and the prompts given to the coding agents.
\end{itemize}

\section{Method: Further Details}
\label{app:math}
\label{app:theory-connections}

This appendix supports \S\ref{sec:framework}. We first prove the construction guarantee and the approximation power of the feature contract. We then specify the implementation defaults, the acceptance checks, and the soundness of validation selection, and we close with the verification experiments. Mathematical results are stated in exact arithmetic, and computational limits are specified separately.

\subsection{PSD Preservation Under Trusted Assembly}
\label{app:proof}

The argument has two parts. Part (a) shows that every matrix the interpreter assembles is symmetric and PSD whenever the components return finite outputs of the declared shape. Part (b) adds purity, which makes those matrices the Gram matrices of one kernel function on the whole domain.

\begin{proof}
\textbf{(a)} Fix finite-valued inputs $x_1,\dots,x_n$.

\emph{Feature map.} The interpreter evaluates $\varphi_\theta$ once per input and stacks the outputs as the rows of $\Xi\in\mathbb{R}^{n\times m}$, then sets $K=\Xi\Xi^\top$. Symmetry is immediate, and $c^\top Kc=\lVert\Xi^\top c\rVert_2^2\ge0$ for every $c\in\mathbb{R}^n$. The same row of $\Xi$ enters both the row and the column of $K$, so the identity holds whether or not $\varphi_\theta$ is pure.

\emph{Spectral.} The weights $w_j=r_j^2/M$ are nonnegative for any real $r_j$. Define $\psi(x)=\big(\sqrt{w_j}\cos\omega_j^\top x,\ \sqrt{w_j}\sin\omega_j^\top x\big)_{j=1}^M\in\mathbb{R}^{2M}$. The angle-difference identity $\cos a\cos b+\sin a\sin b=\cos(a-b)$ gives $\psi(x)^\top\psi(y)=\sum_j w_j\cos(\omega_j^\top(x-y))$. The interpreter forms $K=\Psi\Psi^\top$, so the feature-map case applies.

\emph{Pullback.} Let $z_i=T_\theta(x_i)\in\mathbb{R}^m$, where repeated values are allowed. Then $K_{ij}=k_0(z_i,z_j)$, which is the Gram matrix of $k_0$ on $z_1,\dots,z_n$. That matrix is symmetric and PSD because $k_0$ is PSD on $\mathbb{R}^m$.

\emph{Closure.} The argument is a structural induction on the tree. Leaves give symmetric PSD Gram matrices, either from the trusted library or from the three cases above. Suppose two subtrees give symmetric PSD matrices $K_1$ and $K_2$. Then $K_1+K_2$ and $cK_1$ for $c\ge0$ are symmetric PSD.
The entrywise product $K_1\circ K_2$ is symmetric PSD by the Schur product theorem.

\textbf{(b)} Under purity, each component defines a function on $\cX_0$. Thus, $\kappa_\theta(x,y)$ has one well-defined real value for each pair $(x,y)\in\cX_0^2$. This value does not depend on the batch that contains the pair. For any $n$ and any $x_1,\dots,x_n\in\cX_0$, the matrix $[\kappa_\theta(x_i,x_j)]_{ij}$ is exactly the matrix that the interpreter assembles on that batch. Part (a) shows that this matrix is symmetric and PSD. Therefore, $\kappa_\theta$ satisfies Definition~\ref{def:valid} for the fixed $\theta$. Applying the same argument to every $\theta\in\Theta$ proves that $\kappa$ is valid.
\end{proof}

\paragraph{Scope of the guarantee.}
The proposition certifies the interpreter's assembly when the components meet the stated assumptions. We restrict component behavior through the interface and static source checks, then probe consistency on sampled batches using the relations in Table~\ref{tab:contract-checks}. These checks assess conformance, while the kernel-function guarantee assumes purity over the full stated domain. Nonnegative residual and additive combinations follow from the closure case.

\subsection{Approximation Power of the Feature Contract}\label{app:contract-approximation}

The feature contract restricts how components are assembled, but it need not restrict approximation to a fixed kernel library. The following result makes this distinction precise as feature dimension and computational resources grow.

\begin{proposition}[Uniform approximation by certified features]
\label{prop:contract-density}
Let $\cX\subset\mathbb{R}^d$ be compact and let $k$ be a real continuous PSD kernel on $\cX$. For every $\varepsilon>0$, there are $m<\infty$ and a polynomial map $\varphi:\mathbb{R}^d\to\mathbb{R}^m$ with rational coefficients such that
\begin{equation}
 \sup_{x,y\in\cX}|k(x,y)-\varphi(x)^\top\varphi(y)|<\varepsilon.
 \label{eq:contract-density}
\end{equation}
The approximating kernel is admitted by the mathematical feature-map contract in exact arithmetic.
\end{proposition}

\begin{proof}
Let $k_x=k(x,\cdot)$ be the canonical feature of $x$ in the reproducing kernel Hilbert space (RKHS). It is continuous because
$\|k_x-k_y\|^2=k(x,x)+k(y,y)-2k(x,y)$, so $\{k_x:x\in\cX\}$ is compact and bounded by some $M<\infty$. Choose a finite $\eta$-net in this image, let $V$ be its span, and let $P_V$ be the orthogonal projection. Then $\sup_x\|(I-P_V)k_x\|\le\eta$. Orthogonality gives
\[
 |k(x,y)-\langle P_Vk_x,P_Vk_y\rangle|\le\eta^2.
\]
In an orthonormal basis of $V$, the projected coordinates form a continuous map $z:\cX\to\mathbb{R}^m$. By Stone--Weierstrass, approximate its coordinates uniformly by polynomials. Rational approximation of their finitely many coefficients preserves uniform approximation on the compact domain. Choose the resulting $\varphi$ with $\sup_x\|z(x)-\varphi(x)\|\le\delta$. Since $\sup_x\|z(x)\|\le M$, the extra inner-product error is at most $2M\delta+\delta^2$. Choose $\eta,\delta$ so that $\eta^2+2M\delta+\delta^2<\varepsilon$. The case $k=0$ is represented exactly by the zero feature.
\end{proof}

This applies classical feature-space approximation arguments~\citep{aronszajn1950theory,micchelli2006universal} to the construction interface. Each approximant has a finite executable description, but its degree, dimension, and evaluation cost may grow with accuracy. The proposition supplies neither an efficient synthesis algorithm nor a uniform guarantee under the runtime cap, floating-point arithmetic, or an arbitrary parameter range. It does not assert a complete verifier for unrestricted programs.

Feature maps can also define nonstationary kernels. Unless $\varphi$ has a special form, such as random Fourier features, $\varphi(x)^\top\varphi(y)$ depends on $x$ and $y$ separately rather than only on $x-y$. For example, $\varphi(x)=x$ gives the linear kernel $xy$. Its diagonal is $x^2$, which is not constant. Therefore, no combination of stationary kernels can approximate it uniformly on $[0,1]$. A grammar that includes a linear base kernel already contains $xy$. The value of the feature-map contract is broader: agents can write arbitrary feature maps, including the saturation features discovered for enzyme kinetics (Appendix~\ref{app:chembench}).

\subsection{Implementation Details}
\label{app:impl}

This section specifies the implementation defaults in the order used by a discovery run. The scheduler chooses what to revise, the tuner searches parameters, the fitness and novelty rules decide what enters the archive, and the contract interfaces specify the permitted component computations.

\paragraph{Run scheduler.}
A run starts as an independent root with probability $p_{\mathrm{new}}$, and always while the archive is empty. Otherwise it selects a parent from the quality-diversity elites. The scheduler draws the parent uniformly from all elites with probability $p_{\mathrm{unif}}$ and from the top $k_{\mathrm{top}}$ elites by selection score otherwise. These draws balance exploration and exploitation. When the run revises a parent, it also samples up to $k_{\mathrm{insp}}$ inspirations from elites outside the scheduler-assigned target niche, and a root run samples its inspirations from all elites. In all experiments we use $p_{\mathrm{new}}=p_{\mathrm{unif}}=0.35$ and $k_{\mathrm{top}}=k_{\mathrm{insp}}=3$. The run budget is $B=12$ for BBO and forecasting and $B=8$ for chemistry. If every elite lies in the target niche, it samples from the other elites, so the prompt is not confined to the parent. 

The seven niches are input geometry, spectral and periodic structure, feature projection, additive interaction, multiresolution and local structure, residual composites, and an open-ended niche. The task prompt defines these labels, and the agent declares one label in its formulation (Appendix~\ref{app:agent-prompts}). The scheduler fills empty niches before it cycles through occupied ones. We store formulations, candidate drafts, and parameter trials separately, so a tuning trial never becomes an evolutionary descendant. Formally, each accepted candidate $\kappa$ occupies a cell
\begin{equation}
  c(\kappa) = \big(\mathrm{niche}(\kappa),\, \mathrm{growth}(\kappa),\, \mathrm{band}(\kappa)\big),
  \label{eq:cell}
\end{equation}
where $\mathrm{niche}(\kappa)$ is the declared construction niche, $\mathrm{growth}(\kappa)$ the feature-growth descriptor, and $\mathrm{band}(\kappa)$ the novelty band, obtained by bucketing the functional-novelty distance $d_{\mathrm{nov}}$ of Equation~\eqref{eq:functional-novelty} into four bands: duplicate ($d_{\mathrm{nov}}<0.03$), near\_reference ($<0.08$), distinct ($<0.15$), and strongly\_distinct ($\ge0.15$). Each cell keeps only its best-scoring candidate as the elite. Figure~\ref{fig:archive-map} shows how the grid fills over a campaign.

\begin{figure}[t]
\centering
\includegraphics[width=\textwidth]{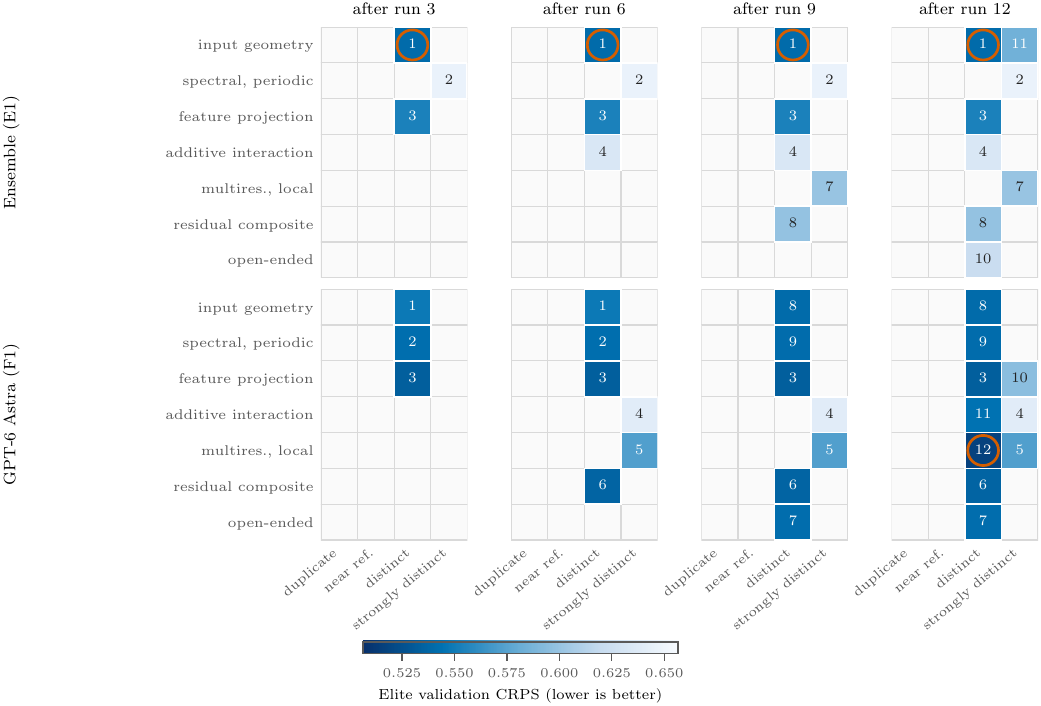}
\caption{\textbf{The archive fills niche by niche over a campaign.} MAP-Elites archive cells after runs 3, 6, 9, and 12 for the first ensemble campaign (E1, top) and the first GPT-6 Astra campaign (F1, bottom). Rows are niches and columns are novelty bands. Each filled cell shows its current elite (best search score), colored by validation CRPS (darker is better) and labeled with the run that produced it. A ring marks the validation-selected program, which is \dwf{} in the first run of E1. Cells are indexed by niche and novelty band only, so the feature-growth descriptor of Equation~\eqref{eq:cell} is omitted.}
\label{fig:archive-map}
\end{figure}

Algorithm~\ref{alg:archive-update} shows the session inside one run. In each round the model chooses a tool, and the backend's result stays in the session so the next round can revise. A failed check therefore costs a round but does not start a new session, and a tuning trial is never an archive member. A program enters the archive only when the session returns a search score, and it replaces an elite only within its own cell. Algorithm~\ref{alg:validation-selection} is a separate procedure that runs after the search ends. It ranks frozen programs by validation CRPS alone, so search scores and cell membership play no role. Test outcomes are opened once, after the validation choice is fixed.

\begin{algorithm}[t]
\caption{One agent session}
\label{alg:archive-update}
\begin{algorithmic}[1]
\STATE \textbf{Config:} round budget $N_{\mathrm{round}}$, tool-call budget $N_{\mathrm{call}}$, tuning budget $N_{\mathrm{tune}}$
\STATE \textbf{Agent actions} (structured tool calls that trusted backend code executes):
\STATE \textbf{propose\_formulation.} Store the fields in Table~\ref{tab:formulation-fields}, including a niche, and return a formulation ID. If a field is missing, return its name to the same session.
\STATE \textbf{stage\_candidate.} Send the Python source directly in the tool call with its formulation ID, initial setting, and parameter ranges. Store the draft. Reject a cited parent whose parsed source is unchanged.
\STATE \textbf{verify\_staged\_candidate.} Tier~0 checks execution and shape, Tier~1 checks sampled positive semidefiniteness, and Tier~2 checks contract assembly and purity. On failure, return the diagnostic to the same session so the agent can revise the draft in a later round. A failed draft does not enter the archive. Novelty affects the score after verification, but it does not cause a verification failure.
\STATE \textbf{optimize\_parameters.} Score the initial setting and up to $N_{\mathrm{tune}}-1$ Halton settings on $D_{\mathrm{tr}}$ only.
\STATE \textbf{submit\_staged\_candidate.} Repeat the Tier-2 checks at the best setting, and freeze the program if they pass.
\STATE \textbf{evaluate\_candidate.} Score that frozen program by Equation~\eqref{eq:selection-score} and attach the cell of Equation~\eqref{eq:cell}.
\STATE \textbf{query\_archive} or \textbf{inspect\_candidate.} Return stored program metadata and source to the session. The agent does not receive file-system paths.
\STATE
\FUNCTION{$\textsc{AgentSession}(\mathrm{niche}^\star,\mathrm{parent},\mathcal I,D_{\mathrm{tr}})$}
  \STATE open a fresh session showing the task, niche $\mathrm{niche}^\star$, the parent, and inspirations $\mathcal I$
  \FOR{$r=1,\ldots,N_{\mathrm{round}}$}
    \STATE query the model and run its tool calls (at most $N_{\mathrm{call}}$ in total) \hfill $\triangleright$ results return to the same session
    \IF{the reply has no tool call}
      \STATE \textbf{return} $(\varnothing,\cdot,\cdot)$ \hfill $\triangleright$ the agent gave up
    \ENDIF
    \IF{a call scored a Tier-2 program $\kappa$}
      \STATE \textbf{return} $(\kappa,\ \text{search score},\ c(\kappa))$ \hfill $\triangleright$ back to Alg.~\ref{alg:loop}
    \ENDIF
  \ENDFOR
  \STATE \textbf{return} $(\varnothing,\cdot,\cdot)$ \hfill $\triangleright$ out of rounds
\ENDFUNCTION
\end{algorithmic}
\end{algorithm}

\paragraph{Lineage classification.}
A candidate with no declared parent is an independent root. Any other candidate counts as a genuine revision only when its parsed source differs from its parent's, so a pure hyperparameter change is classified as superficial.

\paragraph{Parameter tuning.}
The \texttt{optimize\_parameters} tool uses a gradient-free search over the parameters that a program declares. For each parameter, the agent specifies a linear or logarithmic range or an ordered list of categorical values. For a Halton coordinate $u\in[0,1)$ and $J$ choices, the tuner selects choice $\min(\lfloor Ju\rfloor,J-1)$. The backend evaluates the agent's initial setting and up to $N_{\mathrm{tune}}-1$ additional settings from a deterministic Halton sequence \citep{halton1960efficiency}. It returns the setting with the best meta-training score. This procedure supports discrete parameters and agent-written programs that are not differentiable. It also gives every candidate a fixed tuning budget. The default is 16 trials, with a maximum of 64. We use 64 trials for BBO and forecasting and 32 for chemistry.

A Halton sequence distributes the settings across the declared ranges. Each parameter uses a different prime base from 2 through 37, so the tuner supports at most 12 parameters. Tuning uses only meta-training episodes. Its predictive protocol is cheaper than final evaluation because it uses fewer observations, fewer held-out points, and one episode per family. A BBO tuning trial uses $n=4+2d$ training points and 16 held-out points. A forecasting tuning trial uses a 24-month horizon. The backend repeats the Tier~2 checks on the best setting before freezing the program.

The tuner changes only parameters that the agent exposes. For example, every accepted forecasting program computes its raw spectral weights from exposed parameters. Tuning therefore determines these weights rather than retaining the LLM's initial values. Trial fitness also subtracts the heuristic complexity penalty $10^{-3}\log(1+p)$ for $p$ exposed parameters. We fixed this coefficient before the campaigns. Since $p\le12$, the penalty is at most $0.0026$ on the relative-improvement score scale. It is constant across one program's trials, so it cannot change which setting the tuner selects. It does not enter the archive selection score of Equation~\eqref{eq:selection-score}. It only shifts the tuning scores that the agent sees for programs with different parameter counts.

\paragraph{Selection score.}
For BBO, a program's score on meta-training episodes combines four relative improvements over the tuned baselines. For every lower-is-better metric $q$, the relative improvement is $(q_{\mathrm{base}}-q_{\mathrm{cand}})/q_{\mathrm{base}}$: $\Delta_{\mathrm{C}}$, the mean relative CRPS improvement over the best baseline on each episode, $\Delta_{\mathrm{AUC}}$ and $\Delta_{\mathrm{F}}$, the improvements in regret AUC and final regret over the best baseline, and $\Delta_{\mathrm{W}}$, the CRPS improvement on the worst family. The BBO score is
\begin{equation}
  s_{\mathrm{BBO}}=w_{\mathrm{C}}\Delta_{\mathrm{C}}+w_{\mathrm{AUC}}\Delta_{\mathrm{AUC}}
    +w_{\mathrm{F}}\Delta_{\mathrm{F}}+w_{\mathrm{W}}\Delta_{\mathrm{W}}
    -\Omega_{\mathrm{stab}}-\Omega_{\mathrm{nov}},
  \label{eq:selection-score}
\end{equation}
For domains without BBO diagnostics, the score is
\begin{equation}
  s_{\mathrm{pred}}=\Delta_{\mathrm{C}}-\Omega_{\mathrm{stab}}-\Omega_{\mathrm{nov}}.
  \label{eq:prediction-selection-score}
\end{equation}
The Thompson-sampling bandit of Appendix~\ref{app:metabo} uses the applicable score as its reward. The BBO score has nonnegative weights $w_{\mathrm{C}},w_{\mathrm{AUC}},w_{\mathrm{F}},w_{\mathrm{W}}$, where $\Omega_{\mathrm{stab}}=0.01\max(0,\log_{10}\mathrm{cond}(K)-4)$ penalizes ill-conditioned Gram matrices and $\Omega_{\mathrm{nov}}$ is the novelty penalty below. We use $(w_{\mathrm{C}},w_{\mathrm{AUC}},w_{\mathrm{F}},w_{\mathrm{W}})=(0.55,0.30,0.10,0.05)$, which makes predictive quality the main objective, gives regret AUC a substantial secondary role, and uses smaller terms for final regret and worst-family robustness. These values were fixed before the campaigns and were not fitted to validation or test results. The weights guide search only: they decide which programs become elites, and the bandit uses the applicable score as its reward. The frozen program that we report is chosen by validation CRPS alone (\S\ref{sec:metaeval}), so the weights cannot change it directly. The sensitivity analysis below checks how much they matter.

\paragraph{Novelty penalty.}
With the functional novelty distance $d_{\mathrm{nov}}$ in Equation~\eqref{eq:functional-novelty} and a threshold $\tau_{\mathrm{nov}}$ fixed before evaluation, the penalty is
\begin{equation}
  \Omega_{\mathrm{nov}}
  = \lambda_{\mathrm{nov}}\,\max\!\Big(0,\frac{\tau_{\mathrm{nov}}-d_{\mathrm{nov}}}{\tau_{\mathrm{nov}}}\Big)\,
    \Big(1-\min\big(\gamma\max(\Delta_{\mathrm{C}},0),\,1\big)\Big).
  \label{eq:novelty-penalty}
\end{equation}
The first factor grows linearly as the candidate moves closer to a reference kernel, up to $\lambda_{\mathrm{nov}}$. The second factor removes the penalty for a candidate that is clearly better, shrinking linearly to zero at a CRPS improvement of $1/\gamma$. A weak near-duplicate pays the full penalty, while a strong one is kept. We use $\tau_{\mathrm{nov}}=0.08$, $\lambda_{\mathrm{nov}}=0.1$, and $\gamma=5$, so the penalty vanishes at a $20\%$ improvement. The penalty has no units, because $d_{\mathrm{nov}}$ is compared with a fixed threshold and the second factor uses a relative CRPS improvement. It therefore does not depend on the raw scale of CRPS. In practice it rarely applied. Across the 14 BBO campaigns, 4 of the 160 scored discoveries paid a nonzero penalty, and the largest was $0.073$.

\paragraph{Weight sensitivity.}\label{app:weight-sensitivity}
We check how much the weights $(w_{\mathrm{C}},w_{\mathrm{AUC}},w_{\mathrm{F}},w_{\mathrm{W}})$ of the selection score $s$ in Equation~\eqref{eq:selection-score} matter by re-scoring the discovered candidates of the first BBO archive under other weightings. Nine of the ten have all four components $\Delta_{\mathrm{C}},\Delta_{\mathrm{AUC}},\Delta_{\mathrm{F}},\Delta_{\mathrm{W}}$ stored, so we re-score these nine. Besides the original weights, we try an AUC-heavy setting that swaps the original CRPS and AUC weights and keeps the final-regret and worst-family weights, an equal-weight setting, and an extreme CRPS-only setting that sets all other weights to zero. For each setting we record where \dwf{} ranks among the nine and how many of the nine would still be archive elites (Table~\ref{tab:weight-sensitivity}).

\begin{algorithm}[t]
\caption{Validation selection and sealed test evaluation}
\label{alg:validation-selection}
\begin{algorithmic}[1]
\FUNCTION{$\textsc{ValidationSelect}(\mathcal C,D_{\mathrm{val}},D_{\mathrm{test}})$}
  \FOR{each frozen program $\kappa\in\mathcal C$}
    \STATE $\widehat R_{\mathrm{val}}(\kappa)\gets$ mean CRPS on $D_{\mathrm{val}}$ \hfill $\triangleright$ structure fixed, only amplitude and noise refit
  \ENDFOR
  \STATE $\kappa^\star\gets\arg\min_{\kappa\in\mathcal C}\widehat R_{\mathrm{val}}(\kappa)$ \hfill $\triangleright$ search scores and cells play no role
  \STATE \textbf{return} $\kappa^\star$ and its score on $D_{\mathrm{test}}$ \hfill $\triangleright$ test is opened once, after the choice
\ENDFUNCTION
\end{algorithmic}
\end{algorithm}

\begin{table}[t]
\caption{Sensitivity of the first BBO archive to the selection weights. Each row re-scores the nine programs under one weighting of $(\Delta_{\mathrm{C}},\Delta_{\mathrm{AUC}},\Delta_{\mathrm{F}},\Delta_{\mathrm{W}})$ in Equation~\eqref{eq:selection-score}, using the four stored components. ``\dwf{} rank'' is the position of \dwf{} among the nine by search score, where 1 is best. ``Elites kept'' counts the programs that remain elites of their cells.}
\label{tab:weight-sensitivity}
\centering
\small
\begin{tabular}{@{}llcc@{}}
\toprule
Weighting & Weights & \dwf{} rank & Elites kept \\
\midrule
Original & $(.55,.30,.10,.05)$ & 1 & 9/9 \\
AUC-heavy & $(.30,.55,.10,.05)$ & 1 & 9/9 \\
Equal & $(.25,.25,.25,.25)$ & 1 & 9/9 \\
CRPS-only & $(1,0,0,0)$ & 1 & 9/9 \\
\bottomrule
\end{tabular}
\end{table}

\dwf{} ranks first under every weighting, and the top-three order is unchanged: \dwf{}, simple spectral warp, then reversed index-bump dual-frequency warp. All nine programs also stay elites. The CRPS-only setting keeps this order and has Kendall's $\tau=0.89$ against the original ranking. Under it, the top-three pool after three intermediate runs differs, but the final pool does not. The ordering of discovered programs therefore does not hinge on the weights we chose. This replay is not a counterfactual rerun, because different rewards could have changed parent selection and the later proposals. The reported kernel is anyway chosen by validation CRPS alone. Among all ten frozen discoveries, including the one without stored search components, \dwf{} has the lowest validation CRPS at $0.540$, followed by simple spectral warp at $0.553$.

\paragraph{Contract interfaces.}
A feature-map or pullback candidate implements a point function that the interpreter calls once per input and checks for finite outputs of one fixed dimension. A spectral candidate implements \path{spectral_components(input_dimension, parameters)} and returns frequencies together with raw weights that are computed from tuned parameters, and the interpreter squares the raw weights so any real value gives a nonnegative spectral measure.\footnote{Our task prompts describe frequencies in cycles per unit, while the interpreter uses angular frequency, $\cos(\omega^\top(x-y))$ without a factor $2\pi$. Validity is unaffected, and every reported score uses the interpreter's convention, but a registered description can misstate the period of a spectral program by this factor.} Certified source cannot use module-level state, randomness, reflection, decorators, or mutable defaults.

\subsection{Evaluation and Registration Details}
\label{app:protocol}

This section specifies how a candidate is checked, how a formulation is registered, and how the reported scores are computed. The ladder says what each tier establishes. The fingerprint and the CRPS formula then make the novelty score and the predictive metric fully explicit.

\paragraph{Evidence ladder.}

\begin{table}[t]
\caption{Implemented evidence ladder.}
\label{tab:tiers}
\centering
\small
\begin{tabular}{@{}p{0.07\linewidth}p{0.20\linewidth}p{0.60\linewidth}@{}}
\toprule
T & Label & What is established \\
\midrule
0 & Executable & Sandboxed execution succeeds and the output matches the interface \\
1 & Empirical & Randomized Gram-matrix tests pass on sampled input sets \\
2 & Contract certified & Checked interface conformance and trusted PSD-preserving assembly under the stated component assumptions \\
\bottomrule
\end{tabular}
\end{table}

\paragraph{Tier 0 and Tier 1.}
Tier 0 runs each candidate in isolated subprocess trials, with 8 trials of 12 points drawn from a seeded $\mathcal{N}(0,I)$ for each declared input dimension. Tier 1 adds an empirical PSD test on each trial's Gram matrix $K$,
\begin{equation}
  \lambda_{\min}\!\left(\tfrac{1}{2}(K+K^\top)\right) \ge -\eta, \qquad
  \max_{i,j}\lvert K_{ij}-K_{ji}\rvert \le \eta
  \label{eq:tier1}
\end{equation}
with tolerance $\eta=10^{-8}$, a numerical check of the PSD condition in Equation~\eqref{eq:psd}. A failed execution stops the candidate at Tier 0, and a failure of Equation~\eqref{eq:tier1} stops it at Tier 1.

\paragraph{Contract-conformance checks.}
Tier 2 also tests the five metamorphic relations in Table~\ref{tab:contract-checks}, with 2 trials of 12 points for each declared dimension. For example, if $X'$ appends new points to $X$, the upper-left block of $K(X',X')$ must equal $K(X,X)$. For each relation, we compare the interpreter's Gram matrix $K$ on the transformed batch with the matrix $\hat K$ implied by its Gram matrix on the original batch (for example a permuted or truncated copy), and we require
\begin{equation}
  \mathrm{err}(K,\hat K) = \frac{\max_{i,j}\lvert K_{ij}-\hat K_{ij}\rvert}
  {\max\big(\max_{i,j}\lvert K_{ij}\rvert,\,1\big)}
  \le \varepsilon,
  \label{eq:contract-error}
\end{equation}
with $\varepsilon=1.1\times10^{-7}$.

\begin{table}[t]
\caption{Tier 2 contract-conformance relations. See Equation~\eqref{eq:contract-error}.}
\label{tab:contract-checks}
\centering
\small
\begin{tabular}{@{}L{0.24\linewidth}L{0.66\linewidth}@{}}
\toprule
Relation & What must match \\
\midrule
Determinism & The same input, replayed twice \\
Permutation & The Gram of a permuted batch equals the permuted Gram \\
Subset & The Gram of a random subset equals the corresponding submatrix \\
Extension & Appending then truncating points reproduces the original block \\
Duplicate & A duplicated point reproduces its own row \\
\bottomrule
\end{tabular}
\end{table}

\paragraph{Sandbox implementation.}
Each Gram evaluation runs in its own subprocess, under a 20\,s wall-clock limit, 1024\,MB of memory, 15\,s of central processing unit (CPU) time, 64 open files, and 1\,MB of captured output (truncated). The environment is restricted to single-threaded Basic Linear Algebra Subprograms (BLAS), and the hash seed is fixed so a replay follows the same code paths. The sandbox keeps a failing program from taking down the rest of the run. Code from an untrusted third party would need a container backend for stronger isolation.

\paragraph{Formulation fields.}
Before any code is submitted (\S\ref{sec:functional-novelty}), the agent registers the 14 fields in Table~\ref{tab:formulation-fields}. The backend stores these strings and returns a formulation identifier. The later staging call sends that identifier with the Python source, so it does not repeat the strings or pass a file path. The fields state the agent's intended construction and predicted behavior. They are not verification results. The backend uses only the declared niche and feature growth to index the archive. Researchers use the other fields to inspect novelty claims and compare intended behavior with the frozen program. Trusted execution determines acceptance and numerical scores. Appendix~\ref{app:agent-prompts} shows the exact tool schema and examples. Discovery mode excludes a trivial rediscovery of a closure baseline. Closure may still combine a genuinely new certified primitive with a safe baseline.

\begin{table}[t]
\caption{Formulation fields registered before code submission (\S\ref{sec:functional-novelty}).}
\label{tab:formulation-fields}
\centering
\small
\begin{tabular}{@{}L{0.42\linewidth}L{0.52\linewidth}@{}}
\toprule
Field & Purpose \\
\midrule
\path{mathematical_form} & Kernel expression and construction \\
\path{psd_argument} & Proof sketch for positive semidefiniteness \\
\path{novelty_claim} & What is new relative to known kernels \\
\path{closest_known_kernel} & Nearest baseline and how this differs \\
\path{expected_bo_behavior} & Predicted BBO performance characteristics \\
\path{construction_niche} & Declared construction category \\
\path{equivalence_analysis} & Why this is not equivalent to a known kernel \\
\path{optimized_parameters} & Suggested initial parameter values \\
\path{domain_scale_analysis} & Input-scale and periodicity considerations \\
\path{irrelevant_coordinate_mechanism} & How irrelevant inputs are handled \\
\path{cross_coordinate_mechanism} & How input coordinates interact \\
\path{falsification_test} & A test that would disconfirm the proposal \\
\path{feature_growth} & How feature dimension scales with input \\
\path{expected_conditioning} & Expected Gram matrix condition number \\
\bottomrule
\end{tabular}
\end{table}

\paragraph{Behavioral fingerprint.}
Equation~\eqref{eq:functional-novelty} compares kernels through the normalized centered Gram matrices $\bar K_{\kappa}(X_d)$. For each probe dimension $d\in\mathcal{D}=\{2,5,8\}$, the probe set $X_d$ is a fixed set of $n_{\mathrm{probe}}=16$ points drawn uniformly from $[0,1]^d$ with a fixed seed, so every kernel is compared on the same inputs. Let $K_{\kappa}(X_d)\in\mathbb{R}^{n\times n}$ be the Gram matrix of $\kappa$ on $X_d$, and let $H=I-\frac{1}{n}\mathbf{1}\mathbf{1}^{\top}$ be the centering matrix. Then
\begin{equation}
  \bar K_{\kappa}(X_d)=\frac{H K_{\kappa}(X_d) H}
  {\lVert H K_{\kappa}(X_d) H\rVert_F}.
  \label{eq:fingerprint}
\end{equation}
For centered Frobenius norm below $10^{-12}$, the implementation returns the zero matrix. Otherwise, squared fingerprint distance is two minus twice centered kernel alignment~\citep{cortes2012alignment}. Centering removes the constant direction and normalization removes positive amplitude. This is a covariance-shape diagnostic, not a guarantee of predictive equivalence.

\paragraph{Continuous ranked probability score.}
For a Gaussian predictive distribution $\mathcal{N}(\mu,\sigma^2)$ and an observed value $y$, write $z=(y-\mu)/\sigma$, let $\phi$ be the standard normal density, and let $\Phi$ be the standard normal cumulative distribution function (CDF). The closed form \citep{gneiting2007strictly} is
\begin{equation}
  \mathrm{CRPS}\big(\mathcal{N}(\mu,\sigma^2),y\big)
  = \sigma\Big[z\big(2\Phi(z)-1\big)+2\phi(z)-\tfrac{1}{\sqrt{\pi}}\Big].
  \label{eq:crps}
\end{equation}
We report the mean of Equation~\eqref{eq:crps} over held-out points, taking the GP posterior mean and variance at each point as $\mu$ and $\sigma^2$.

\subsection{Soundness of Validation Selection}
\label{app:selection-guarantee}

Search proposes and tunes programs using meta-training episodes only. Validation then chooses one of the resulting frozen programs (\S\ref{sec:metaeval}). Because the validation episodes play no role in search, choosing among the frozen programs is an ordinary model-selection problem over a finite set, and the standard argument applies.

\paragraph{Setup.} Let $\kappa_1,\dots,\kappa_N$ be the $N$ frozen programs that enter validation, each with its search-time parameters fixed. Let $\mathcal{E}$ be the distribution of meta-validation episodes, which first picks a family uniformly at random and then draws an episode from it. For an episode $e$, let $\mathcal{L}(\kappa,e)$ be the score of $\kappa$ after the backend refits amplitude and noise on the episode's training points, for example the mean held-out CRPS of Equation~\eqref{eq:crps}. Lower scores are better. Proposition~\ref{prop:selection} assumes that every score lies in $[0,\mathcal{L}_{\max}]$. The reported evaluator does not clip CRPS, so the proposition is a conditional model-selection result rather than a certificate for the realized finite sample. Validation uses $m$ independent episodes $e_1,\dots,e_m$, with the same number drawn from each family. Define
\begin{equation*}
  R(\kappa) = \mathbb{E}_{e\sim\mathcal{E}}\bigl[\mathcal{L}(\kappa,e)\bigr],
  \qquad
  \widehat{R}(\kappa) = \frac{1}{m}\sum_{j=1}^{m}\mathcal{L}(\kappa,e_j),
  \qquad
  \hat\kappa \in \operatorname*{arg\,min}_{i\le N}\widehat{R}(\kappa_i).
\end{equation*}
$R(\kappa)$ is the expected score on a fresh validation episode, $\widehat{R}(\kappa)$ is its empirical estimate, and $\hat\kappa$ is the program that validation selects.

\begin{proposition}[Validation selection]
\label{prop:selection}
Suppose the candidates $\kappa_1,\dots,\kappa_N$ are independent of the meta-validation episodes. Then for every $\delta\in(0,1)$, with probability at least $1-\delta$,
\begin{equation*}
  R(\hat\kappa) \;\le\; \min_{i\le N} R(\kappa_i)
  + \mathcal{L}_{\max}\sqrt{\frac{2\log(2N/\delta)}{m}}.
\end{equation*}
\end{proposition}
\begin{proof}
Condition on the candidates. For each $i$, the scores $\mathcal{L}(\kappa_i,e_1),\dots,\mathcal{L}(\kappa_i,e_m)$ are independent and lie in $[0,\mathcal{L}_{\max}]$, and their mean has expectation $R(\kappa_i)$ because every family contributes equally. Hoeffding's inequality gives $\Pr\bigl(|\widehat{R}(\kappa_i)-R(\kappa_i)|\ge t\bigr)\le 2e^{-2mt^2/\mathcal{L}_{\max}^2}$. A union bound over the $N$ candidates with $t=\mathcal{L}_{\max}\sqrt{\log(2N/\delta)/(2m)}$ gives $|\widehat{R}(\kappa_i)-R(\kappa_i)|<t$ for all $i$ with probability at least $1-\delta$. On this event, for $i^\star\in\arg\min_i R(\kappa_i)$, $R(\hat\kappa) < \widehat{R}(\hat\kappa)+t \le \widehat{R}(\kappa_{i^\star})+t < R(\kappa_{i^\star})+2t$. The bound holds for every candidate set, and hence unconditionally.
\end{proof}

Proposition~\ref{prop:selection} applies the standard generalization bound for empirical risk minimization over a finite hypothesis class \citep[Chapter~4]{shalevshwartz2014understanding}. The bound depends on the number of programs that reach validation. It does not depend on program length, the size of the program space, or the number of search proposals. If selection reused the search episodes, the guarantee would need to cover every program that the proposer could generate. Such bounds grow with program length \citep{singhal2026llmpriors}. Separate search and validation splits therefore preserve the finite-selection guarantee despite the open-ended program space. The same condition explains why a kernel selected in hindsight on evaluation data can have an optimistic score and fail to transfer. The best discovery on SF$_6$ provides one example (\S\ref{sec:results-ts}).

\paragraph{The cost of validation.} The gap grows only as $\sqrt{\log N}$ in the number of candidates but shrinks as $1/\sqrt{m}$ in the number of validation episodes. Adding candidates is therefore cheap for the bound, while tightening it requires more validation episodes. Validation episodes are not free. In BO, scoring a program on an episode can require a full rollout of several acquisition steps on a fresh task, and in a real application every such episode costs function evaluations.

\paragraph{What the guarantee covers.}
\begin{itemize}
  \item \emph{Families.} The guarantee applies to the meta-validation families. Transfer to the withheld meta-test families is established empirically (\S\ref{sec:results}).
  \item \emph{Independence.} The assumption holds for BBO, whose episodes use disjoint seeds. It holds approximately for forecasting, whose windows partly overlap.
  \item \emph{Selection only.} The guarantee covers the choice among proposed programs. How often the proposer produces a strong candidate is measured empirically, since \method's proposals depend on the archive and are not independent draws.
  \item \emph{Constants.} The bound is loose at our sample sizes. For the BBO campaign, with $N=10$, $m=50$ and $\delta=0.05$, the gap term is about $0.49\mathcal{L}_{\max}$. The paired intervals of \S\ref{sec:results} are what show which kernel is better.
\end{itemize}

\paragraph{Why a held-out score rather than the marginal likelihood.} The log marginal likelihood measures how probable the training data are under the prior, but this training-data criterion can favor a prior that explains observed inputs without producing the best posterior predictions on new inputs \citep{lotfi2023bayesian}. \method therefore uses the marginal likelihood only to refit amplitude and noise within an episode, and selects programs by held-out CRPS on episodes that search never saw.

\subsection{Randomized Screening Stress Test}
\label{app:screening-stress}

This experiment asks whether a finite numerical screen can stand in for a validity proof. We fixed the candidate generator, the public screen, and a withheld suite before looking at any outcome. The main counts include all parseable programs produced by the model, without manual selection. Diagnostic canaries and certified controls then check that the protocol catches known blind spots and still accepts valid constructions.

\paragraph{Unrestricted candidate generation.}
We asked the ensemble proposer for 40 kernel programs, in eight independent calls of five candidates each. Each response supplied Python source implementing \texttt{kernel\_matrix(x, parameters)} with only NumPy or \texttt{math}, together with a parameter dictionary and a short rationale. The prompt asked for mathematically diverse kernels on inputs of unknown dimension, and it said that acceptance would use a finite randomized PSD screen. It did not reveal the withheld tests, and it did not require a feature-map, closure, pullback, or spectral construction. We kept all 40 parseable responses, with no manual selection. We ran this generation twice with the same prompt, and once with a formulation-first prompt that additionally required a complete proof sketch, before the code, that each program returns a symmetric PSD matrix for every input set, size, and dimension, with each entry depending only on its own pair of points. This mirrors the formulation step of the search (\S\ref{sec:functional-novelty}). Each execution ran in a sandbox with a 20-second wall-clock limit, 15 CPU seconds, 1~GB of memory, and a 1~MB output limit.

\paragraph{Frozen randomized screen.}
The Tier-1 screen evaluates each program on eight independent matrices $X\in\mathbb{R}^{12\times 2}$ drawn from $\mathcal{N}(0,I)$. For each resulting Gram matrix $K$, it requires
\begin{equation}
  \begin{aligned}
  \max_{i,j}|K_{ij}-K_{ji}| &\leq 10^{-8},\\
  \lambda_{\min}\!\left((K+K^\top)/2\right) &\geq -10^{-8}.
  \end{aligned}
  \label{eq:stress-screen}
\end{equation}
A program must also execute successfully and return the required matrix shape on all eight trials. We freeze the program bank after this screen, and only the programs that pass proceed to the unseen suite.

\paragraph{Unseen inputs and consistency checks.}
The withheld suite contained 84 Gram-matrix cases, spanning input dimensions $d\in\{1,2,3,5,8,16,32\}$ and matrix sizes from 5 to 64. It included standard-normal and uniform inputs, coordinate scales from $10^{-4}$ to $10^4$, large translations, and duplicated, nearly duplicated, collinear, and tightly clustered points. Forty of the 84 cases were randomized mixtures of dimension, matrix size, distribution, scale, and translation. We applied the symmetry and eigenvalue criteria of Equation~\eqref{eq:stress-screen} to every case.

We also checked determinism, permutation equivariance, and consistency under subsetting, extension, and duplication, in dimensions 2 and 8. These relations are necessary for a program to define a pairwise kernel. In particular, if $X'$ is formed by appending $Z$ to $X$, the upper-left block of $K(X',X')$ must equal $K(X,X)$, because the covariance assigned to a fixed pair cannot depend on unrelated points elsewhere in the batch. We marked relative discrepancies above $10^{-7}+10^{-8}$ as failures.

\paragraph{Controls and canaries.}
Four Tier 2 controls cover the four certified interfaces, and five diagnostic canaries were written to pass the public draws and then fail on an untested dimension, size, range, duplicated point, or batch-dependent normalization. The canaries are excluded from the headline counts.

\begin{figure*}[t]
\centering
\includegraphics[width=0.95\textwidth]{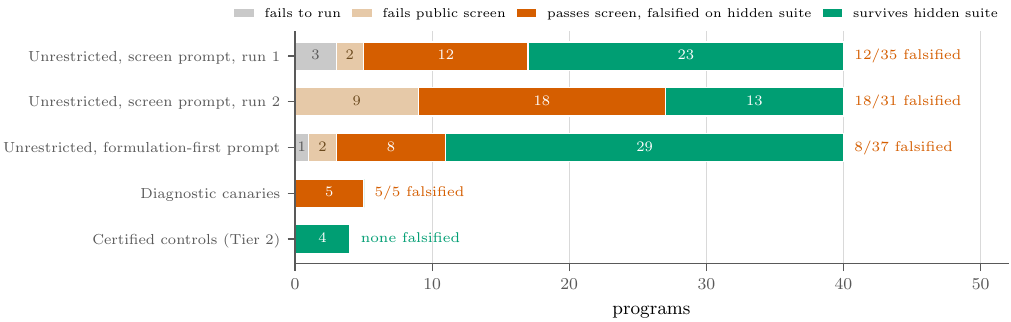}
\caption{\textbf{Finite numerical screening admits invalid programs.} Outcome of each cohort of kernel programs, split into programs that fail to run, fail the public Tier-1 screen, pass it but are falsified on the hidden suite, and survive the hidden suite. Labels give the number falsified among the programs that passed the public screen. Canaries were written to pass the public screen and fail the hidden suite. Certified controls pass Tier~2. Unrestricted programs that survive the hidden suite are empirically unfalsified, not certified.}
\label{fig:stress-funnel}
\end{figure*}

\paragraph{Results and failure modes.}
The hidden suite exposes failures that the public screen misses (Figure~\ref{fig:stress-funnel}). In the first run, of 40 unrestricted programs, three fail to execute and two produce an indefinite Gram matrix during the public screen, leaving 35 Tier-1 passes. The hidden suite falsifies 12 of these. Six produce a non-PSD matrix, six violate a kernel-consistency relation, and two fail on an unseen shape. Two programs fail in more than one category.

The diagnostic cohorts behaved as intended in every run. All five canaries passed Tier~1 and were then falsified, and all four certified controls survived all 84 input cases and all consistency checks. The experiment demonstrates the practical limitation of the finite screen. In the first run it admitted 12 programs whose failures showed up only on a broader, and still finite, evaluation. The unrestricted programs that survived remain empirically unfalsified, and they are not certified.

\paragraph{Repeated generation and a formulation-first prompt.}
The second run with the same prompt admits 18 invalid programs among 31 passes, so the failure rate varies across generations. Of these 18 programs, 16 violate a consistency relation, usually because they normalize by statistics from the whole batch. Requiring a PSD argument before code reduces the failure rate to 8 of 37. The reduction is significant against the second run alone (Fisher exact $p=0.003$) and against both screen runs pooled ($p=0.02$). However, the formulation-first prompt does not eliminate invalid programs. Most remaining failures are implementation errors beneath a correct argument, such as parameter vectors sized for two input dimensions. Thus, a correct PSD argument does not make the submitted code a valid kernel on the full domain. Trusted assembly closes this gap.

\paragraph{Live screening test.}
We also test whether relaxing acceptance changes the live search. We rerun the BBO ablation protocol with a Tier-1 acceptance gate, keeping the proposer, context, and construction contracts identical to the full-archive arm. Across five seeded campaigns and 49 accepted candidates, the relaxed gate admits no extra program. Every accepted candidate also passes Tier~2.

The program bank and live search therefore answer different questions. The bank shows that finite numerical screening can accept invalid programs, and that a PSD argument written before the code does not prevent this. The live search admits no such program because its task context offers only the contract interfaces. Every accepted program therefore supplies point functions or frequencies, the interpreter assembles the Gram matrix, and the relaxed gate has nothing extra to admit. Appendix~\ref{app:search-ablation} gives the search protocol.

\paragraph{Live search with unrestricted programs admitted.}
We then extend the task context so that the agent may also submit an unrestricted \texttt{kernel\_matrix} program, accepted once it passes the Tier-1 screen. This arm uses the ensemble proposer. Across three seeded campaigns, the agent produces 35 accepted programs, and every one uses a construction contract. The task context describes the contract interfaces in more detail than the unrestricted option, which may favor contracts. On the withheld suite, no accepted program yields an indefinite matrix or violates a consistency relation. Six raise an error on a 16-dimensional input, twice the benchmark's maximum input dimension, because they size parameter arrays for at most eight coordinates. The interpreter then rejects the call instead of returning a matrix. Under the contracts, failures are therefore loud and confined to inputs outside the benchmark's domain, whereas the unrestricted programs in the bank fail silently. The three validation-selected kernels reach test CRPS $0.581$, $0.604$, and $0.610$, each below the fixed Mat\'ern-5/2 reference ($0.612$).

\section{Prediction and Optimization on BBO Benchmarks}
\label{app:bbo}

This appendix supports \S\ref{sec:results-bo}. It gives the shared function-family benchmark and search configuration, the predictive CRPS results, the separate GP-based BO protocol and regret results with paired tests, and the controls against meta-learned deep kernels and CAKE. It also reports the search ablations, the frontier proposer, and the robustness checks on a broader benchmark.

\subsection{Protocol and Search Configuration}
\label{app:metabo}

This section gives the shared task split, the predictive CRPS protocol used in \S\ref{sec:results-bo}, and the separate GP-based BO protocol evaluated by regret. It also specifies the baseline grids and the ensemble used for these BBO benchmark experiments.

\paragraph{Task families and splits.}
Training and validation use Branin ($d=2$), Ackley ($d=2$), Cosine ($d=8$), Hartmann-6 ($d=6$), and Levy ($d=6$). Test-only families are Bukin ($d=2$),
Drop-Wave ($d=2$), Griewank ($d=5$), H\"older Table ($d=2$), Rosenbrock ($d=6$), and
Rastrigin ($d=6$).

Episodes apply optimum-preserving randomizations drawn from a seeded generator. The seed of episode $j$ of family $i$ is $o+1000i+j$, with offset $o=0$ for training, $10^4$ for validation, and $2\times10^4$ for test, so the three splits never share random draws.
The randomizations are coordinate permutations, reflections, per-coordinate powers in $[0.75,1.25]$, an
invertible triangular cross-coordinate coupling with strengths in $[-1.5,1.5]$, and
up to two appended irrelevant coordinates (observed dimension capped at eight).
No observation noise is added. Training and test targets are the clean objective values, standardized by the training sample. The GP fit then chooses a diagonal noise variance from $\{10^{-4},10^{-3},10^{-2},10^{-1}\}$ by training marginal likelihood in Equation~\eqref{eq:amp-noise}. That fitted variance is a model parameter. The benchmark targets themselves stay clean. Table~\ref{tab:splits} summarizes the split under the meta-evaluation framework of \S\ref{sec:metaeval}.

\begin{table}[t]
\caption{Task-family instantiation and split (\S\ref{sec:metaeval}, \S\ref{sec:domains}). Meta-training and meta-validation use the same five families with disjoint seed offsets ($0$ and $10^4$). Meta-test uses six families never seen during search, at offset $2\times10^4$. Each family contributes 10 seeded episodes per split, giving 50 meta-training, 50 meta-validation, and 60 meta-test tasks. Targets are clean function values.}
\label{tab:splits}
\centering
\small
\begin{tabular}{@{}lcc@{}}
\toprule
Family & $d$ & Split(s) \\
\midrule
Branin & 2 & train, val \\
Ackley & 2 & train, val \\
Cosine & 8 & train, val \\
Hartmann-6 & 6 & train, val \\
Levy & 6 & train, val \\
\midrule
Bukin & 2 & test \\
Drop-Wave & 2 & test \\
Griewank & 5 & test \\
H\"older Table & 2 & test \\
Rosenbrock & 6 & test \\
Rastrigin & 6 & test \\
\bottomrule
\end{tabular}
\end{table}

\paragraph{Predictive protocol.}
Each episode fits on $n=8+2d$ uniformly sampled observations and evaluates 32 held-out
test points, and targets are standardized by training statistics. Every kernel is used in
$y=f(x)+\varepsilon$, where $f\sim\mathcal{GP}(0,\sigma_f^2\kappa_\theta)$ and
$\varepsilon\sim\mathcal{N}(0,\sigma_n^2)$. Candidate-specific hyperparameters $\theta$ remain
frozen, while the outer amplitude and noise are fit by
\begin{equation}
  (\hat\sigma_f^2,\hat\sigma_n^2)=\operatorname*{arg\,max}_{\sigma_f^2\in \mathcal{G}_f,\ \sigma_n^2\in \mathcal{G}_n}\ \log p(\mathbf{y}\mid X,\sigma_f^2\kappa_\theta,\sigma_n^2),
  \label{eq:amp-noise}
\end{equation}
with $\mathcal{G}_f=\{0.1,0.3,1,3\}$ and $\mathcal{G}_n=\{10^{-4},10^{-3},10^{-2},10^{-1}\}$. The predictive evaluation metric is mean held-out CRPS in Equation~\eqref{eq:crps}. Final kernel selection minimizes this metric on meta-validation episodes. The combined search-time score is defined separately in Equation~\eqref{eq:selection-score}.

\paragraph{BBO protocol.}
This optimization evaluation is separate from the predictive protocol above. Each rollout starts from $4+2d$ uniformly sampled points and takes $T=8$ expected improvement (EI) steps. At each step, we draw 128 fresh candidates uniformly from $[0,1]^d$ and choose the one with the largest EI. The backend refits amplitude and noise by Equation~\eqref{eq:amp-noise}. Let $f^\star$ be the task maximum and $y_t^\star$ the best value observed after step $t$. We report
\begin{equation}
 r_t=\frac{\max(f^\star-y_t^\star,0)}{\max(f^\star-y_0^\star,10^{-8})},
 \qquad \mathrm{RegretAUC}=\frac1T\sum_{t=1}^T r_t,
 \label{eq:regret-auc}
\end{equation}
where lower is better. Search-time rollouts instead use three EI steps and 64 fresh candidates per step. Their regret AUC and final regret contribute to the combined meta-training score in Equation~\eqref{eq:selection-score}, alongside held-out predictive CRPS. The final program selection uses validation CRPS alone.

\paragraph{Baselines.}
Fixed-kernel grids cover lengthscales $\{0.1,0.25,0.5,1\}$ for RBF, Mat\'ern-3/2, and Mat\'ern-5/2, with periodic, rational-quadratic, ARD, linear, and spectral-mixture variants alongside. The winner of each family by meta-training fitness enters the frozen comparison (Table~\ref{tab:metaresults}). Input warping \citep{snoek2014inputwarping}, a grid spectral mixture \citep{suwandi2025gsmp}, and random Fourier features \citep{rahimi2007random} are fit by marginal likelihood and archived as Tier 2 constructions under the same acceptance checks.

\paragraph{LLM ensemble and bandit.}
The ensemble uses three LLM backbones, GLM-5.2, DeepSeek-V4-Pro, and Qwen-3.8 Max. Each backbone runs inside the same tool-calling agent harness (Appendix~\ref{app:prompt-interface}). The CAKE baseline of \S\ref{sec:baselines} uses Qwen-3.8 Max. The ensemble beats its Qwen member used alone (Qwen-only test CRPS $0.616$ against the ensemble mean $0.593$, Appendix~\ref{app:revision-saved}), while a single frontier model beats the ensemble (Appendix~\ref{app:proposer-scaling}). A Thompson-sampling bandit selects one model per discovery run, and it keeps that choice for every round inside the run. Let $s_t$ be a run's best selection score in Equation~\eqref{eq:selection-score}. The running min-max normalized reward is
\begin{equation}
  \tilde s_t = \frac{s_t - \min_{t' \le t} s_{t'}}{\max_{t' \le t} s_{t'} - \min_{t' \le t} s_{t'}} \in [0,1]
  \label{eq:bandit-reward}
\end{equation}
We set $\tilde s_t=\tfrac12$ when the denominator is zero. For arm $i$, let $W_i$ be the most recent 20 runs on that model. The Beta counts are
\begin{equation}
  \alpha_i = \alpha_0 + \sum_{t \in W_i} \tilde s_t, \qquad
  \beta_i = \beta_0 + \sum_{t \in W_i} (1 - \tilde s_t),
  \label{eq:bandit-posterior}
\end{equation}
with uniform priors $\alpha_0=\beta_0=1$. The bandit draws $\vartheta_i \sim \mathrm{Beta}(\alpha_i,\beta_i)$ for every arm and selects the arm with the largest draw for the next run. A round-robin warmup continues until every arm has completed at least 2 runs, and Thompson sampling starts after that.
\paragraph{Search configuration and budget.}
Proposals come from the three-model ensemble, selected by the bandit
above. Twelve discovery runs (each
capped at ten rounds and one hundred tool calls) produced ten Tier 2 verified pullback
programs, consuming 482 backend parameter-tuning trials of up to 64 per program, while
the other two runs terminated without a verified submission. Every accepted program
passed Tier 2 verification before submission, no evaluation failed on any task, and all
ten appear in the full comparison in Table~\ref{tab:fullresults}. A predictive
evaluation costs about one second per candidate, and a search-time rollout suite about
three seconds. Table~\ref{tab:agent-cost} in Appendix~\ref{app:ts} reports the
agent-loop wall-clock, tool-call, and token cost per discovery run.

\subsection{Full Frozen-Benchmark Results}
\label{app:fullresults}

Table~\ref{tab:fullresults} lists validation and test values alongside Figure~\ref{fig:bbo-results}.

Table~\ref{tab:fullresults} lists all evaluated candidates and learned-feature references, grouped by method family. Figure~\ref{fig:bbo-results} summarizes their test performance. Its bars show paired 95\% bootstrap intervals for differences from \dwf{}, rather than intervals for each method's absolute mean. Table~\ref{tab:wilcoxon} provides the paired tests.

The discovered candidates have test CRPS values from $0.601$ to $0.676$, so discovery does not guarantee that every accepted program transfers well. Validation selects one program before testing. In this campaign, it selects \dwf{}, which also has the lowest test CRPS among the discoveries.

\begin{table*}[t]
\caption{Frozen-benchmark candidates and learned-feature references, with validation
and test CRPS and regret AUC\@. Lower is better. Bold marks the best
value within each block.}
\label{tab:fullresults}
\label{tab:metaresults}
\centering
\small
\setlength{\tabcolsep}{6pt}
\begin{tabular}{@{}lcccc@{}}
\toprule
 & \multicolumn{2}{c}{CRPS $\downarrow$} & \multicolumn{2}{c}{AUC $\downarrow$} \\
\cmidrule(lr){2-3}\cmidrule(lr){4-5}
Candidate & val & test & val & test \\
\midrule
Linear & 0.783 & 0.791 & 0.913 & 0.872 \\
Periodic ($\ell{=}1$, $p{=}2$) & \textbf{0.549} & 0.618 & 0.816 & \textbf{0.678} \\
Rational quadratic (RQ; $\ell{=}0.25$, $\alpha{=}1$) & 0.559 & 0.618 & \textbf{0.785} & 0.706 \\
Mat\'ern-5/2 ($\ell{=}0.5$) & 0.551 & \textbf{0.612} & 0.797 & 0.693 \\
RBF ($\ell{=}0.25$) & 0.565 & 0.641 & 0.823 & 0.756 \\
ARD Mat\'ern-5/2 & 0.567 & 0.626 & 0.830 & 0.752 \\
ARD Mat\'ern-3/2 & 0.569 & 0.624 & 0.817 & 0.713 \\
\midrule
Spectral mixture & 0.629 & 0.709 & 0.869 & 0.784 \\
Grid spectral mixture (GSM) & 0.641--0.787 & 0.674--0.853 & 0.837--0.963 & 0.835--0.928 \\
Input warping & 0.566--0.585 & 0.616--0.645 & 0.796--0.847 & 0.695--0.759 \\
Random Fourier features  & 0.582--0.597 & 0.636--0.661 & 0.811--0.875 & 0.669--0.762 \\
CAKE (best of 50) & 0.554 & 0.609 & 0.803 & 0.708 \\
\midrule
Dual warp--fold (\dwf{}) & \textbf{0.540} & \textbf{0.601} & \textbf{0.789} & \textbf{0.671} \\
Simple spectral warp & 0.553 & 0.618 & 0.810 & 0.694 \\
Warp multiscale bump & 0.581 & 0.636 & 0.840 & 0.803 \\
Warp index-bump dual-freq (rev.) & 0.584 & 0.634 & 0.816 & 0.728 \\
Dual spectral warp interaction & 0.598 & 0.644 & 0.799 & 0.752 \\
Warp local-bump phase & 0.600 & 0.646 & 0.817 & 0.759 \\
Warp triangular phase-spread & 0.615 & 0.654 & 0.824 & 0.738 \\
Warp index-bump dual-freq & 0.622 & 0.653 & 0.831 & 0.778 \\
Warp spectral interaction & 0.633 & 0.662 & 0.815 & 0.817 \\
Warp spectral fold & 0.647 & 0.676 & 0.958 & 0.936 \\
\bottomrule
\end{tabular}
\end{table*}

\subsection{Meta-Learned Deep-Kernel Surrogates}
\label{app:fsbo}

Our BBO protocol meta-trains on some function families and tests on others, which is also the setting of meta-learned GP surrogates. We therefore compare against FSBO \citep{wistuba2021fsbo}, which meta-learns a deep kernel across related tasks and fine-tunes it on each new task. Where \method returns a short program, FSBO returns a neural feature map, so the comparison asks how much meta-training data a learned network needs to match a discovered kernel.

\paragraph{Protocol.}
We use the official FSBO implementation\footnote{\url{https://github.com/machinelearningnuremberg/FSBO}, the code released with \citet{pineda2021hpob}.} without modification. It feeds a four-layer, 32-unit ReLU network into a Mat\'ern-$5/2$ GP, meta-trains by marginal likelihood with Adam at learning rate $10^{-4}$ and batch size 64, and keeps the checkpoint with the best meta-validation loss. Three adaptations connect it to our benchmark. Inputs are zero-padded to eight coordinates, the benchmark maximum, because the network has a fixed input size. Meta-training tasks are meta-training episodes, meta-validation tasks are meta-validation episodes, and the six test families stay unseen. Finally, two data budgets bracket the comparison. The \emph{matched} budget uses exactly the ten meta-training episodes that search scored, 518 points in total. Because these tasks hold only 44 to 56 points each, the batch is capped at 22 so that validation support and query sets stay disjoint. The \emph{rich} budget gives FSBO its natural regime of 250 episodes with 256 points each, 64{,}000 points in total, at the official batch size. Its checkpoint-selection tasks draw 256 points from each meta-validation episode, as the official validation loop requires more points than the batch. We train five seeds per budget and select one by meta-validation CRPS before testing.

We evaluate each selected model in two modes. The \emph{frozen} mode fixes the learned feature map and lengthscale and refits only amplitude and noise on the benchmark grid, the protocol that every discovered kernel follows. The \emph{native} mode follows FSBO's own test procedure. It fine-tunes the network from the meta-trained checkpoint for up to 100 epochs on the observed points before every prediction and every BO step, and it chooses queries by EI over the same candidate pools. Episodes and pools reuse the benchmark's seeded generators.

\begin{figure}[t]
\centering
\includegraphics[width=\textwidth]{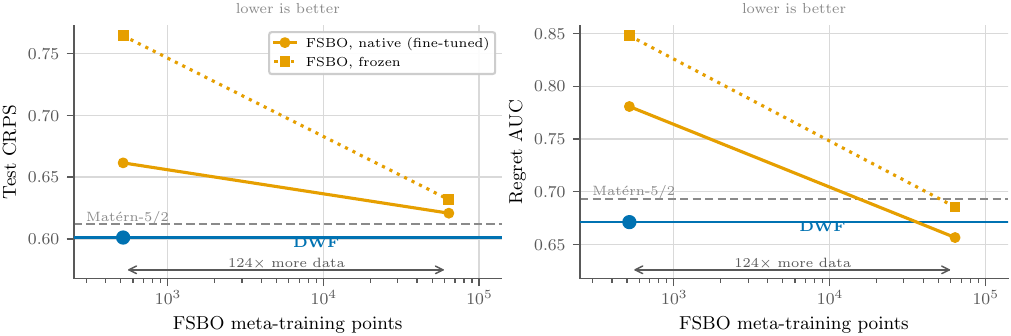}
\caption{\textbf{FSBO needs about two orders of magnitude more meta-training data to match \dwf{}.} Test CRPS (left) and regret AUC (right) on the 60 BBO test tasks, lower is better. FSBO is meta-trained on the 518 points that search scored (matched budget) or on 64{,}000 points (rich budget). Solid lines show native FSBO, which is fine-tuned on every episode. Dotted lines show frozen FSBO, which refits only amplitude and noise. \dwf{} (blue) is archived with 518 points, and the dashed line is fixed Mat\'ern-5/2. Lines join the two measured FSBO budgets and do not interpolate.}
\label{fig:fsbo}
\end{figure}

\paragraph{With the same data, \dwf{} is clearly stronger.}
Meta-trained on the episodes that search used, FSBO trails \dwf{} in both modes and on both metrics (Figure~\ref{fig:fsbo}, two-sided paired Wilcoxon $p\le0.005$ for all four contrasts). Frozen FSBO is worse than every fixed kernel except the linear one. Fine-tuning on each episode helps, but it does not close the gap (CRPS $0.6615$ against $0.6011$, regret AUC $0.7807$ against $0.6712$).

\paragraph{With 124 times more data, FSBO reaches parity.}
Under the rich budget, FSBO in native mode reaches test CRPS $0.6208$ and regret AUC $0.6566$. Neither differs significantly from \dwf{} ($p=0.53$ and $p=0.25$), and the frozen model is likewise not separated ($p=0.18$ and $p=0.93$). A meta-learned network thus needs about two orders of magnitude more meta-training data to match a single discovered program, and it does not surpass it. The frozen comparison also shows where the difference arises. Rich FSBO fits the training families about as well as \dwf{} (validation CRPS $0.538$ against $0.540$), but it transfers less well to the held-out families ($0.632$ against $0.601$). A run at batch size 22, which we used before adopting the official setting, gave weaker FSBO results (native CRPS $0.6439$, regret AUC $0.6742$), so the table reports the stronger configuration.

\subsection{Repeated Discovery and Search-Budget Controls}
\label{app:revision-saved}

We report seven ensemble campaigns, comprising the initial search and six additional searches, alongside a separate Qwen-only campaign. Figure~\ref{fig:campaign-dots} includes every completed ensemble campaign under this protocol. In each archive, the reported kernel minimizes validation CRPS among discovered candidates. The six additional campaigns freeze their selected candidate identifiers before test evaluation. Evaluation uses the same 50 validation and 60 test tasks, with eight BO steps per task. All selected kernels complete these evaluations without failures.
\begin{figure*}[t]
\centering
\includegraphics[width=0.95\textwidth]{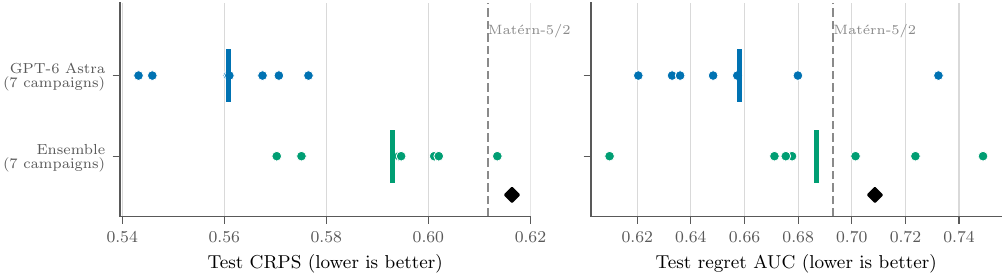}
\caption{\textbf{Seven ensemble and seven frontier campaigns on the 60 BBO test tasks.} The frontier campaigns use GPT-6 Astra. Each dot is one campaign's validation-selected kernel, bars mark the mean per proposer, and the dashed line is fixed Mat\'ern-5/2 ($\ell{=}0.5$, test CRPS $0.6117$, regret AUC $0.6931$). The black diamond is one Qwen-only campaign (test CRPS $0.616$, regret AUC $0.709$). Lower is better for both metrics. Campaigns that beat Mat\'ern: frontier 7 of 7 on CRPS and 6 of 7 on regret AUC, ensemble 6 of 7 on CRPS and 4 of 7 on regret AUC.}
\label{fig:campaign-dots}
\end{figure*}

Across the seven ensemble campaigns (Figure~\ref{fig:campaign-dots}), test CRPS is $0.5930\pm0.0153$ (mean $\pm$ sample standard deviation, SD), compared with $0.6117$ for the common fixed Mat\'ern-5/2 reference. Six of seven campaigns improve on this reference, with a mean reduction of $3.05\%$ and a range from $-0.30\%$ to $6.77\%$.

Regret AUC is $0.6869\pm0.0444$ across campaigns, compared with $0.6931$ for the fixed Mat\'ern-5/2 reference. Four of seven campaigns improve on it, and the second ensemble achieves the lowest AUC ($0.6097$). These outcomes show consistent predictive gains alongside campaign-dependent optimization performance. Dispersion is computed across seven complete campaigns evaluated on shared tasks. Search budgets are those recorded for each campaign. The Qwen-only campaign supplies a same-proposer comparison with the reported CAKE entry under its recorded training objective and aggregate search budget.

\paragraph{Search-budget-matched CAKE control.}
CAKE \citep{suwandi2025cake} uses LLM-guided crossover and mutation to search compositions of library kernels. It searches a kernel expression rather than an unrestricted component program. The reported campaign made 101 total LLM calls across its twelve runs. We gave CAKE the same total budget, one evolutionary search of 60 steps ($\approx$102 LLM calls under CAKE's own crossover-and-mutation-per-step schedule). Each candidate kernel expression is scored by its mean Bayesian information criterion (BIC) when fitted independently to each of the five training families, so the search must find one structure that generalizes across the training distribution, as \method's own fitness requires. Under this call-budget control, CAKE reaches $0.616$ test CRPS and $0.855$ test regret AUC. These values are worse than \dwf{} on both metrics, and the CRPS matches the same-proposer Qwen-only control's $0.616$. The call budgets are approximately matched, while CAKE retains its BIC objective and proposal procedure.

\subsection{Search-Ablation: Archive Sharing and Open-Ended Synthesis}
\label{app:search-ablation}

We compare the full method with independent proposals and a restricted construction grammar. All arms use the same proposer ensemble, per-campaign iteration budget, evidence ladder, and BBO evaluation protocol. The independent-proposal arm changes feedback, cross-iteration novelty screening, and candidate retention together, so it tests their combined contribution. The closure arm tests the effect of restricting the construction space while retaining archive search.

\begin{itemize}[leftmargin=*,itemsep=1pt,topsep=2pt]
\item \textbf{Full archive.} This arm retains the archive and construction interfaces used in the main method, but this ablation harness optimizes predictive CRPS alone. A shared MAP-Elites archive persists across iterations within a campaign, so each proposal round sees prior discoveries, and the novelty gate of \S\ref{sec:functional-novelty} screens against the growing archive.
\item \textbf{Independent proposals.} The four construction contracts and the evidence ladder stay the same, but each iteration proposes on its own, with no shared archive and no cross-iteration novelty gate. Only the final round's best candidate per campaign is kept.
\item \textbf{Closure grammar.} The shared-archive search loop stays the same, but the agent may use only the closure contract (\S\ref{sec:contracts}), which matches the compositional-grammar baselines of \S\ref{sec:related}.
\end{itemize}

Each arm runs 12 seeded campaigns under the same frozen evaluation protocol as \S\ref{sec:mainresults}, and we score the validation-selected candidate from each campaign once on the shared meta-test tasks.

\begin{table*}[t]
\caption{Search-ablation results on the BBO domain. The first three rows are 12 seeded campaigns per arm, reported as mean $\pm$ sample SD. CKS is one deterministic greedy search. Failed calls are model or tool calls that did not yield an accepted response, summed over each arm's 12 campaigns.}
\label{tab:search-ablation}
\centering
\begin{tabular*}{\textwidth}{@{\extracolsep{\fill}}lccc@{}}
\toprule
Search & Test CRPS & Mean AUC & Failed calls \\
\midrule
\textbf{Full archive} & $\mathbf{0.571\pm0.031}$ & 0.681 & 28 \\
Independent proposals & $0.617\pm0.037$ & 0.695 & 15 \\
Closure grammar & $0.618\pm0.021$ & 0.706 & 22 \\
CKS greedy grammar, depth $\leq3$ & $0.625$ & 0.684 & -- \\
\bottomrule
\end{tabular*}
\end{table*}

The full archive configuration improves test CRPS over both campaign ablations (Table~\ref{tab:search-ablation}). Against independent proposals, 11 of 12 campaigns favor the full method, with mean difference $\Delta=-0.046$. Against closure grammar, 11 of 12 again favor the full method, with $\Delta=-0.047$. Each exact sign-flip test gives $p=0.006$, or $p=0.013$ after Holm correction across the two contrasts. For the CKS check, we fit squared-exponential (SE), periodic, linear, and rational-quadratic leaves by marginal likelihood, greedily add, multiply, or replace one leaf up to depth three, and select the frozen structure by validation CRPS. The selected rational-quadratic kernel reaches $0.625$ test CRPS and $0.684$ regret AUC.

The full-archive mean of $0.571$ averages 12 validation-selected kernels. It differs from the original frozen \dwf{} result of $0.601$ in Table~\ref{tab:metaresults} and the seven-campaign mean of $0.593$ in \S\ref{sec:mainresults}. These summaries describe different sets of discovery campaigns. This harness also optimizes predictive CRPS alone, while the reported campaigns optimize the combined score of Equation~\eqref{eq:selection-score}, which also rewards regret AUC.

\subsection{A Frontier Proposer: GPT-6 Astra}
\label{app:proposer-scaling}

We test whether a stronger coding model yields stronger kernels. We replace the ensemble proposer with a single frontier model, GPT-6 Astra, and keep everything else fixed: the task context, the number of runs and their round and tool-call limits, the construction contracts, the acceptance checks, the tuning budget, validation selection, and the frozen evaluation. We ran seven BBO campaigns and report all of them. Each spends 100--112 LLM calls, against 101 for the first ensemble campaign. We also ran three campaigns each on ChemBench and forecasting, and two on GlucoseBench.

\paragraph{Prediction and optimization.}
Figure~\ref{fig:campaign-dots} shows lower test CRPS for all seven frontier campaigns than for fixed Mat\'ern. The campaign-level CRPS difference between the frontier and ensemble proposers is supported by an exact two-sided permutation test ($p=0.003$). Regret AUC improves less consistently, and its campaign-level difference is not statistically separated ($p=0.22$).

\paragraph{Progress across discovery runs.}
Figure~\ref{fig:search-progress} shows how the best validation CRPS falls as runs accumulate. The frontier proposer's archive is at the level of the best tuned baseline after its first run, while the ensemble needs about eight. The ensemble's mean flattens after that, and the frontier's keeps falling through the last run.

\begin{figure}[t]
\centering
\includegraphics[width=0.5\columnwidth]{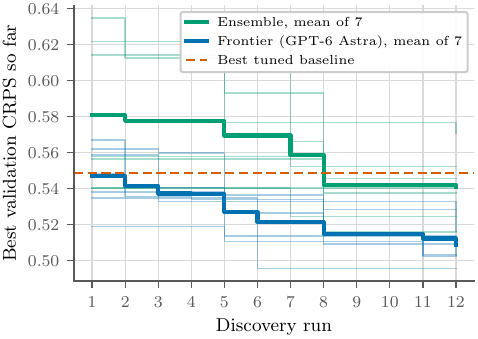}
\caption{\textbf{A stronger proposer starts ahead and keeps improving.} Best validation CRPS among the programs accepted after each discovery run (lower is better). Thin lines are the 14 individual campaigns. Thick lines are the means over the seven ensemble and seven frontier (GPT-6 Astra) campaigns. The dashed line is the best tuned baseline. A run that accepts no program leaves its curve unchanged.}
\label{fig:search-progress}
\end{figure}

The same validation-CRPS leader does not produce a matching regret curve (Figure~\ref{fig:search-progress-regret}). At each run we keep the program with the best validation CRPS so far, carry that choice forward when a run accepts nothing, and plot its validation regret AUC. The campaign is the unit. Mean regret falls for both proposers, while individual campaigns still move in both directions.

\begin{figure}[t]
\centering
\includegraphics[width=0.5\columnwidth]{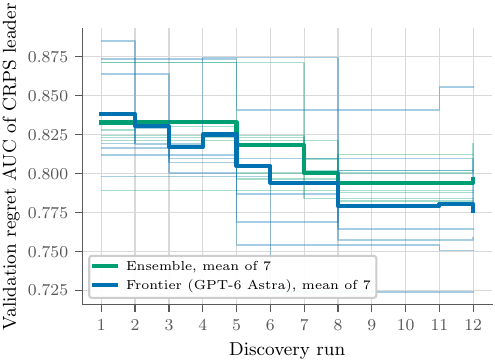}
\caption{\textbf{Validation regret of the validation-CRPS leader.} At each discovery run, the plotted value is the validation regret AUC of the program with the best validation CRPS among those accepted so far. Thin lines are the 14 campaigns. Thick lines are means over seven ensemble and seven frontier campaigns. The score is validation regret.}
\label{fig:search-progress-regret}
\end{figure}

\paragraph{Other domains.}
In ChemBench, validation selects frontier campaign F2, whose kernel reaches $0.312$ test CRPS, 44\% below the ensemble selection and significantly better than all six learned references (Appendix~\ref{app:chembench}). Two of the three frontier ChemBench campaigns submitted programs in only four of their eight runs. In the other runs, GPT-6 Astra proposed programs with more than twelve tunable parameters, which the tuner rejects (Appendix~\ref{app:impl}), and ran out of rounds. This rule applies to every proposer. In GlucoseBench, both frontier campaigns beat the strongest tuned reference by 22\%, against 15\% for the best ensemble campaign (Appendix~\ref{app:glucose}). Forecasting is the exception. The three frontier selections score 0.191--0.198 three-record CRPS, about the same as fixed RBF ($0.205$) and far from the ensemble's selection ($0.091$). Here validation barely separates kernels, since all three frontier selections lie within $0.080$--$0.086$ validation CRPS, close to RBF, so a better proposer has little signal to exploit.

\subsection{Broad Benchmark Outside the Search Families}
\label{app:dwf-broad}

We compare \dwf{} with Mat\'ern-5/2 on tasks that were not used in discovery, validation, or the original test. We assembled a suite of 44 analytic test functions in one to eight dimensions, using definitions from the SFU Virtual Library of Simulation Experiments.\footnote{\url{https://www.sfu.ca/~ssurjano/optimization.html}} The benchmark script provided with the submission code specifies each implementation, domain, and optimum. For each function we evaluate two domain regimes: the standard domain and an off-center subdomain that retains the global optimum away from its midpoint. Each of the 88 function--regime combinations has ten seeded episodes, with a new off-center subdomain drawn for each episode. Predictive episodes use $n=8+2d$ training points and 32 held-out points. BO episodes use $4+2d$ initial points and $T = 8$ EI steps with 128 fresh candidates per step, matching the main protocol. We compare the archived kernels with grid-fitted amplitude and noise, and a second pair that refits isotropic lengthscale, amplitude, and noise by marginal likelihood.

We also use eight real regression datasets with dimensions five to thirteen: diabetes and California housing from scikit-learn,\footnote{\url{https://scikit-learn.org/stable/datasets/}} and yacht, airfoil, concrete, energy, machine CPU, and Boston from OpenML.\footnote{\url{https://www.openml.org/}} The OpenML dataset IDs are 42370, 43919, 4353, 43918, 230, and 531, respectively. The energy target is Y1. The benchmark script removes rows with nonfinite values and constant input columns, then min-max scales inputs over each full dataset. Each dataset uses 20 random splits at $n=25$ and $n=100$ training points, with 100 held-out points. These comparisons test predictive CRPS only.

\begin{figure*}[t]
\centering
\includegraphics[width=0.98\textwidth]{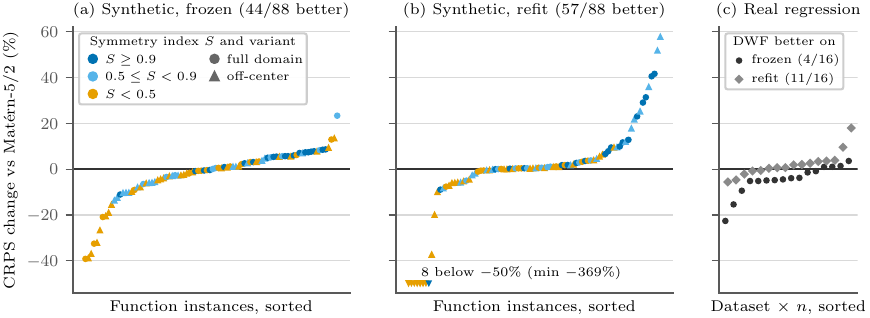}
\caption{\textbf{DWF outside the search families.} CRPS change relative to Mat\'ern-5/2, where positive values favor \dwf{}. Each synthetic point averages ten episodes for one function and domain variant. Color gives the empirical reflection-symmetry index $S$, defined in the paragraph below. Circles use the standard domain, and triangles use an off-center subdomain. The refit panel clips eight losses below $-50\%$, with a minimum of $-369\%$. Real-regression points average 20 splits for one dataset and training size.}
\label{fig:dwf-broad}
\end{figure*}

Figure~\ref{fig:dwf-broad} summarizes the comparison. For task values $f(u)$ on $u\in[0,1]^d$, we estimate $S=1-\frac{1}{d}\sum_{j=1}^d \mathbb{E}[(f(u)-f(r_j(u)))^2]/(2\operatorname{Var}(f(u)))$, where $r_j$ reflects coordinate $j$ about $1/2$. We use 4096 fixed-seed uniform points per task. Frozen \dwf{} wins 44 of 88 synthetic function instances, with median relative CRPS change $-0.2\%$ and paired Wilcoxon $p=0.77$. Its regret AUC is better on 52 instances, with a median gain of $2.7\%$ and $p=0.004$. Refitting raises the predictive win count to 57 and the median gain to $0.6\%$ ($p=0.025$), but eight large losses make the mean change $-11.9\%$. Refit regret AUC is not separated from Mat\'ern ($p=0.70$).

The effect follows the geometry diagnosed above. Among the 20 instances with empirical symmetry index $S\ge0.9$, frozen \dwf{} wins 16, with median CRPS gain $5.6\%$ ($p=0.0002$). Across all 88 instances, $S$ and the frozen CRPS gain have Spearman correlation $\rho=0.37$ ($p=0.0004$). On the 44 off-center variants, frozen \dwf{} wins only 15 and has median change $-3.3\%$ ($p=0.016$).

The real-data check is less favorable. With the archived geometry, \dwf{} wins 4 of 16 dataset-by-sample-size comparisons and has median change $-4.2\%$ ($p=0.011$). Per-dataset refitting raises this to 11 of 16 and a median gain of $1.3\%$, but the aggregate difference is not significant ($p=0.12$). The broad benchmark therefore supports \dwf{} as a readable prior for approximate central reflection.

\section{Analysis of the DWF Kernel}
\label{app:dwf}

This appendix collects the analysis of \dwf{}, the kernel discussed in \S\ref{sec:dwf-analysis}. We first give its discovery record, frozen geometry, and source code. We then state the regularity and soft symmetry it induces. The last part reports the matched controls and the sharpness and symmetry tests. Appendix~\ref{app:dwf-broad} then compares \dwf{} on tasks outside the search families, where it helped when the tasks had approximate reflection symmetry.

\subsection{Paired Predictive Comparisons}
\label{app:wilcoxon}
Table~\ref{tab:wilcoxon} reports paired CRPS differences on the same 50 validation and 60 test tasks. Negative differences favor \dwf{}. The Wilcoxon signed-rank test assumes symmetric differences for a location-shift interpretation. Holm correction over the five competitors within each split leaves every conclusion at the $0.05$ level unchanged. These task-level tests complement the paired intervals and familywise results rather than measuring variation across independent discovery campaigns.

\begin{table*}[t]
\caption{Paired Wilcoxon tests of \dwf{}'s per-task CRPS against each competitor ($n{=}50$ validation, $n{=}60$ test). ``Wins'' counts tasks where \dwf{} is lower. Med.\ $\Delta$CRPS is the median signed difference, and negative favors \dwf{}. The $p$-values are unadjusted. Holm correction over the five competitors within each split leaves every 0.05-level conclusion unchanged. $^*p<0.05$, $^{**}p<0.01$, $^{***}p<0.001$.}
\label{tab:wilcoxon}
\centering
\small
\begin{tabular}{@{}lcccccc@{}}
\toprule
 & \multicolumn{3}{c}{Validation ($n{=}50$)} & \multicolumn{3}{c}{Test ($n{=}60$)} \\
\cmidrule(lr){2-4}\cmidrule(lr){5-7}
Competitor & wins & med.\ $\Delta$CRPS & $p$ & wins & med.\ $\Delta$CRPS & $p$ \\
\midrule
\multicolumn{7}{l}{\textit{vs.\ fixed baselines}} \\
Periodic ($\ell{=}1$, $p{=}2$) & 34/50 & $-0.009$ & $0.013$* & 45/60 & $-0.018$ & $7.0\times 10^{-5}$*** \\
Mat\'ern-5/2 ($\ell{=}0.5$) & 31/50 & $-0.015$ & $0.069$ & 39/60 & $-0.013$ & $0.023$* \\
RQ ($\ell{=}0.25$, $\alpha{=}1$) & 39/50 & $-0.018$ & $3.5\times 10^{-5}$*** & 45/60 & $-0.019$ & $2.2\times 10^{-5}$*** \\
\midrule
\multicolumn{7}{l}{\textit{vs.\ discovered kernels}} \\
Simple spectral warp & 37/50 & $-0.011$ & $7.0\times 10^{-4}$*** & 46/60 & $-0.019$ & $3.6\times 10^{-6}$*** \\
Warp multiscale bump & 41/50 & $-0.037$ & $1.6\times 10^{-8}$*** & 43/60 & $-0.031$ & $9.3\times 10^{-5}$*** \\
\bottomrule
\end{tabular}
\end{table*}

\subsection{Discovery Record and Frozen Geometry}\label{app:dwf-record}

The \dwf{} source passed Tier 2 on the first verification, and submission re-checked the same source at the selected parameters before evaluation. Appendix~\ref{app:dwf-code} gives the submitted transform, and Appendix~\ref{app:revision-saved} records that replaying it reproduces the reported episode scores.

The kernel is the pullback $k_T(x,x')=k^{5/2}_{\ell}(T(x),T(x'))$, where $k^{5/2}_{\ell}$ is a Mat\'ern-5/2 kernel with lengthscale $\ell=0.5$. For each input coordinate, the transform concatenates a monotone warp and a triangular fold:
\begin{align}
w_{\lambda}(u) &= \frac{1-e^{-\lambda u}}{1-e^{-\lambda}},
\qquad
a_{\nu}(u) = \tfrac{1}{\pi}\arcsin\!\big(\sin(\pi\nu u)\big),
\label{eq:dwf}\\
T(x) &= \big(w_{\lambda}(x_j),\;\sin(\zeta)\,a_{\nu}(x_j)\big)_{j=1}^{d}.\nonumber
\end{align}
The warp remains close to the original coordinate. The fold maps approximate reflections to similar values, while the warp prevents them from becoming identical. Figure~\ref{fig:dwf-gp-samples} shows prior functions generated by this geometry.

\paragraph{Frozen geometry.}
Figure~\ref{fig:dwf-geometry} plots the frozen maps. On $[0,1]$ the warp differs from the identity by at most $0.027$. The fold spans less than one cycle, peaks at $u=0.513$, and identifies $u$ with $1/\nu-u$. That reflection is centered at $0.513$, near the middle of the unit interval. At the selected frequency, repeated feature values arise from this single reflection. The right panel shows what the joint covariance does with that fold. For the anchor $u=0.2$, whose fold-mirror is $u=0.825$, joint \dwf{} agrees with Mat\'ern-$5/2$ at the mirror, because the warp block still separates the two inputs. Along the rest of the slice the two curves differ by as much as $0.21$, so the fold does change local shape. The additive split, with equal branch weights, keeps covariance near $0.70$ at the same mirror pair, while the fold-only branch would assign covariance $1$ there. The additive form assigns a separate covariance component to the repeated fold coordinate, as quantified in Proposition~\ref{prop:soft-symmetry}.

\begin{figure*}[t]
\centering
\includegraphics[width=\textwidth]{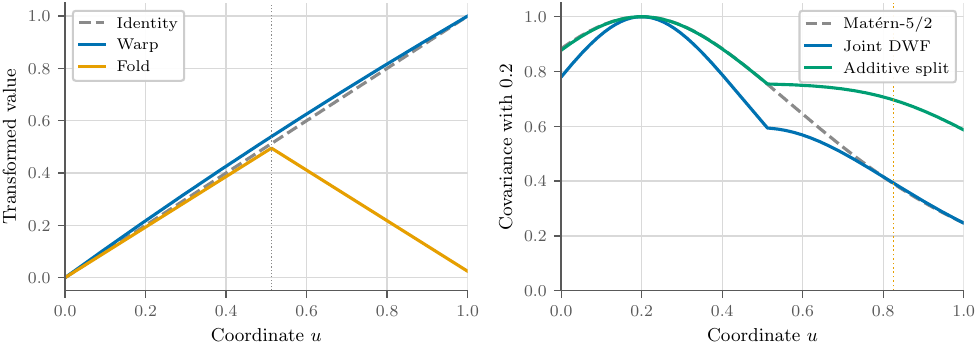}
\caption{Frozen \dwf{} geometry, where $u$ is a coordinate in $[0,1]$. Left: the warp stays near the identity, and the fold is one tent peaking at $u=0.513$ (dotted line). Right: covariance with a reference point at $u=0.2$ and the dotted line at its fold mirror $u=0.825$. Joint \dwf{} matches Mat\'ern-$5/2$ at the mirror, while the additive split stays near $0.70$.}
\label{fig:dwf-geometry}
\end{figure*}

\subsection{How \dwf{} Was Proposed}
\label{app:dwf-trace}

The agent registered its formulation in its first turn and staged the program with starting values $\lambda=2.0$, $\nu=1.5$, $\zeta=\pi/2$, and base lengthscale $0.5$. The backend scored these together with 63 Halton points, and the agent submitted the top trial ($\lambda=0.215$, $\nu=0.975$, $\zeta=1.433$). The base lengthscale remained fixed at $\ell=0.5$ throughout tuning because the declared tuning field did not control the base kernel. All reported \dwf{} evaluations therefore use this fixed lengthscale. The agent writes $\kappa$ for the fold frequency that the main text calls $\nu$.

The shared system prompt, domain task descriptions, and assignment are reproduced in Appendix~\ref{app:agent-prompts}.

\subsection{Reference Implementation}
\label{app:dwf-code}
The listing below is the Python code that the agent wrote for \dwf{} under the pullback contract (\S\ref{sec:contracts}).
\begin{CodeBox}{\texttt{transform\_point} (pullback contract)}
\begin{lstlisting}[style=kcode]
def transform_point(x, parameters):
    x = np.asarray(x, dtype=float).ravel()
    d = x.shape[0]
    warp_rate = float(parameters["warp_rate"])
    fold_frequency = float(
        parameters["fold_frequency"])
    fold_phase = float(parameters["fold_phase"])

    if abs(warp_rate) < 1e-10:
        warp = x.copy()
    else:
        warp = (1.0 - np.exp(-warp_rate * x)) / (
            1.0 - np.exp(-warp_rate))

    fold = (1.0 / np.pi) * np.arcsin(
        np.sin(np.pi * fold_frequency * x)
    ) * np.sin(fold_phase)

    result = np.concatenate([warp, fold])
    return result
\end{lstlisting}
\end{CodeBox}
Under the pullback contract the agent supplies only this pointwise transform, and it never builds a Gram matrix. The trusted interpreter evaluates \texttt{transform\_point} on each input, obtains the $2d$-dimensional feature vector, and applies the certified Mat\'ern-$5/2$ base kernel to pairs of those vectors to assemble $K$. Because the agent's code never sees pairs of points and never assembles the matrix, the program stays inside the contract of \S\ref{sec:contracts} and inherits the pullback validity guarantee.

\subsection{Prior Samples}
\label{app:dwf-samples}
Figure~\ref{fig:dwf-gp-samples} shows what the frozen kernel implies before any data are seen.

\begin{figure*}[t]
\centering
\includegraphics[width=\textwidth]{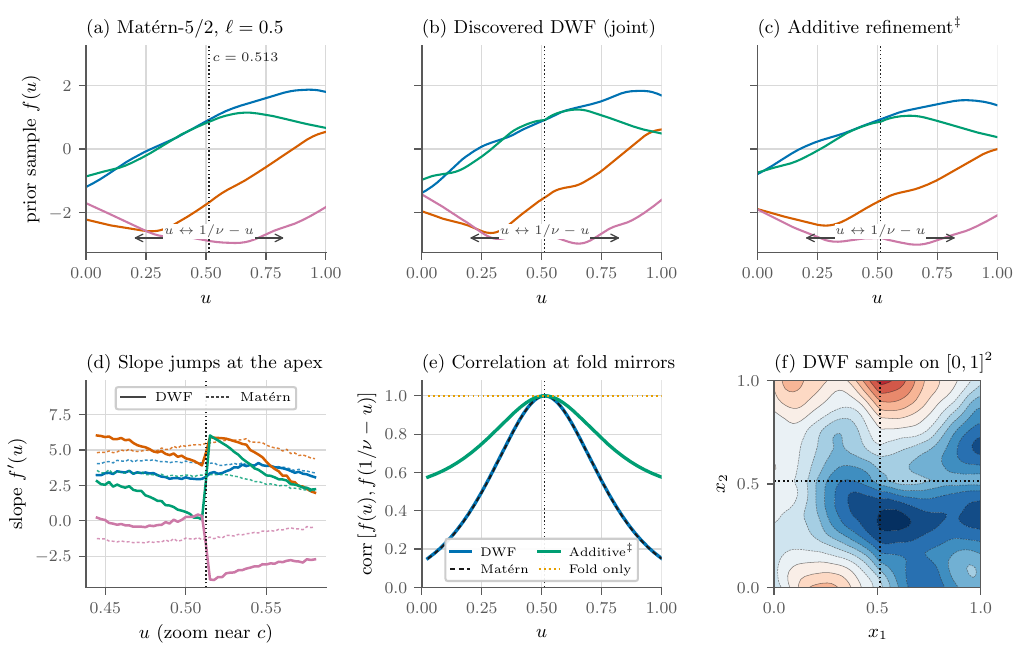}
\caption{\textbf{Prior samples under the discovered kernel.} Four draws from unit-variance GPs on $[0,1]$. Panels (a)--(c) share the same standard normals, so the draws differ only through the kernel: (a) Mat\'ern-$5/2$ with $\ell=0.5$, (b) joint \dwf{} at its frozen parameters, (c) the refinement $k_{\mathrm{add}}=\tfrac12 k_w+\tfrac12 k_a$. The dotted line marks the fold peak $c=1/(2\nu)=0.513$, and arrows join the mirror pair $u=0.2$ and $1/\nu-u=0.825$. (d) Finite-difference slopes near $c$. They jump for \dwf{} (solid) and stay continuous for Mat\'ern (dashed), apart from small ripples from the $10^{-8}$ jitter. (e) Prior correlation between $f(u)$ and $f(1/\nu-u)$. Joint \dwf{} is within $0.003$ of Mat\'ern, the additive kernel stays above $0.57$, and the fold branch alone fixes it at $1$. (f) One \dwf{} draw on $[0,1]^2$, with contours bending along $x_j=c$. $^\ddagger$Human researcher-defined refinement, not search output.}
\label{fig:dwf-gp-samples}
\end{figure*}

\subsection{Regularity of DWF}\label{app:dwf-regularity}

A smooth base kernel can induce a nonsmooth GP after a folded input transform. The result below characterizes that prior regularity. Its empirical relevance is tested separately by the smoothing ablation in Appendix~\ref{app:dwf-mechanism-controls}.

\begin{proposition}[Derivative jump at a fold]
\label{prop:dwf-cusp}
In one dimension, let $\nu>0$, $s=\sin\zeta$, and let $c=1/(2\nu)$ lie in the interior of the input interval. Let $G$ be a zero-mean, unit-variance, isotropic Mat\'ern-$5/2$ GP on $\mathbb R^2$ with lengthscale $\ell>0$, and set $F(u)=G(w_\lambda(u),s a_\nu(u))$. The one-sided mean-square derivatives exist and satisfy
\begin{align}
 [F']_c&:=F'_+(c)-F'_-(c)=-2s\nu\,\partial_2G(T(c)),\\
 \operatorname{Var}\big([F']_c\big)&=\frac{20s^2\nu^2}{3\ell^2}.
 \label{eq:dwf-jump}
\end{align}
If $s\ne0$, there is no two-sided mean-square derivative at $c$.
\end{proposition}

\begin{proof}
Locally $a_\nu(c+h)=1/2-\nu|h|$, so the two tangent vectors are $(w'_\lambda(c),s\nu)$ and $(w'_\lambda(c),-s\nu)$. The mean-square chain rule gives the first identity. The Mat\'ern-$5/2$ correlation has expansion $k(r)=1-5r^2/(6\ell^2)+o(r^2)$. Differentiating the covariance in its two arguments gives $\operatorname{Var}(\partial_2G)=5/(3\ell^2)$~\citep[Section 4.1.1]{rasmussen2006gaussian}. Multiplication by $(2s\nu)^2$ proves the variance formula. A positive variance of the difference prevents the two derivative limits from agreeing in mean square.
\end{proof}

The claim concerns the prior's one-sided mean-square derivatives. It does not require every posterior mean or covariance section to display a kink. For GP amplitude $\sigma_f^2$, multiply the variance by $\sigma_f^2$. In multiple dimensions, the same argument applies along a coordinate crossing its fold with the other coordinates fixed. Smoothness of the outer kernel alone therefore does not characterize the regularity induced by a nonsmooth transform.

\subsection{Additive Kernels and Soft Reflection Symmetry}\label{app:dwf-soft-symmetry}

The additive refinement allocates separate covariance components to the warp and fold. The following RKHS identities explain how it encourages agreement at fold mirrors without requiring identical predictions there.

\begin{proposition}[Additive norm and soft reflection symmetry]
\label{prop:soft-symmetry}
Let $k_w(x,y)=k_w^0(w(x),w(y))$ and $k_a(x,y)=k_a^0(a(x),a(y))$ be PSD kernels and $k_{\mathrm{add}}=\alpha k_w+(1-\alpha)k_a$ for $0<\alpha<1$. Then
\begin{equation}
 \|f\|_{\mathcal H_{k_{\mathrm{add}}}}^2
 =\inf_{f=f_w+f_a}\left\{
 \frac{\|f_w\|_{\mathcal H_{k_w}}^2}{\alpha}
 +\frac{\|f_a\|_{\mathcal H_{k_a}}^2}{1-\alpha}\right\}.
 \label{eq:additive-norm}
\end{equation}
Define $d_k(x,y)^2=k(x,x)+k(y,y)-2k(x,y)$. If $a(x)=a(y)$, then
\begin{align}
 d_{k_{\mathrm{add}}}(x,y)&=\sqrt\alpha\,d_{k_w}(x,y),\\
 |f(x)-f(y)|&\le\sqrt\alpha\,\|f\|_{\mathcal H_{k_{\mathrm{add}}}}d_{k_w}(x,y).
 \label{eq:soft-reflection}
\end{align}
On a compact domain, if $w$ is continuous and injective, $k_w^0$ is universal on its image, and $k_a$ is continuous, then $k_{\mathrm{add}}$ is universal.
\end{proposition}

\begin{proof}
The RKHS sum construction is the quotient of the direct sum of the two scaled RKHSs, giving the infimum norm~\citep{aronszajn1950theory,bach2008consistency}. Canonical squared distances add with the same weights. The fold distance is zero on a common fiber, proving the distance identity. The reproducing property and Cauchy--Schwarz give the function-difference bound. Finally, the warp pullback is universal because the continuous injection has a continuous inverse on its compact image. Its RKHS is included in the additive RKHS with norm increased by at most $\alpha^{-1/2}$, so density is preserved.
\end{proof}

The joint DWF kernel is also universal under the same compactness and base-universality assumptions because its transform is injective. Thus Proposition~\ref{prop:soft-symmetry} compares regularization geometry rather than claiming an unconditional approximation advantage for addition. Its mirror identity uses the same warp kernel on both sides. For the joint radial kernel at a fold mirror, the fold distance vanishes and the joint covariance equals that warp covariance when their lengthscales agree. The additive kernel instead allocates an explicit covariance component to the shared fold value.

\subsection{DWF versus Parametric Input Warping}
\label{app:dwf-warping}

The original power and exponential warps reach $0.566$--$0.585$ validation CRPS, compared with $0.540$ for \dwf{}. However, they use a negative log-likelihood (NLL) fitting objective. We therefore retune the constructions on a common predictive objective to compare their structure more directly. Table~\ref{tab:dwf-controls} reports the results.

\paragraph{Component ablation.}
Both branches contribute to prediction. The validation-selected joint kernel reaches $0.6011$ test CRPS, compared with $0.6223$ for the warp alone and $0.6147$ for the fold alone. Adding the two branch kernels improves CRPS further to $0.5668$. We evaluate this human researcher-defined refinement separately from the frozen search output, using the matched protocol below.

\paragraph{Matched tuning and frozen selection.}
We give each of seven arms 64 configurations for each of five scrambled-Halton seeds (101--105). All arms use the same ten training episodes, with two per family, the same GP amplitude/noise grids, and the same predictive CRPS objective. We check a 600-second limit between configurations, but every run completes all 64 trials before the cap. Each arm starts with an identity-oriented configuration and the archived DWF geometry.

Within each repeat, we retain the configuration with the lowest training CRPS. We then use 50 validation episodes to choose one repeat per arm and record the selected IDs before evaluating the common 60 test episodes and eight-step BO rollouts. This protocol tests refinements of the existing discovery.

\begin{table*}[t]
\centering
\caption{Matched predictive-objective controls (tuning protocol described above). Test CRPS mean $\pm$ SD is over five tuning repeats on the same 60 episodes, not over five discovery campaigns. Selected columns report the validation-chosen repeat. Archived DWF is a frozen anchor with a different historical search budget. $^\ddagger$Human researcher-defined refinement, not search output.}
\label{tab:dwf-controls}
\small
\begin{tabular}{lrrrr}
\toprule
Construction & Test CRPS (repeat mean $\pm$ SD) & Selected val. & Selected test & Selected AUC \\
\midrule
Identity Mat\'ern & $0.6198 \pm 0.0010$ & 0.5546 & 0.6189 & 0.6824 \\
Warp only & $0.6245 \pm 0.0028$ & 0.5590 & 0.6223 & 0.6896 \\
Fold only & $0.6339 \pm 0.0283$ & 0.5519 & 0.6147 & 0.6901 \\
Additive branches$^\ddagger$ & $0.5668 \pm 0.0000$ & 0.5020 & 0.5668 & 0.6190 \\
Joint branches & $0.6131 \pm 0.0071$ & 0.5404 & 0.6011 & 0.6712 \\
Rational quadratic & $0.6123 \pm 0.0049$ & 0.5489 & 0.6090 & 0.6896 \\
Optimized ARD & $0.6511 \pm 0.0156$ & 0.5730 & 0.6381 & 0.7612 \\
Archived DWF & $0.6011$ & 0.5404 & 0.6011 & 0.6712 \\
\bottomrule
\end{tabular}
\end{table*}

\paragraph{Result and interpretation.}
The selected additive kernel improves both reported means: test CRPS falls from $0.6011$ to $0.5668$ (5.7\%), and regret AUC from $0.6712$ to $0.6190$ (7.8\%). The regret reduction is descriptive because its multiplicity-adjusted test does not cross the 0.05 threshold, as reported below. The refinement improves CRPS on four of six family means, with losses on Bukin and Drop-Wave. Every arm completes all test episodes.

We quantify the paired CRPS gain by resampling episodes within each held-out family 5,000 times. The additive-minus-joint difference is $-0.0343$, with a 95\% bootstrap interval $[-0.0468,-0.0211]$. A two-sided paired Wilcoxon test gives Holm-adjusted $p=0.00028$ across the six non-anchor comparisons against joint DWF. These comparisons describe differences within the evaluated benchmark families.

All five additive repeats select the same shared initial geometry, with equal mixture weights and branch lengthscales $0.5$. This explains their zero SD across repeats. Joint DWF still improves on either separately retuned branch, but the additive construction performs best under this protocol.

\subsubsection{Sharpness Ablation}
\label{app:dwf-mechanism-controls}

\paragraph{Local smoothing on the BBO benchmark.}
We smooth the fold to test whether its sharp cusp is necessary for the additive gain. We compare sharp and smoothed versions of both joint and additive constructions, together with identity, warp-only, and fold-only controls. Each arm receives the same 64 configurations for each of five tuner seeds, including the archived geometry, with no runtime cutoff. We select one repeat per arm on the same 50 validation episodes, then evaluate it on 60 test episodes and eight-step BO rollouts. These ablations reuse the benchmark families and do not constitute new discovery campaigns.

For a raw fold value $a\in[-1/2,1/2]$, put $t=1/2-|a|$. Within $t<\epsilon$, replace $t$ by $\epsilon q(t/\epsilon)$, where $q(r)=3r^2-3r^3+r^4$, and retain the original fold elsewhere. With $\epsilon=0.02$, this removes the first-derivative cusp, preserves extrema and reflection fibers, and changes the unscaled feature by at most $27\epsilon/256=0.002109375$. The composed fold is $C^2$ at the patch boundaries and extrema. Smoothing changes local geometry as well as regularity.

\begin{table}[t]
\centering
\caption{BBO sharp/smooth ablation. Entries use validation-selected configurations evaluated on the same 60 test episodes. Lower is better. $^\ddagger$Human researcher-defined refinement, not search output.}
\label{tab:dwf-sharp-smooth}
\small
\begin{tabular}{lrr}
\toprule
Construction & Test CRPS & Regret AUC \\
\midrule
Joint & 0.601146 & 0.671159 \\
Joint, smooth & 0.601162 & 0.671159 \\
Additive$^\ddagger$ & 0.566832 & 0.618955 \\
Additive, smooth$^\ddagger$ & 0.566879 & 0.618955 \\
\bottomrule
\end{tabular}
\end{table}

Smoothing changes prediction only slightly. Within each construction, the sharp and smooth variants select the same structural parameters. The joint sharp-minus-smooth CRPS difference is $-0.000016$, with 95\% interval $[-0.000051,0.000024]$ and adjusted $p=1$. For the additive kernel, the difference is $-0.000047$, with interval $[-0.000079,-0.000014]$ and adjusted $p=0.0262$. Although detectable, the additive change amounts to only $0.008\%$ of CRPS. The paired AUC differences are exactly zero.

The additive gain persists after smoothing (Table~\ref{tab:dwf-sharp-smooth}). The additive-minus-joint CRPS difference is $-0.03431$ (5.7\%), with interval $[-0.04650,-0.02110]$ and adjusted $p=0.00030$. The AUC difference is $-0.05220$, with interval $[-0.09681,-0.01035]$ and adjusted $p=0.0742$, so we treat the AUC gain as descriptive. A sharp cusp is therefore unnecessary for the observed CRPS gain, even though Proposition~\ref{prop:dwf-cusp} identifies a derivative jump in the original prior.

For this follow-up, we use 20,000 paired bootstrap resamples within each family, weight family means equally, and apply approximate centered-bootstrap two-sided tests of the mean difference. Holm correction covers the three contrasts and two metrics together. We report unadjusted intervals. These bootstrap tests differ from the earlier Wilcoxon comparisons.

\section{Time-Series Forecasting}
\label{app:forecasting}

This appendix supports \S\ref{sec:results-ts}. It specifies the forecasting protocol, reports the full frozen results, and adds the two held-out records, a long-horizon diagnostic, and the AutoGP comparison.

\subsection{Protocol}
\label{app:ts}

This section specifies the forecasting benchmark used in \S\ref{sec:results-ts}, from the greenhouse-gas records and episode construction through the baseline grid and the cost of the agent loop.

\paragraph{Records and episodes.}
The benchmark bundles monthly global means of four greenhouse gases from the NOAA
Global Monitoring Laboratory\footnote{Monthly greenhouse-gas records, taken from the NOAA Global Monitoring Laboratory, \url{https://gml.noaa.gov/}.} (CO$_2$ with $569$ months from 1979, CH$_4$ with $514$
from 1983, N$_2$O with $304$ from 2001, and SF$_6$ with $346$ from 1997), with missing
months removed. The two further held-out records, CFC-12 and CFC-11, come from the NOAA HATS program and are described in Appendix~\ref{app:ts-extra}. Inputs are the month index scaled to $[0,1]$ over the full record.

For episode $i$, draw $W_i$ uniformly from $\{216,\dots,T_{\mathrm{rec}}-48\}$, where $T_{\mathrm{rec}}$ is the record length. After optionally reversing the full record, use observations $1{:}W_i$ for training and $W_i+1{:}W_i+48$ for evaluation. The remaining observations are unused in that episode. Add Gaussian noise only to the training targets, with standard deviation drawn uniformly between zero and $4\%$ of the clean training-window standard deviation. Standardize both splits using clean training-window statistics.

Seeds are offset by split: meta-train uses seeds $0$--$8999$, meta-validation
$10000$--$18999$ on the same three gases, and the held-out records use
$20000$ and above, with SF$_6$ from $20000$, CFC-12 from $21000$, and CFC-11 from $22000$. These seed ranges distinguish generated episodes. Underlying observations and target windows can overlap across episodes.

\paragraph{Fitting and fitness.}
Predictive evaluation fits only an outer covariance amplitude over
$(0.1,0.3,1,3)$ and diagonal noise over $(10^{-4},10^{-3},10^{-2},10^{-1})$ by training-window marginal likelihood, using the same grid as the synthetic benchmark in Equation~\eqref{eq:amp-noise}. It then scores held-out CRPS on the standardized 48-month horizon. Search-time fitness is
negative mean held-out CRPS across training episodes. Candidate-specific parameters
are tuned only through the backend's declared spaces (up to $64$ Halton-sequence
trials per program). Program fitness also subtracts a $10^{-3}\log(1+p)$ complexity
penalty for $p$ exposed parameters. The penalty is constant across a program's trials,
so it does not affect which setting the backend freezes. The functional
novelty gate of Equation~\eqref{eq:functional-novelty} applies with threshold
$\tau_{\mathrm{nov}}=0.08$.

\paragraph{Baseline grid.}
Fixed references use closure trees over normalized time: RBF, Mat\'ern, rational quadratic, linear, periodic, ARD, and spectral-mixture variants, and the strongest grid point per family is selected on training fitness only. Selection is separate for each benchmark, so the rational-quadratic winner here ($\ell{=}0.5$) differs from the BBO selection ($\ell{=}0.25$). Three fitted reference families, namely monotone input warping, grid spectral mixtures, and tuned random Fourier features (Appendix~\ref{app:metabo}), join the same frozen protocol, each archived as one winner per training family.

\paragraph{Search configuration and budget.}
This search reuses the ensemble configuration of \S\ref{sec:mainresults}: twelve
discovery runs capped at ten rounds and one hundred tool calls produced nine
Tier 2 verified programs (six spectral, two feature-map, one residual feature map),
consuming $498$ backend parameter-tuning trials. A predictive evaluation costs about one second per episode at these
window sizes.

\paragraph{Agent-loop cost.}
Table~\ref{tab:agent-cost} reports rounds, tool calls, and token usage per discovery
run. The mean gap between consecutive run starts is 2.6 minutes for BBO and 3.1 minutes for time series.

\begin{table*}[t]
\caption{Agent-loop cost per discovery run, min--max over the 12 runs of
each reported search.}
\label{tab:agent-cost}
\centering
\begin{tabular*}{\textwidth}{@{\extracolsep{\fill}}lcc@{}}
\toprule
Per discovery run & BBO & Time-series \\
\midrule
Rounds & 1--10 & 8--10 \\
Tool calls & 0--11 & 7--10 \\
Input tokens & 4k--190k & 69k--184k \\
Output tokens & 0.3k--10k & 3k--9.6k \\
\bottomrule
\end{tabular*}
\end{table*}

Both bounds in Table~\ref{tab:agent-cost} trace to specific outcomes. The BBO search's zero-tool-call minimum is one discovery run that returned a
single long planning turn and never called a tool, exhausting its round budget
immediately. Together with one other run whose optimized parameters failed Tier 2
re-verification at submission, this accounts for the two BBO runs left without a
verified submission above. All three unsubmitted time-series runs instead spent their
full ten-round budget cycling \texttt{stage\_candidate} and
\texttt{verify\_staged\_candidate} against the same interface-conformance failure.

Figure~\ref{fig:toolcalls} groups every tool call from both reported searches by workflow phase. Forecasting runs spent 45.9\% of their calls drafting and verifying code, compared with 26.8\% for BBO. A spectral program returns frequencies of shape $M\times d$ and raw weights of length $M$. The trusted interpreter constructs the covariance from those outputs (Definition~\ref{def:contract}). The program does not return a symbolic cosine sum or assemble the Gram matrix. The call shares alone do not identify why forecasting needed more verification attempts.

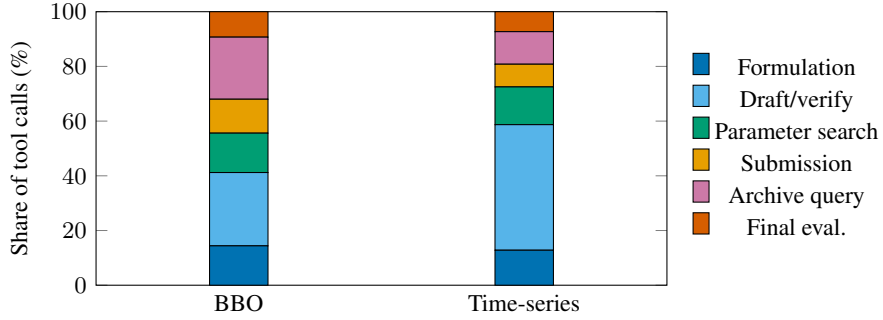
\begin{figure}[t]
\centering
\begin{tikzpicture}
\begin{axis}[
  ybar stacked,
  width=0.62\linewidth, height=5.2cm,
  bar width=22pt,
  ymin=0, ymax=100,
  ylabel={Share of tool calls (\%)},
  symbolic x coords={BBO,Time-series},
  xtick=data,
  enlarge x limits=0.5,
  legend style={at={(1.03,0.5)}, anchor=west, font=\footnotesize, draw=none,
    /tikz/every even column/.append style={column sep=4pt}},
  tick label style={font=\footnotesize}, label style={font=\footnotesize},
]
\addplot[fill=oiBlue] coordinates {(BBO,14.4) (Time-series,12.8)};
\addplot[fill=oiSky] coordinates {(BBO,26.8) (Time-series,45.9)};
\addplot[fill=oiGreen] coordinates {(BBO,14.4) (Time-series,13.8)};
\addplot[fill=oiOrange] coordinates {(BBO,12.4) (Time-series,8.3)};
\addplot[fill=oiPurple] coordinates {(BBO,22.7) (Time-series,11.9)};
\addplot[fill=oiVerm] coordinates {(BBO,9.3) (Time-series,7.3)};
\legend{Formulation, Draft/verify, Parameter search, Submission, Archive query, Final eval.}
\end{axis}
\end{tikzpicture}
\caption{Share of tool calls by workflow phase, pooled across the 12 runs of each reported search (97 BBO calls, 109 time-series calls). Each phase groups these tools. Formulation: \texttt{propose\_formulation}. Draft/verify: \texttt{stage\_candidate} and \texttt{verify\_staged\_candidate}. Parameter search: \texttt{optimize\_parameters} and \texttt{evaluate\_parameters}. Submission: \texttt{submit\_staged\_candidate}. Archive query: \texttt{query\_archive} and \texttt{inspect\_candidate}. Final evaluation: \texttt{evaluate\_candidate}.}
\label{fig:toolcalls}
\end{figure}

\begin{table}[t]
\caption{Frozen greenhouse-gas candidates, fitted reference families, and the
closed-grammar CAKE winner: mean CRPS over
validation episodes and over ten held-out SF$_6$ windows (lower is better).}
\label{tab:ts-fullresults}
\centering
\small
\setlength{\tabcolsep}{4pt}
\begin{tabular}{@{}lcc@{}}
\toprule
Method & val & SF$_6$ \\
\midrule
Linear & 0.826 & 0.959 \\
Periodic ($p{\approx}$annual) & 0.241 & 1.799 \\
RQ ($\ell{=}0.5$) & 0.095 & 0.067 \\
Mat\'ern-5/2 $\ell{=}1$ & 0.091 & 0.050 \\
Mat\'ern-3/2 $\ell{=}1$ & 0.099 & 0.075 \\
RBF $\ell{=}0.5$ & $\mathbf{0.086}$ & $0.029$ \\
ARD-RBF ($\ell=1$) & 0.165 & 0.016 \\
ARD-RBF ramp & 0.124 & 0.109 \\
\midrule
\multicolumn{3}{l}{\textit{learned-feature references}} \\
Spectral mixture & 0.750 & 1.024 \\
Grid spectral mixture (GSM) & 1.155--1.347 & 1.433--1.626 \\
Input warping & 0.121--0.172 & $\mathbf{0.014}$--0.168 \\
Random Fourier features & 0.114--0.204 & 0.019--0.071 \\
\midrule
\multicolumn{3}{l}{\textit{Kernaut discoveries}} \\
Harmonic--trend root & 0.128 & 0.018 \\
Dyadic multiband harmonics & 0.100 & 0.024 \\
Poly-modulated chirp & 0.126 & 0.026 \\
Chirp--poly feature map & 0.118 & 0.052 \\
Residual period-bank + poly & 0.088 & 0.059 \\
Power-law trend + Lorentz & 0.190 & 0.208 \\
Mat\'ern-cascade quasiperiodic & 0.192 & 0.211 \\
Gaussian harmonics$^\dagger$ & 0.219 & 0.223 \\
Harmonic--trend revision$^\dagger$ & 1.501 & 1.704 \\
CAKE best (of 30) & 0.082 & 0.036 \\
\bottomrule
\end{tabular}
\end{table}

\subsection{Full Time-Series Benchmark Results}
\label{app:ts-fullresults}

We report all frozen candidates and reference families in Table~\ref{tab:ts-fullresults}. Ranges show variation across fits to individual training families. We retain entries marked $^\dagger$ under the soft novelty-penalty policy and select baseline grids on training episodes. Uniform ARD at lengthscale 1 and isotropic RBF at lengthscale 0.5 use the same one-dimensional kernel family with different lengthscales.

Validation and SF$_6$ rank candidates differently. On the original ten-window SF$_6$ evaluation, the validation-selected kernel has CRPS $0.059$, compared with $0.029$ for RBF. The 30-window diagnostic in Table~\ref{tab:forecast-diagnostics} instead gives ARD-RBF the lowest overall CRPS, while harmonic--trend root is slightly better on forward windows. Harmonic--trend root is the strongest discovery on SF$_6$, but it loses all 60 CFC windows to the validation-selected kernel and averages $0.334$ across the three test records (Appendix~\ref{app:ts-extra}). Thus, the validation-selected kernel's main benefit is transfer to the two CFC records, rather than the best fit to SF$_6$ alone.

\paragraph{Closed-grammar forecasting search and diagnostics.}
CAKE searches on thirty training episodes from the three training gases. Validation selects a program at $0.082$ CRPS, ahead of the fixed kernels on that split. Its original ten-window SF$_6$ score is $0.036$, against $0.029$ for RBF. The three-record comparison in Appendix~\ref{app:ts-extra} is the primary transfer result.

Table~\ref{tab:forecast-diagnostics} separates forward and reversed SF$_6$ windows and reports interval coverage and width. Harmonic--trend root was selected for this diagnostic by its SF$_6$ test performance. The thirty windows comprise twenty forward and ten reversed windows, with 197 overlapping target-window pairs among 435 pairs, so the summaries describe variation within one record.

\begin{table*}[t]
\centering
\caption{Frozen forecasting diagnostics on 30 SF$_6$ windows (20 forward, 10 reversed). Coverage and width refer to nominal 95\% marginal predictive intervals on forward windows. Targets are clean held-out values, and variance includes fitted observation noise, matching the benchmark CRPS. Widths are in training-standardized units. The results are descriptive, from one record with overlapping windows.}
\label{tab:forecast-diagnostics}
\small
\setlength{\tabcolsep}{3pt}
\begin{tabular*}{\textwidth}{@{\extracolsep{\fill}}lrrrrr@{}}
\toprule
Kernel & All CRPS & Forward CRPS & Reversed CRPS & \shortstack[r]{Forward\\coverage} & \shortstack[r]{Forward\\width} \\
\midrule
Residual period-bank (selected) & 0.0456 & 0.0488 & 0.0391 & 0.918 & 0.334 \\
Harmonic--trend root & 0.0194 & 0.0174 & 0.0235 & 0.925 & 0.121 \\
RBF ($\ell=0.5$) & 0.0264 & 0.0284 & 0.0223 & 0.880 & 0.153 \\
Uniform ARD-RBF ($\ell=1$) & 0.0184 & 0.0178 & 0.0195 & 0.916 & 0.117 \\
\bottomrule
\end{tabular*}
\end{table*}

\subsection{Additional Held-Out Forecasting Records}
\label{app:ts-extra}

\paragraph{Records and protocol.}
We add two held-out records after search: the NOAA HATS combined global monthly means of CFC-12 and CFC-11\footnote{\url{https://gml.noaa.gov/hats/}}, with 586 and 587 months from 1977. Both rise and then decline, a shape absent from the three training gases. We chose them for this shape before evaluating any kernel. Episodes follow the SF$_6$ protocol exactly, with the same 48-month horizon, training-window standardization, noise injection, and reversal rule. The 30 SF$_6$ windows are the same as in Table~\ref{tab:forecast-diagnostics}. The CFC records use the next seed blocks, $21000$ and $22000$.

We score 58 frozen kernels: the 19 members of the forecasting archive (nine discoveries and ten training-selected fixed kernels), all 30 CAKE winners, and the three fitted winners each of input warping, random Fourier features, and the grid spectral mixture. No kernel is retuned or reselected. Each family's representative in Table~\ref{tab:ts-results} is its member with the lowest validation CRPS on the original validation episodes. No evaluation fails.

\begin{table*}[t]
\caption{Mean CRPS $\pm$ SEM over 30 held-out windows per record (lower is better). Each unstarred frozen row is selected on validation data. The overall mean weights the three records equally, and Fwd.\ uses forward-time windows only. $^\ast$Best fixed selects a fixed baseline separately for each test record. Harmonic--trend selects a discovery by SF$_6$ test performance only, then evaluates that same discovery on both CFC records. $^\dagger$AutoGP infers a new structure for every window and averages eight predictions.}
\label{tab:ts-results}
\centering
\footnotesize
\setlength{\tabcolsep}{4pt}
\begin{tabular}{@{}lccccc@{}}
\toprule
Method & SF$_6$ & CFC-12 & CFC-11 & Mean & Fwd. \\
\midrule
Fixed (RBF) & $\mathbf{0.026}\pm0.003$ & $0.268\pm0.061$ & $0.320\pm0.055$ & $0.205\pm0.027$ & $0.098\pm0.007$ \\
Best fixed$^\ast$ & $0.026\pm0.003$ & $0.182\pm0.053$ & $0.134\pm0.034$ & $0.114\pm0.021$ & $0.024\pm0.002$ \\
GSM & $1.473\pm0.047$ & $0.812\pm0.173$ & $0.633\pm0.105$ & $0.973\pm0.069$ & $0.690\pm0.017$ \\
Input warping & $0.103\pm0.010$ & $0.143\pm0.039$ & $0.129\pm0.026$ & $0.125\pm0.016$ & $0.057\pm0.004$ \\
RFF & $0.052\pm0.008$ & $0.252\pm0.067$ & $0.237\pm0.033$ & $0.181\pm0.025$ & $0.105\pm0.012$ \\
CAKE & $0.027\pm0.003$ & $0.220\pm0.063$ & $0.189\pm0.051$ & $0.145\pm0.027$ & $\mathbf{0.022}\pm0.002$ \\
CKS grammar, depth $\leq3$ & $0.063\pm0.005$ & $0.229\pm0.061$ & $0.234\pm0.049$ & $0.175\pm0.026$ & $0.065\pm0.004$ \\
AutoGP$^\dagger$ & $0.027\pm0.005$ & $0.185\pm0.050$ & $0.148\pm0.038$ & $0.120\pm0.021$ & $0.028\pm0.002$ \\
\midrule
Period-bank + poly & $0.046\pm0.006$ & $\mathbf{0.121}\pm0.027$ & $\mathbf{0.106}\pm0.022$ & $\mathbf{0.091}\pm0.012$ & $0.034\pm0.003$ \\
Harmonic--trend$^\ast$ & $0.019\pm0.003$ & $0.365\pm0.082$ & $0.617\pm0.116$ & $0.334\pm0.047$ & $0.138\pm0.013$ \\
\bottomrule
\end{tabular}
\end{table*}

\paragraph{Results.}
Table~\ref{tab:ts-extra} gives the paired comparisons. On both CFC records, the selected discovery ranks first of 58 by mean CRPS and clearly improves on the validation-selected fixed kernel, random Fourier features, and the grid spectral mixture. Its lower means against input warping and CAKE are not separated from zero. On SF$_6$ it loses to the fixed kernel and CAKE. The ranges within families are wide on the new records. Input warping spans $0.143$--$0.737$ on CFC-12 and $0.111$--$0.749$ on CFC-11, and CAKE spans $0.220$--$1.831$ and $0.189$--$1.470$. Selecting within a family by validation therefore matters for every method, not only for discovery.

Forward and reversed windows behave differently. Averaged over the forward windows of the three records, the discovery reaches $0.034$, against $0.022$ for CAKE and $0.024$ for the best fixed kernel per record. Its overall advantage therefore comes mainly from reversed windows, where its equal-record mean is $0.136$ against $0.194$ for AutoGP, $0.246$ for CAKE, and $0.291$ for RBF (Figure~\ref{fig:forecast-direction}). The three validation-selected frontier kernels score $0.191$--$0.198$ overall and remain behind the ensemble selection in both directions. These means treat the three records equally. Windows overlap inside each record, so the spread describes those records. Harmonic--trend root, the best discovered kernel on SF$_6$, loses every window on both CFC records to the selected discovery.

\begin{figure}[t]
\centering
\includegraphics[width=\columnwidth]{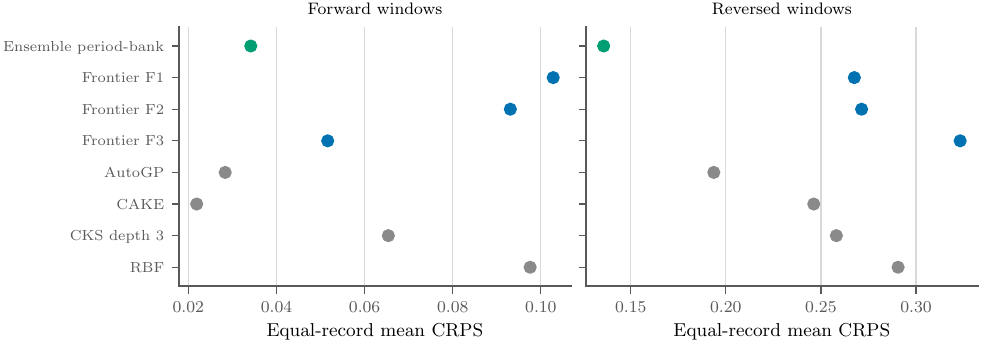}
\caption{\textbf{The same 90 windows, split by time direction.} Each point is the equal-record mean CRPS (mean within SF$_6$, CFC-12, and CFC-11, then the mean of those three). Forward and reversed windows are separate summaries inside the 48-month benchmark. The final-120-month diagnostic is a different protocol and is not included. Frontier F1--F3 are the validation-selected kernels from the three GPT-6 Astra forecasting campaigns.}
\label{fig:forecast-direction}
\end{figure}

\begin{table*}[t]
\caption{Paired comparisons of the validation-selected discovery (residual period-bank + poly) against each validation-selected reference, over 30 windows per record. Entries give windows won by the discovery and the two-sided Wilcoxon $p$-value. Pooled combines all 90 windows. Windows overlap within records, so $p$-values describe variation within records and not independent replications. Holm correction over the five references within each column leaves every 0.05-level conclusion unchanged.}
\label{tab:ts-extra}
\centering
\small
\setlength{\tabcolsep}{3.5pt}
\begin{tabular*}{\textwidth}{@{\extracolsep{\fill}}lcccc@{}}
\toprule
Comparator & SF$_6$ & CFC-12 & CFC-11 & Pooled \\
\midrule
Fixed (RBF, $\ell{=}0.5$) & 7/30 ($7.1\times10^{-5}$) & 26/30 ($2.0\times10^{-6}$) & 25/30 ($2.4\times10^{-5}$) & 58/90 ($7.5\times10^{-6}$) \\
Input warping & 26/30 ($2.6\times10^{-7}$) & 12/30 ($0.98$) & 15/30 ($0.42$) & 53/90 ($8.7\times10^{-5}$) \\
Random Fourier features & 14/30 ($0.52$) & 21/30 ($0.0066$) & 27/30 ($3.9\times10^{-7}$) & 62/90 ($4.0\times10^{-7}$) \\
Grid spectral mixture & 30/30 ($1.9\times10^{-9}$) & 30/30 ($1.9\times10^{-9}$) & 29/30 ($8.0\times10^{-8}$) & 89/90 ($2.7\times10^{-16}$) \\
CAKE & 2/30 ($9.3\times10^{-9}$) & 16/30 ($0.052$) & 16/30 ($0.24$) & 34/90 ($0.92$) \\
\midrule
Best fixed per record$^\ast$ & 7/30 ($7.1\times10^{-5}$) & 17/30 ($0.20$) & 14/30 ($0.72$) & -- \\
Harmonic--trend root$^\ast$ & 4/30 ($3.0\times10^{-5}$) & 30/30 ($1.9\times10^{-9}$) & 30/30 ($1.9\times10^{-9}$) & 64/90 ($1.8\times10^{-8}$) \\
\bottomrule
\end{tabular*}
\end{table*}

\paragraph{AutoGP.}
We also run native AutoGP inference \citep{saad2023sequential} independently on each of the same 90 windows, with eight particles, 75 Markov chain Monte Carlo (MCMC) steps, and 10 Hamiltonian Monte Carlo (HMC) steps per rejuvenation. We average the eight particle predictive distributions and score the resulting mixture directly, without selecting a particle using held-out targets. All 90 windows completed. AutoGP reaches mean CRPS $0.027$, $0.185$, and $0.148$ on SF$_6$, CFC-12, and CFC-11, giving an equal-record mean of $0.120$. It lowers mean CRPS relative to fixed RBF, CAKE, and validation-selected depth-three CKS. The Holm-adjusted signed-rank $p$-values are $7.5\times10^{-8}$ against fixed RBF, $0.62$ against CAKE, and $8.0\times10^{-8}$ against CKS. Relative to the selected discovery, AutoGP has paired mean difference $+0.029$, with a record-stratified bootstrap interval of $[+0.006,+0.055]$, 44 wins and 46 losses, and Holm-adjusted signed-rank $p=0.39$. AutoGP is stronger on forward windows than the discovery, $0.028$ against $0.034$, although CAKE remains lower at $0.022$. These results strengthen the classical comparison while preserving the evidence for cross-record transfer by one frozen structure. In this evaluation, AutoGP infers its structure from each training window and forecasts the held-out block without seeing its future targets. It does not update after each held-out observation. AutoGP uses substantially more inference, while the other primary rows freeze one structure after validation. The 90 windows are fixed and overlap within records, so the intervals and tests describe variation over these benchmark windows rather than uncertainty over a population of gas records.

\paragraph{Forecast visualizations.}
Following the forecast displays of \citet{saad2023sequential} and \citet{saad2024scalable}, Figure~\ref{fig:ts-forecasts} plots the posterior predictive of each validation-selected representative on one window per record, with the surrounding record retained for context. To avoid choosing windows that favor any method, we take the window on which the selected discovery has its median CRPS among the 30 windows of that record, and plot every method on that same window. The SF$_6$ window runs forward, and both CFC windows are reversed, so the model backcasts toward the concentration peak. On CFC-12, the fixed RBF kernel follows the broader trend of the training window past the peak, and CAKE extends the recent linear decline. Both have narrow intervals. The discovery and input warping flatten toward the plateau, and their intervals cover the held-out record. On CFC-11, the fixed kernel again overshoots, and CAKE, input warping, and the discovery reach similar CRPS. On the forward SF$_6$ window, CAKE and the fixed kernel track the linear rise more tightly than the discovery, whose interval widens with the horizon.

\begin{figure*}[t]
\centering
\includegraphics[width=\textwidth]{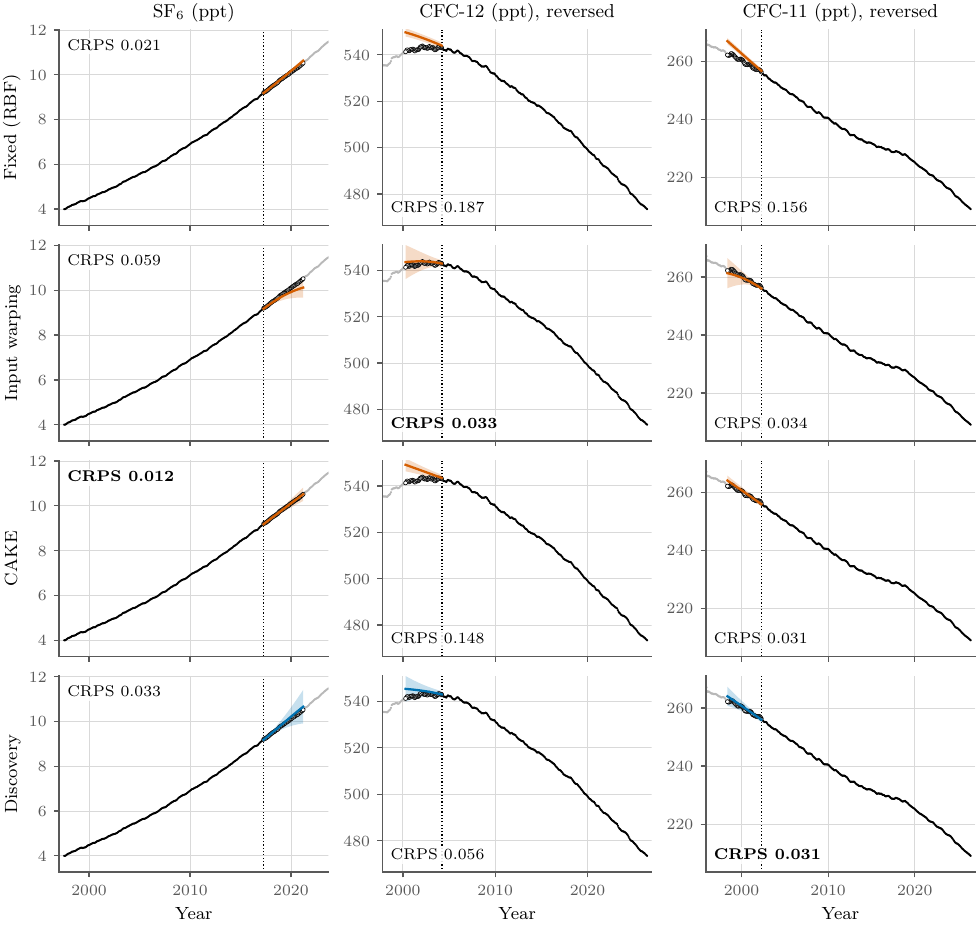}
\caption{Posterior forecasts of the validation-selected kernels on the median-CRPS window of each record (rows are methods, columns are records). Black dots are the observed prefix, hollow markers are the held-out 48 months, and light grey dots are the rest of the record. Lines and bands are the predictive mean and 95\% interval including observation noise, in parts per trillion (ppt). Both CFC windows are reversed, so the forecast runs backward from the dotted origin. Panel CRPS is on the standardized scale of Table~\ref{tab:ts-results}.}
\label{fig:ts-forecasts}
\end{figure*}

\paragraph{Long-horizon diagnostic.}
The 48-month benchmark mixes forward and reversed windows. We therefore add a deterministic diagnostic that trains on each record except its final 120 months and forecasts those ten years forward. At this much longer horizon, CAKE reaches CRPS $0.057$, $0.038$, and $0.075$ on SF$_6$, CFC-12, and CFC-11, compared with $0.117$, $0.079$, and $0.150$ for the selected discovery. Figure~\ref{fig:ts-longhorizon} helps localize this gap. The discovery tends to flatten its posterior mean and widen its interval, while CAKE preserves the record-specific long-run direction more closely. These three diagnostic splits favor CAKE for ten-year forward prediction. The discovery's gains on the 48-month benchmark, which mixes forward and reversed windows, therefore do not establish an advantage at this longer forward horizon. Because both methods were selected on the original validation windows, the horizon change alone does not explain their performance gap. 

\begin{figure*}[t]
\centering
\includegraphics[width=\textwidth]{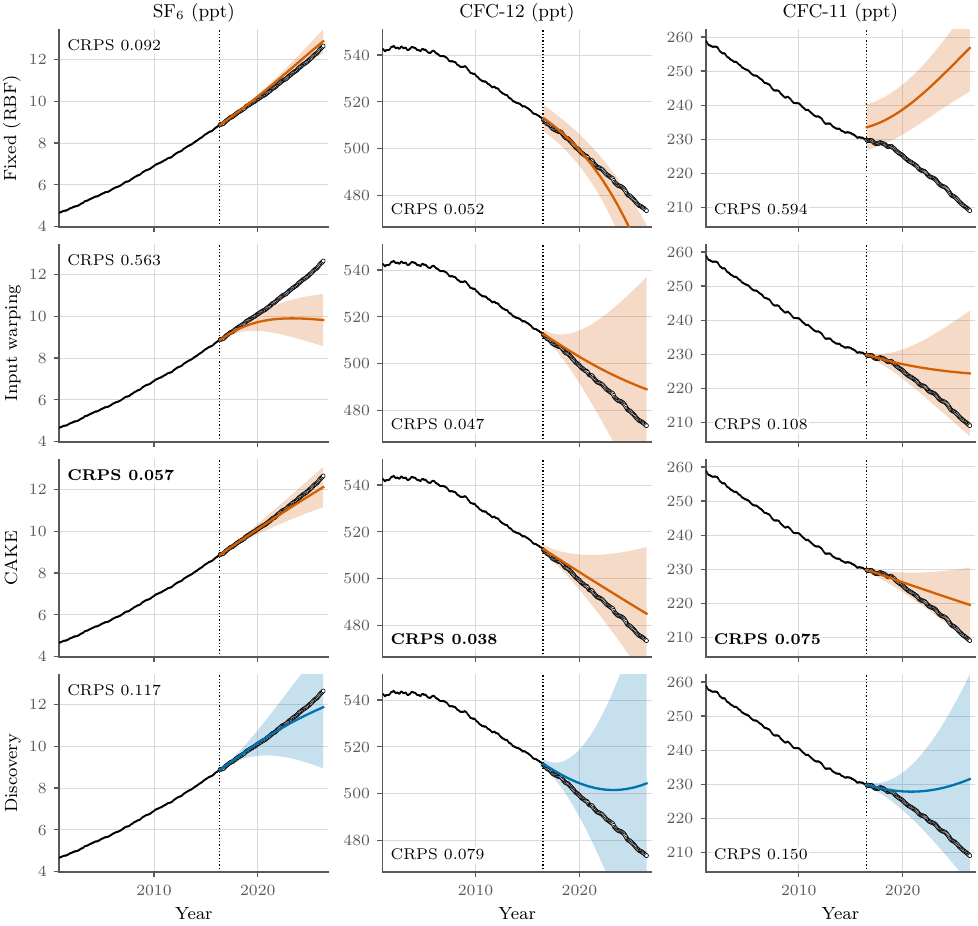}
\caption{Forward forecasts of the final 120 months of each record, a horizon 2.5 times longer than the main benchmark (rows are methods, columns are records). Black dots are the observed prefix, hollow markers are the held-out final ten years, and the dotted line is the forecast origin. Lines and bands are the predictive mean and 95\% interval including observation noise, in ppt. Panel CRPS is on the training-standardized scale, and bold marks the lowest value in each column.}
\label{fig:ts-longhorizon}
\end{figure*}

\paragraph{Additive decomposition.}
To examine how the selected discovery represents a rise-to-decline transition, Figure~\ref{fig:ts-decomposition} decomposes its posterior on the predeclared CFC-12 peak-origin split with a 120-month horizon. The Mat\'ern and low-frequency sinusoid-bank components jointly represent the trajectory, while the fitted constant component is negligible. The wide component intervals show that the additive allocation is uncertain outside the observed range. The Mat\'ern component carries most of the curvature, and the sinusoid bank adds a smooth, nearly linear drift. The sinusoid bank therefore acts as a slow trend on this record, not as an annual seasonal component. The training residuals are small relative to the signal and show no visible trend. The split-level CRPS of $0.070$ is an interpretive example, not an aggregate performance claim.

\begin{figure}[t]
\centering
\includegraphics[width=0.42\textwidth]{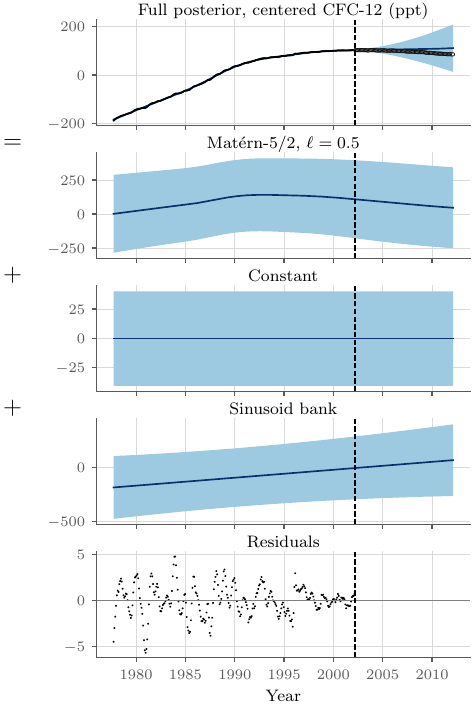}
\caption{Additive posterior decomposition of the selected discovery on CFC-12 (peak-origin split, 120-month horizon). Top row: observed and held-out points with the posterior mean and its two-standard-deviation band. Next rows: the Mat\'ern, constant, and combined sinusoid-bank components, then the training residuals. All rows are deviations from the training mean in ppt, so the component means sum to the full posterior mean. The dashed line is the forecast origin, and the split CRPS is $0.070$.}
\label{fig:ts-decomposition}
\end{figure}

\section{Scientific Structure Discovery}
\label{app:extra-domains}
This appendix details the enzyme-kinetics (\S\ref{sec:chembench}) and glucose-dynamics (\S\ref{sec:glucose}) benchmarks. Both follow the same split-and-freeze principle. Search and validation data are disjoint from the test data, which come from unseen enzyme mechanisms or glucose patients. Each frozen kernel is evaluated on episodes that search never sees. Both benchmarks use the construction contracts of \S\ref{sec:contracts} unchanged. The input encodings, oracles, and evaluation protocols are described below.

\subsection{Enzyme Kinetics (ChemBench)}
\label{app:chembench}

ChemBench, the ActiveSciBench-Chem benchmark of \citet{kabra2026llm}, poses enzyme-kinetics rate-law discovery as a seven-input, one-output regression problem. Substrate concentration $C_A$, inhibitor concentration $C_I$, second-substrate concentration $C_B$, product concentration $C_P$, enzyme loading, temperature, and pH map to an observed reaction rate $r_0$. Training uses
ten canonical single-mechanism domains (Michaelis-Menten saturation, four inhibition
types, Arrhenius temperature dependence, pH activity, ping-pong bisubstrate kinetics,
substrate inhibition, and Hill cooperativity). The five structurally distinct held-out
mechanisms are ordered bi-bi kinetics, reversible Michaelis-Menten, allosteric
activation, fractal kinetics, and metal-ion activation. None is a recombination of the
training mechanisms. The evaluator only calls the oracle's public
\texttt{run(params)} interface and never reads the underlying rate-law code. Search
fitness is negative held-out CRPS plus a worst-domain penalty (mirroring the worst-case
robustness term of Equation~\eqref{eq:selection-score}, Appendix~\ref{app:impl}), computed on the ten training domains only.

The validation-selected ensemble campaign E1 maps the normalized input $x=(a,i,b,p,e,t,h)$ to a nine-dimensional biochemical geometry before applying an RBF base kernel. Let $\operatorname{sig}(z)=(1+e^{-z})^{-1}$ denote the logistic sigmoid, $s(z)=\operatorname{sig}((z-\mu_A)/\sigma_A)$, $T_{\mathrm K}(t)=278+90t$, $h_0=(\mathrm{pH}_{\mathrm{opt}}-4)/6$, $\bar w=w_{\mathrm{pH}}/6$, and
\begin{equation}
q(h)=\operatorname{sig}\!\left(\frac{h-h_0+\bar w}{\bar w/2}\right)
\left[1-\operatorname{sig}\!\left(\frac{h-h_0-\bar w}{\bar w/2}\right)\right].
\end{equation}
Its discovered transform is
\begin{equation}
\begin{aligned}
T_{\mathrm{chem}}(x)=\big(&\xi_A s(a),\xi_A s(b),e,\xi_I\,i\,s(a),\tfrac12\,p\,s(a),\\
&\xi_T e^{-E_a/(R\,T_{\mathrm K}(t))},\xi_{\mathrm{pH}}q(h),
\tfrac12s(a)s(b),\tfrac12\big),
\end{aligned}
\label{eq:chem-feature-map}
\end{equation}
Here, $R=8.314$ J/(mol K), and $E_a$ is expressed in J/mol inside the exponential. We obtained the frozen parameters from 32 Halton trials initialized at the agent's values. For example, tuning moved $E_a$ from $20$ to $34.65$ kJ/mol and $\xi_A$ from $1$ to $1.991$. The frozen values are $(\xi_A,\xi_I,\xi_T,\xi_{\mathrm{pH}})=(1.991,1.125,0.0864,0.0195)$, $(\mu_A,\sigma_A)=(1.101,0.207)$, $E_a=34.65$ kJ/mol, $\mathrm{pH}_{\mathrm{opt}}=6.282$, and $w_{\mathrm{pH}}=0.201$.

The learned distance represents saturation, inhibitor and product coupling, and bisubstrate interaction jointly rather than treating the seven coordinates as exchangeable. The program includes temperature and pH features. However, the frozen kernel shows almost no sensitivity to these inputs in the coordinate-sweep diagnostic (Appendix~\ref{app:kernel-descriptions}). Figure~\ref{fig:chem-features} plots the components.

\begin{figure*}[t]
\centering
\includegraphics[width=\textwidth]{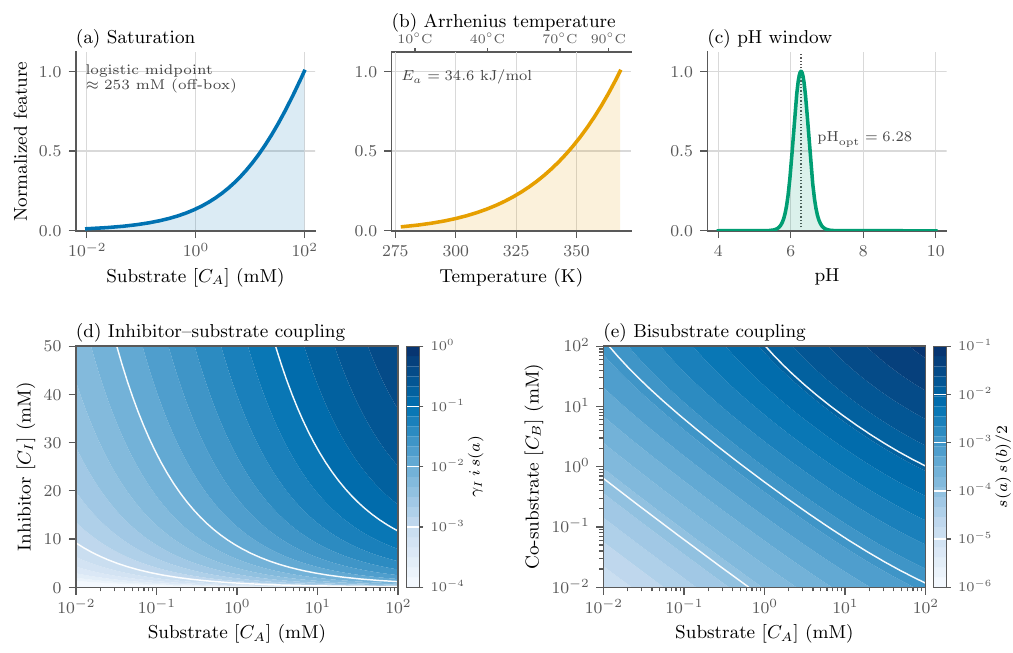}
\caption{\textbf{The discovered ChemBench kernel is assembled from interpretable biochemical features.} Top: substrate saturation, Arrhenius temperature dependence, and the pH window, each normalized to compare shape. At their frozen weights, the temperature and pH features vary by less than $10^{-5}$ and $0.02$ across the input box (Appendix~\ref{app:kernel-descriptions}). Bottom: the inhibitor--substrate coupling $\xi_I\,i\,s(a)$ and the bisubstrate coupling $\tfrac12 s(a)s(b)$, the fourth and eighth features of Equation~\eqref{eq:chem-feature-map}, on a log color scale with white contours at each decade. Both grow with substrate concentration because the saturation midpoint lies above the $100$\,mM bound.}
\label{fig:chem-features}
\end{figure*}

Nine separate stochastic evolutionary campaigns (eight
archive-linked proposal runs per campaign) each produced a distinct
validation-selected program, seven input transforms (one of them residual) and two
feature maps. The strongest programs encode saturation,
inhibition and substrate interactions, temperature dependence, and pH
sensitivity, giving an inspectable representation of the learned input geometry. On the common 75 held-out episodes (Figure~\ref{fig:chem-val-test}), the five
campaigns with the lowest validation CRPS improve mean CRPS over the strongest
evaluated fixed reference by 9.5\%--23.0\%, and each improves the mean on
all five held-out mechanisms. The other four validate markedly worse ($0.73$--$0.79$
against $0.62$--$0.67$) and range from a 1.8\% gain to a 10.6\% loss. Validation and
test CRPS agree closely across campaigns (Spearman $\rho=0.95$, $p<10^{-4}$), and the
mean improvement over all nine is 8.13\% (sample SD 10.90 percentage points, range $-10.6\%$ to $23.0\%$). This dispersion is across complete search campaigns. The 75 test episodes are shared across campaigns, so pooling all 675 scores as independent observations would overstate the evidence. Campaigns reuse scheduler seeds 0--7 and differ through stochastic proposals, so they are repeated campaigns and not nine disjoint deterministic seed schedules. The strongest candidate also has
18.6\% lower CRPS than the best evaluated input-warping fit. The following comparison uses a shared predictive objective to assess this geometry against stronger learned references.

\paragraph{Matched predictive-objective references.}
We fit isotropic Mat\'ern, independently scaled ARD, monotone coordinate warps, an author-defined domain-feature library, and two DKL architectures on the same 20 training episodes. Every method uses five initializations (seeds 101--105) and up to 300 full-batch Adam evaluations at learning rate $0.01$, with the outer amplitude and noise chosen on the benchmark grid. Validation episodes select one initialization per arm before any test evaluation, and structural parameters are then frozen on the common 75 held-out episodes. Coordinate-aware DKL uses a $7\rightarrow16\rightarrow8$ ReLU feature map with an RBF head, and the shared-coordinate control sums coordinate activations before the head. The domain library is a researcher-defined reference informed by the domain, not an automatic discovery.

\begin{figure}[t]
\centering
\includegraphics[width=0.8\columnwidth]{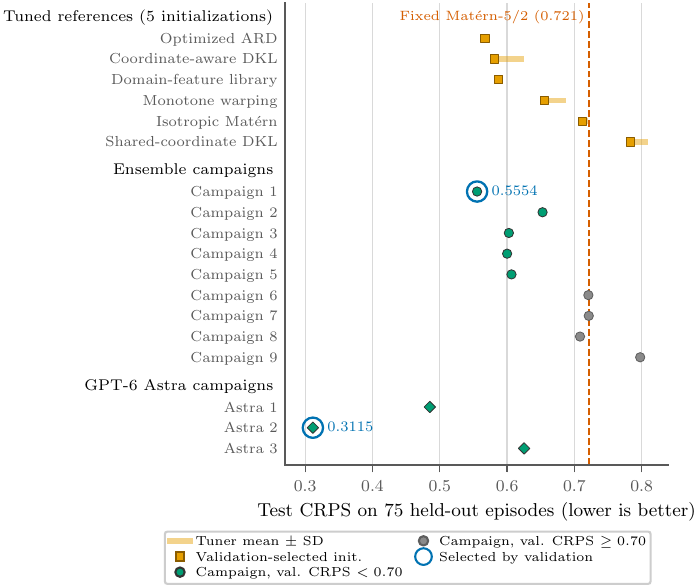}
\caption{\textbf{Matched ChemBench references and discovery campaigns on the same 75 test episodes.} Squares mark each reference's validation-selected initialization, and bars give mean $\pm$ SD over five initializations. Each campaign point is its validation-selected frozen discovery, green if validation CRPS is below $0.70$ and grey otherwise. Blue rings mark the campaign selected overall. The dotted line is fixed Mat\'ern-5/2. Historical campaigns keep their original budgets.}
\label{fig:chem-controls}
\end{figure}

Table~\ref{tab:chem-paired} compares the validation-selected kernels with each reference over the 75 held-out episodes. The ensemble selection E1 has the lowest mean test CRPS of the ensemble campaigns and the references, but it is not significantly better than optimized ARD, coordinate-aware DKL, or the domain library after Holm correction. We therefore describe it as competitive with them. The frontier selection F2 is significantly better than all six references and wins on at least 71 of 75 episodes against each. Its improvement over Mat\'ern-5/2 is 38\%, 80\%, 87\%, 70\%, and 27\% on ordered bi-bi, reversible Michaelis--Menten, allosteric activation, fractal kinetics, and metal activation. Every frontier campaign improves every held-out mechanism. In frontier campaigns F2 and F3, only four of eight runs submitted a program, because the tuner rejects programs with more than twelve parameters (Appendix~\ref{app:proposer-scaling}).

\begin{table*}[t]
\caption{Paired comparison of the validation-selected ChemBench kernels with the matched references over 75 held-out episodes. Entries give the mean CRPS difference (discovery minus reference, negative favors the discovery), its within-mechanism paired-bootstrap 95\% interval where computed, episodes won, and the Holm-adjusted two-sided Wilcoxon $p$-value across the six references.}
\label{tab:chem-paired}
\centering
\small
\begin{tabular}{@{}L{0.28\linewidth}L{0.32\linewidth}L{0.32\linewidth}@{}}
\toprule
Reference (test CRPS) & Ensemble E1 ($0.555$) & Frontier F2 ($0.312$) \\
\midrule
Optimized ARD ($0.567$) & $-0.012$ $[-0.033,0.010]$, $p=0.87$ & $-0.256$ $[-0.326,-0.193]$, 73/75, $p=6\times10^{-13}$ \\
Coordinate-aware DKL ($0.581$) & $-0.025$ $[-0.054,0.001]$, $p=0.32$ & $-0.269$ $[-0.346,-0.202]$, 71/75, $p=2\times10^{-12}$ \\
Domain library ($0.588$) & $-0.032$ $[-0.054,-0.011]$, $p=0.13$ & $-0.276$ $[-0.347,-0.215]$, 75/75, $p=3\times10^{-13}$ \\
Monotone warping ($0.655$) & $-0.100$, $p=3\times10^{-11}$ & $-0.344$, 71/75, $p=6\times10^{-13}$ \\
Isotropic Mat\'ern ($0.712$) & $-0.156$, $p=8\times10^{-13}$ & $-0.400$, 74/75, $p=3\times10^{-13}$ \\
Shared-coordinate DKL ($0.783$) & $-0.228$, $p=1\times10^{-12}$ & $-0.472$, 74/75, $p=3\times10^{-13}$ \\
\bottomrule
\end{tabular}
\end{table*}

\paragraph{The Astra kernel.}
The frontier F2 kernel is a feature map on the seven public inputs, and it never reads the oracle's rate laws. It builds ten substrate-response shapes modeled on the training rate laws: Michaelis--Menten saturation, competitive, uncompetitive, and noncompetitive inhibition by the inhibitor and by the product, ping-pong bisubstrate kinetics, substrate inhibition, and Hill cooperativity. Each shape is evaluated at three substrate scales and paired with its sensitivity to log substrate concentration. The result is multiplied by enzyme, Arrhenius-like temperature, and pH-ionization factors, and a linear backbone is added. Its eleven parameters are tuned by the Halton search. An unseen mechanism is thus predicted as a combination of known mechanism shapes and their local variations, much like a library of canonical mechanisms in model discovery \citep{murphy2026mdagent}.

These comparisons describe the strongest of nine frozen discoveries, which is also the campaign that validation selects. Across campaigns, eight of nine have lower mean CRPS than fixed Mat\'ern, while one has lower mean CRPS than optimized ARD and coordinate-aware DKL. The search can therefore produce an inspectable geometry that competes with learned references, but its average outcome is weaker. Figure~\ref{fig:chem-controls} separates the individual campaigns from the reference fits.

All six reference arms complete every test episode.

Figure~\ref{fig:chem-recovery} compares local predictions on two held-out mechanisms. Figure~\ref{fig:chem-rate-recovery} shows a substrate-concentration slice of the held-out fractal-kinetics mechanism. The discovered kernel follows the near-flat region and the subsequent rise more closely than the fixed Mat\'ern baseline, which predicts a large negative excursion before overshooting at high concentration. The discovered mean also deviates from the true curve at high concentrations. These local curves illustrate predictive behavior, while aggregate held-out performance is measured across episodes.

\begin{figure}[t]
\centering
\includegraphics[width=0.5\textwidth]{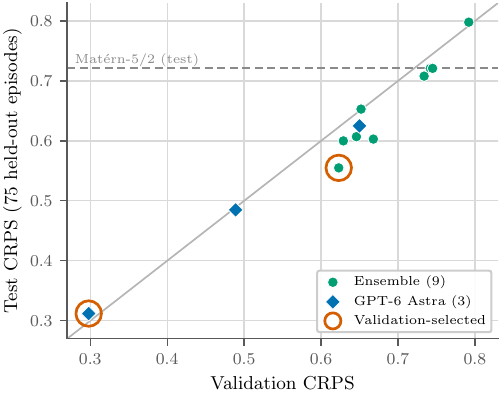}
\caption{\textbf{Validation CRPS predicts held-out CRPS across ChemBench campaigns.} Validation and test CRPS of the validation-selected kernel of each search campaign, nine ensemble and three GPT-6 Astra, on 75 held-out episodes from five mechanisms never seen during search (lower is better, Spearman $\rho=0.95$ over the ensemble campaigns). Rings mark the campaign that validation selects within each proposer, and the dashed line is fixed Mat\'ern-5/2 ($\ell{=}1$, test CRPS $0.721$).}
\label{fig:chem-val-test}
\end{figure}

\begin{figure*}[t]
\centering
\begin{subfigure}[t]{0.48\textwidth}
\centering
\includegraphics[width=\linewidth]{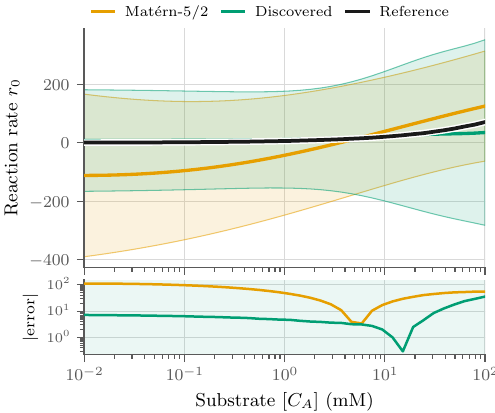}
\caption{Predicted reaction rate versus substrate concentration $[C_A]$ on the held-out fractal-kinetics mechanism, on a finer query grid than the benchmark. The discovered kernel (green) follows the true rate (black) more closely than the fixed baseline (orange), which predicts a large negative excursion. Both curves deviate at high concentrations.}
\label{fig:chem-rate-recovery}
\end{subfigure}
\hfill
\begin{subfigure}[t]{0.48\textwidth}
\centering
\includegraphics[width=\linewidth]{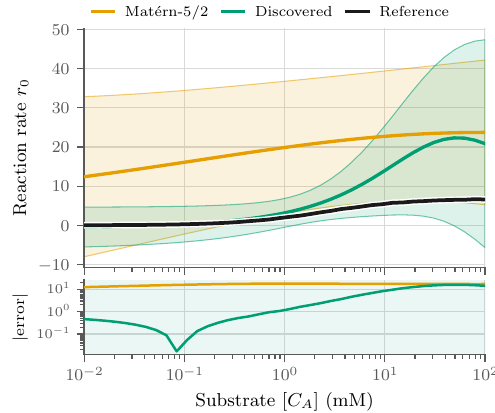}
\caption{Predicted reaction rate versus substrate concentration $[C_A]$ on the held-out ordered-bi-bi-kinetics mechanism. The discovered kernel (green) is closer to the true rate (black) at low concentrations. Both it and the fixed Mat\'ern-5/2 baseline (orange) overestimate the rate as concentration increases.}
\label{fig:chem-bibi-recovery}
\end{subfigure}
\caption{Local predictive behavior on two held-out biochemical mechanisms. In each panel, shading shows $\pm1.96$ predictive standard deviations. The lower strip shows the absolute error of each posterior mean, tinted where the discovered kernel is closer.}
\label{fig:chem-recovery}
\end{figure*}

The ordered bi-bi slice (Figure~\ref{fig:chem-bibi-recovery}) shows a similar local advantage at low concentrations, followed by overestimation as concentration increases. Both models have wide predictive intervals. These examples illustrate differences in posterior means and uncertainty, but do not by themselves establish interval calibration or recovery of the underlying rate law.

On the training-domain competitive-inhibition diagnostic, the discovered
kernel improves aggregate CRPS at the twenty-experiment budget, while its
posterior mean resolves inhibitor-level differences less clearly than the
dense oracle curves. This distinguishes predictive accuracy with sparse
observations from recovery of individual input effects. Mechanism-transfer
performance is evaluated separately on the five held-out mechanisms.

\subsection{Glucose Dynamics (GlucoseBench)}
\label{app:glucose}

GlucoseBench \citep{murphy2026mdagent} adapts the simglucose implementation \citep{xie2018simglucose} of the UVA/Padova type-1 diabetes simulator \citep{kovatchev2009silico}. Each of 30 virtual profiles, comprising 10 children, 10 adolescents, and 10 adults, has a hidden 13-dimensional nonlinear ordinary differential equation (ODE) state. An episode lasts six hours, the patient eats one meal and may receive one insulin bolus, and a continuous glucose monitor (CGM) reports noisy glucose every five minutes. The benchmark asks for forecasts of CGM after a 45-minute prefix of a new intervention for the same patient.

\paragraph{Inputs and predictive rule.}
A GP input is one CGM reading, encoded as five coordinates in $[0,1]$: reading time, meal size, meal start, bolus size, and bolus start, all read from the public delivered-input log. The last four are constant within an episode, so the kernel decides both how readings correlate over time and how whole trajectories share strength across interventions. For each forecast, amplitude and noise are fit by training likelihood on the full training episodes only, and the query prefix only conditions the posterior, as the benchmark protocol requires.

\paragraph{Protocol.}
\method has no experiment design, so every kernel sees the same six training episodes per patient: the two passive benchmark episodes and four distinct menu interventions drawn with a fixed per-patient seed. Search scores leave-one-training-episode-out forecasts on the ten child profiles, so every search target is a public training episode. Validation uses the same folds on the ten adolescent profiles. The validation-selected kernel is then frozen and tested on the ten adult profiles with the benchmark's own evaluator and sealed test interventions, twelve per patient. We report CGM normalized MSE. For query $q$ of patient $p$, let $\mathrm{MSE}_{pq}$ be the mean squared error of the forecast over the forecast window. Then
\begin{equation}
  \mathrm{nMSE}_{pq}=\frac{\mathrm{MSE}_{pq}}{\max\!\big(\mathrm{Var}(y^{\mathrm{passive}}_p),\,1\big)},
  \label{eq:glucose-nmse}
\end{equation}
where $\mathrm{Var}(y^{\mathrm{passive}}_p)$ is the population variance of the CGM readings in the first two passive training episodes of patient $p$. The benchmark takes the arithmetic mean of $\mathrm{nMSE}_{pq}$ over the queries of each patient, then the geometric mean across patients.

This protocol learns a kernel on some profiles and applies it to others. Forecasting remains within patient, since each test patient's GP is fit only on that patient's episodes, but kernel selection pools information across profiles, which the per-patient benchmark does not do. The model discovery agent of \citet{murphy2026mdagent} chooses its training interventions actively and fits per-patient ODE models, whereas our interventions are fixed, so its reported numbers are not directly comparable to ours.

\paragraph{References.}
The baseline kernel structures tuned across patients use the child profiles and the same 64-trial Halton budget as the discoveries, and all pass Tier~2. The per-patient marginal-likelihood baselines described below instead fit on each target patient. The relative-time grid applies RBF or Mat\'ern kernels to time, time since the meal, time since the bolus, meal size, and bolus size. Relative-time ARD-Mat\'ern-5/2 and ARD-RBF use the same five relative inputs with one lengthscale each. Raw-input ARD-Mat\'ern-5/2 uses the five raw coordinates. The physiological library is written by the authors. It combines gamma-shaped meal-absorption and insulin-action curves, their interaction, and a Mat\'ern residual. Two per-patient references fit ARD-Mat\'ern-5/2 by marginal likelihood on each patient's own training data, on raw and on relative inputs, which is standard practice and pools nothing across patients. We also report the fixed-kernel family winners selected on training data, the best fixed kernel in hindsight, and persistence (the last observed value).

\paragraph{Results.}
Figure~\ref{fig:glucose-heatmap} reports every method at four training budgets and four prefix lengths. Table~\ref{tab:glucose-paired} compares each discovery with the strongest reference, patient by patient. Both Astra campaigns have the same validation error ($0.203$), and validation CRPS picks campaign 1 ($0.204$ against $0.205$). Among the ensemble campaigns, validation ranks them in their test order. Across all 24 evaluated kernels, adolescent validation error predicts adult test error with Spearman $\rho=0.94$ ($p=10^{-11}$), close to the $0.95$ we find for ChemBench.

\begin{figure*}[t]
\centering
\includegraphics[width=0.95\textwidth]{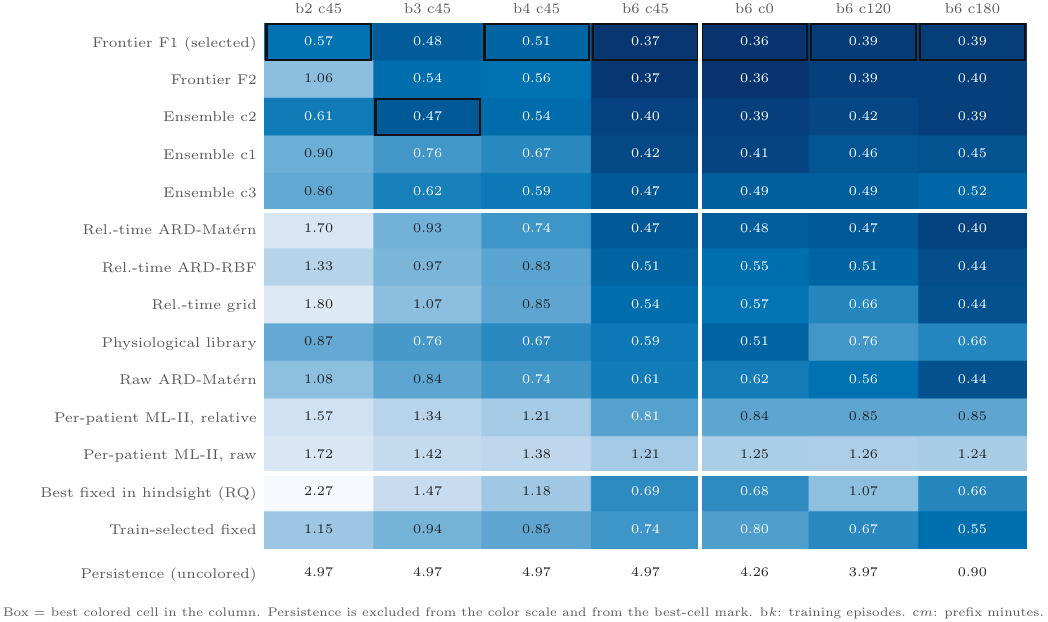}
\caption{\textbf{GlucoseBench adult test error for every method and setting.} CGM normalized MSE (arithmetic mean over queries within each patient, then geometric mean over 10 patients), colored on a log scale where darker is lower. In the column codes, \texttt{b}$k$ means $k$ training episodes and \texttt{c}$m$ means an $m$-minute observed prefix. Columns show \texttt{b2 c45}, \texttt{b3 c45}, \texttt{b4 c45}, and \texttt{b6 c45}, followed by prefix lengths of 0, 120, and 180 minutes at six training episodes. Rows are the discovered kernels, the tuned references, and the fixed-kernel and persistence baselines. Rel.-time uses time since the meal and bolus as inputs, and ML-II is per-patient marginal-likelihood fitting. A box marks the best colored cell in each column. Persistence is printed under the heatmap and is left off the color scale.}
\label{fig:glucose-heatmap}
\end{figure*}

\begin{table*}[t]
\caption{Paired per-patient comparison with the strongest reference, relative-time ARD-Mat\'ern. In each column, \texttt{b}$k$ denotes $k$ training episodes and \texttt{c}$m$ denotes an $m$-minute observed prefix. Each cell gives patients won out of 10, the raw two-sided exact Wilcoxon $p$-value, and the percentage change in geometric-mean nMSE. Negative percentages are lower error. A win count of 5 is an even split. Holm adjustment over these 35 cells leaves every adjusted $p$-value above $0.05$ (minimum $0.068$).}
\label{tab:glucose-paired}
\centering
\small
\setlength{\tabcolsep}{3.9pt}
\newcommand{\gcell}[2]{\shortstack[c]{#1\\#2}}
\begin{tabular}{@{}lccccccc@{}}
\toprule
Discovery & b2 c45 & b3 c45 & b4 c45 & b6 c0 & b6 c45 & b6 c120 & b6 c180 \\
\midrule
Frontier F1 & \gcell{10, .002}{$-67\%$} & \gcell{9, .004}{$-48\%$} & \gcell{9, .027}{$-31\%$} & \gcell{10, .002}{$-26\%$} & \gcell{10, .002}{$-23\%$} & \gcell{10, .002}{$-18\%$} & \gcell{6, 1.0}{$-3\%$} \\
Frontier F2 & \gcell{8, .014}{$-38\%$} & \gcell{10, .002}{$-42\%$} & \gcell{7, .027}{$-25\%$} & \gcell{10, .002}{$-26\%$} & \gcell{8, .020}{$-22\%$} & \gcell{9, .014}{$-17\%$} & \gcell{6, .92}{$-2\%$} \\
Ensemble c2 & \gcell{9, .004}{$-64\%$} & \gcell{10, .002}{$-50\%$} & \gcell{8, .037}{$-27\%$} & \gcell{8, .020}{$-20\%$} & \gcell{7, .049}{$-15\%$} & \gcell{8, .014}{$-12\%$} & \gcell{5, .62}{$-3\%$} \\
Ensemble c1 & \gcell{10, .002}{$-47\%$} & \gcell{6, .28}{$-19\%$} & \gcell{5, .77}{$-10\%$} & \gcell{8, .027}{$-16\%$} & \gcell{7, .23}{$-10\%$} & \gcell{7, .56}{$-3\%$} & \gcell{2, .020}{$+12\%$} \\
Ensemble c3 & \gcell{10, .002}{$-50\%$} & \gcell{9, .010}{$-33\%$} & \gcell{7, .19}{$-20\%$} & \gcell{5, 1.0}{$+2\%$} & \gcell{5, .77}{$+0\%$} & \gcell{5, .77}{$+4\%$} & \gcell{2, .037}{$+28\%$} \\
\bottomrule
\end{tabular}
\end{table*}

The relative-time representation explains much of the gain over the original fixed kernels, as the tuned relative-time ARD reference shows. The Astra discoveries still improve on it by about 22\% at the main setting (b6 c45) and at almost every other setting. With two training episodes, the five discoveries reduce geometric-mean nMSE by 38--67\% relative to this reference. With six episodes at a 45-minute prefix, the reductions are 0--23\%. The $p$-values in Table~\ref{tab:glucose-paired} are raw. After Holm correction over all 35 cells, none remain below $0.05$. With a 180-minute prefix, no discovery has a raw $p$-value below $0.05$ in the favorable direction. Ensemble c1 and c3 have higher error there (raw $p=0.020$ and $0.037$), and those two contrasts also fall above $0.05$ after the same Holm correction. Against the original fixed kernels, all three ensemble discoveries win on 8 to 10 of 10 patients ($p\le0.02$), with 31--46\% lower error. Against the per-patient ML-II references, every discovery wins on all 10 patients.

\paragraph{Robustness to the training interventions.}
We redraw the four random training interventions with two other seeds and also use the demonstration policy from the benchmark's README. Ensemble c2 remains the best method under all three alternatives (adult b6, normalized mean squared error (nMSE) at c45 and c0): seed 1 $0.370/0.367$ against $0.403/0.414$ for relative-time ARD-Mat\'ern, seed 2 $0.370/0.362$ against $0.429/0.435$, and the README policy $0.357/0.353$ against $0.377/0.385$.

\paragraph{Ablation.}
We remove feature groups from the frozen ensemble c2 kernel without retuning. Without the meal features, adult error rises from $0.402$ to $1.164$, and without the insulin features to $0.522$. Removing the interaction terms instead slightly improves adult error to $0.380$ (validation $0.252$ against $0.238$). The meal and insulin features carry the gain, and the interaction terms do not help.

\paragraph{What the best ensemble kernel encodes.}
Ensemble c2 is an explicit 16-feature program on time since the meal and time since the bolus. It uses the feature-map contract, so the interpreter forms $k(x,x')=\varphi(x)^\top\varphi(x')$. It does not pass a learned neural representation to a separate RBF or Mat\'ern head. After tuning, the meal gate is a logistic function of time since the meal. It is half open at 46 minutes, and it goes from 10\% to 90\% open over about 76 minutes (logistic scale 17 minutes). The insulin gate is zero before 23 minutes after the bolus. After that, it decays with a time constant of about 127 minutes. Coordinate sweeps (as in Appendix~\ref{app:kernel-descriptions}) give kernel correlations of $0.44$ for time, $0.83$ for meal size, $0.50$ for meal start, $0.94$ for bolus size, and $0.95$ for bolus start.

\begin{figure}[t]
\centering
\includegraphics[width=0.6\columnwidth]{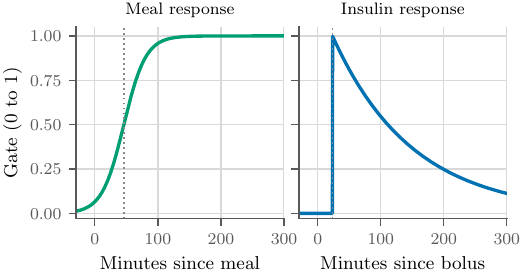}
\caption{\textbf{The selected glucose kernel encodes a delayed meal effect and a fading insulin effect.} The two gates of the ensemble campaign-2 feature map after tuning, as functions of time since the meal and since the bolus. Dotted lines mark the 46-minute midpoint of the meal gate and the 23-minute start of the insulin gate. After that start, insulin decays with a time constant of 127 minutes.}
\label{fig:glucose-learned}
\end{figure}

\paragraph{Calibration.}
At b6 c45, both Astra discoveries lower mean CRPS by about 15\% relative to relative-time ARD-Mat\'ern, win on all 10 adults ($p=0.002$), and are indistinguishable from each other ($p=0.92$). Their nominal 95\% intervals cover about 82\% of adult readings, compared with 76\% for that strongest tuned reference and 77--79\% for the ensemble discoveries. The Astra kernels therefore improve probabilistic accuracy and reduce undercoverage, but they remain below the nominal level. The physiological library is the exception at 93\%.

\begin{figure*}[t]
\centering
\includegraphics[width=0.85\textwidth]{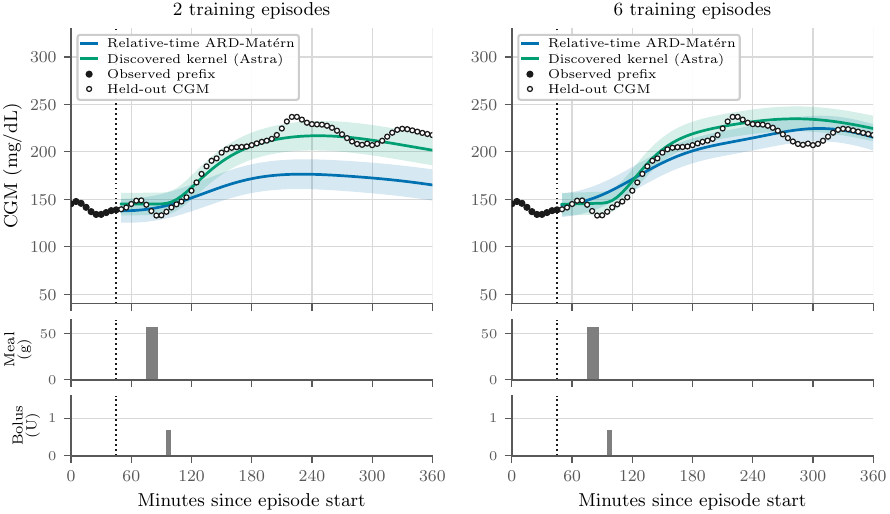}
\caption{\textbf{With two training episodes the reference undershoots the meal response, with six the gap closes.} One adult and one held-out query, forecast after the 45-minute prefix (dotted line) from 2 versus 6 training episodes. Curves are posterior means with 95\% intervals for the validation-selected Astra kernel (campaign 1) and relative-time ARD-Mat\'ern. Lower panels show the delivered meal and bolus. The patient is a median case for the ensemble kernel.}
\label{fig:glucose-before-after}
\end{figure*}

\section{Transfer and Interpretation of Discovered Kernels}
\label{app:cross-domain}

This appendix reports how validation-selected kernels transfer across domains and how we read what each discovered kernel encodes.

\subsection{Cross-Domain Transfer of Selected Kernels}
\label{app:transfer-full}

We test whether performance in one domain identifies kernels that remain useful in another. We transfer frozen BBO kernels to greenhouse-gas forecasting, forecasting kernels to held-out BBO tasks, and the validation-selected kernels from both domains to ChemBench. We also test the selected chemistry kernel in the other two domains. Only the outer amplitude and noise variance are refit, so every comparison measures reuse of a fixed kernel structure rather than another round of search.

\begin{figure}[t]
\centering
\includegraphics[width=0.6\columnwidth]{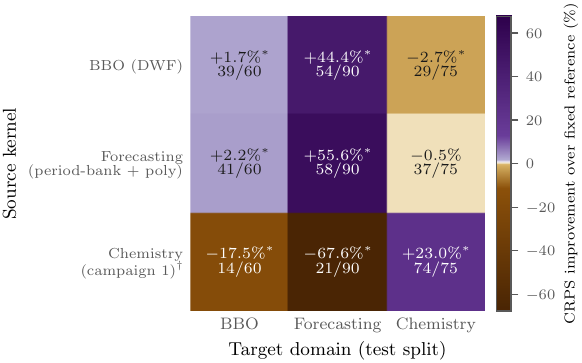}
\caption{\textbf{Cross-domain reuse of validation-selected kernels.} Rows are source kernels and columns are target domains. Each cell gives the target test CRPS change (\%) relative to that domain's fixed reference and the number of evaluation units won. Negative CRPS changes mean lower error and better performance. Purple favors the transferred kernel and brown favors the reference. Stars mark unadjusted paired tests with $p<0.05$. $^\dagger$The chemistry-source cells use the fixed coordinate conversion described below, so read them as interface-compatibility checks.}
\label{fig:transfer-matrix}
\end{figure}

Figure~\ref{fig:transfer-matrix} separates reuse from target-specific advantage. \dwf{} approaches the best fixed kernel chosen separately for each forecasting record, but it does not match the kernel discovered for forecasting. The forecasting selection yields a smaller improvement on BBO. Neither selection improves ChemBench, whose seven named inputs encode different scientific roles. The chemistry kernel can be made executable on BBO and forecasting, but it trails the fixed reference in both. Assigning arbitrary coordinates to roles such as substrate concentration does not preserve the biochemical geometry that helps on unseen enzyme mechanisms.

\paragraph{Fixed conversion for the chemistry kernel.}
The selected chemistry program directly indexes seven inputs with biochemical meanings, so it cannot run on the original BBO and forecasting inputs. For the exploratory cells in Figure~\ref{fig:transfer-matrix}, we copy up to the first seven target coordinates, fill missing coordinates with $0.5$, and ignore an eighth coordinate. The rule is fixed before evaluation and does not use outcomes. It produces no execution failures, but raises CRPS by 17.5\% on BBO and 67.6\% on forecasting relative to each domain's fixed reference. It wins 14 of 60 BBO tasks ($p=5.8\times10^{-7}$) and 21 of 90 forecast windows ($p=1.1\times10^{-5}$) in two-sided paired Wilcoxon tests. We treat these as compatibility checks rather than direct semantic transfer because the copied coordinates do not become biochemical variables.

\subsection{Natural-Language Descriptions of Discovered Kernels}
\label{app:kernel-descriptions}

The Automatic Statistician turns a kernel found by compositional search into a natural-language report \citep{duvenaud2013structure,lloyd2014automatic}. Each base kernel and operator in its grammar has a sentence template, so the report is correct by construction for any kernel the grammar can express. The connection between executable descriptions and kernel geometry also motivates \citet{hamzi2026algorithmic}. \method has no grammar to attach templates to, but every discovery run registers a formulation before it submits code (\S\ref{sec:functional-novelty}). The registration states the mathematical form, the closest known kernel, and the behavior the agent expects. It is written before tuning, however, so it describes what the agent intended rather than what the tuned kernel does. We therefore condense each registration into a short report and check every claim against the frozen program with the trusted interpreter.

Table~\ref{tab:kernel-descriptions} gives the reports for the headline discoveries. The second column describes the frozen kernel. The third states what tuning settled, which the registration could not know in advance.

\paragraph{Residual period-bank kernel.}
The name describes the selected forecasting program. Its frozen covariance is
\begin{equation}
 k_{\mathrm{PB}}(t,t')=k^{5/2}_{0.5}(t,t')+a+
 \sum_{m=1}^{4}\sum_{j=0}^{8}v_j\cos\!\big(2\pi f_{mj}(t-t')\big),
 \label{eq:period-bank-frozen}
\end{equation}
where $t$ is time normalized to $[0,1]$, $a=0.015625$,
$f_{mj}=m(f_c+(j-4)b/9)$, and
$v_j=\exp(-(j-4)^2/(2\sigma^2))$.
The frozen parameters are $f_c=0.02844$, $b\approx0.0162222222$, and $\sigma\approx2.630177515$.
The interpreter adds the feature inner product to the trusted Mat\'ern-$5/2$ kernel with lengthscale $0.5$.
Each sine--cosine feature pair gives one cosine term in Equation~\eqref{eq:period-bank-frozen}.
The feature labeled polynomial in the program is constant.
An outer covariance amplitude and observation noise are fitted separately under Equation~\eqref{eq:amp-noise}.

\begin{table*}[t]
\caption{Plain-language reports of the headline discoveries. Each report is condensed from the agent's registered formulation and checked against the frozen program with the trusted interpreter. The right column states what tuning settled.}
\label{tab:kernel-descriptions}
\centering
\small
\begin{tabular}{@{}L{0.14\textwidth}L{0.41\textwidth}L{0.39\textwidth}@{}}
\toprule
Kernel & Report & What tuning settled \\
\midrule
\dwf{} (BBO) & The objective varies smoothly, as under a Mat\'ern-5/2 prior, along a slightly rescaled copy of each coordinate. Values at points mirrored about the middle of a coordinate, $u$ and $1.025-u$, are encouraged but not forced to agree. The prior has a kink at the mirror line $u=0.513$. & Tuned geometry: the fold frequency settled at $\nu=0.975$, so the fold is a single tent peaking at $0.513$ that reflects each coordinate once, and the warp stays within $0.027$ of the identity (Appendix~\ref{app:dwf-record}). The registration had anticipated multi-scale periodicity and boundary sensitivity, which the tuned kernel does not use. The kink is Proposition~\ref{prop:dwf-cusp}. \\
\addlinespace
Residual period-bank + poly (forecasting) & The record follows a smooth trend, as under a Mat\'ern-5/2 prior with lengthscale $0.5$ of the normalized span, plus a constant offset, plus a second smooth component built from 36 sinusoids. Every sinusoid has a period between 7 and 47 times the length of the record, so within one record this component acts as a record-wide smooth trend rather than a cycle. & Tuned behavior: the sinusoid frequencies settled at $0.021$--$0.143$ cycles per normalized span, so the bank models slow curvature across the whole record. The registration described an annual cycle, but it took the annual period (about $0.03$ of the span) as a frequency. An annual cycle would need 25--47 cycles per span. The ``poly'' feature is a constant.  \\
\addlinespace
ChemBench, ensemble E1 (enzyme kinetics) & The rate depends smoothly on a monotone transform of log substrate concentration that is steepest at high concentration, the same transform of the second substrate, their product, and log enzyme loading. Inhibitor and product concentrations matter only in proportion to substrate saturation. & Tuned relevance: as in ARD, tuning decides which inputs matter. Enzyme loading and the two substrates matter most, while inhibitor and product matter little (Figure~\ref{fig:chem-slices}), and the saturation midpoint lies above the input range (Figure~\ref{fig:chem-features}). The program includes Arrhenius and pH-window features, but tuning set their weights so that the correlation is $1.000000$ at every base point. These probes show little sensitivity to temperature and pH, without establishing exact invariance. \\
\bottomrule
\end{tabular}
\end{table*}

\begin{table*}[t]
\ContinuedFloat
\caption{Plain-language reports of the headline discoveries, continued.}
\vspace{4pt}
\centering
\small
\begin{tabular}{@{}L{0.14\textwidth}L{0.41\textwidth}L{0.39\textwidth}@{}}
\toprule
Kernel & Report & What tuning settled \\
\midrule
ChemBench, frontier F2 (enzyme kinetics) & An explicit 248-feature map treats the rate as a linear combination of ten substrate-response shapes modeled on the training mechanisms, each at three substrate scales, their log-substrate sensitivities, and a small linear backbone. The response features are multiplied by enzyme, temperature, and pH factors, so posterior coefficients can select both mechanism shape and environmental modulation. & Tuned relevance: enzyme loading and substrate dominate, product, inhibitor, and temperature matter less, and pH is nearly ignored (Figure~\ref{fig:chem-slices}). The frozen kernel therefore retains strong enzyme and substrate effects, some product and inhibitor response, and almost no pH dependence. \\
\addlinespace
GlucoseBench, frontier F1 & An ARD-Mat\'ern-5/2 residual is augmented by 36 intervention features. At three response scales, gamma envelopes encode causal meal and bolus transients, long-period wave packets encode response age, and cumulative channels retain dose exposure after the transients fade. Signed and complementary channels allow meal and bolus responses to oppose or covary. & Tuned behavior: signed channels dominate because their mixing parameter is $\rho=0.20$, while cumulative memory remains active with gain $0.56$. The base kernel's time lengthscale is $3.21$ normalized spans and its four intervention lengthscales are $10^6$. On public training inputs, endpoint-sweep correlations are $0.90$ for time, $0.94$ for meal size, $0.96$ for meal start, and $0.98$ for bolus size and start. \\
\bottomrule
\end{tabular}
\end{table*}

\subsection{Coordinate-Sweep Diagnostics}
\label{app:coordinate-sweeps}
We inspect the frozen covariance by drawing 100 base points uniformly from the input box. For each coordinate, we compare the kernel correlation $k(x,x')/\sqrt{k(x,x)k(x',x')}$ after setting that coordinate to $0$ and to $v\in[0,1]$, with the others fixed. Figure~\ref{fig:chem-slices} plots the mean correlation along each sweep. This normalization is standard for covariance kernels \citep[Equation~4.35]{rasmussen2006gaussian}. For positive diagonal values, the correlation is the inner product of unit-normalized feature vectors, whose squared distance is twice one minus the correlation. Values near $1$ therefore indicate similar normalized features at the sampled input pairs. They do not rule out changes in marginal variance or prove invariance over the whole domain. We use the sweep as a diagnostic of kernel geometry, not as a partial-dependence plot of the fitted predictor.

\begin{figure}[t]
\centering
\includegraphics[width=\textwidth]{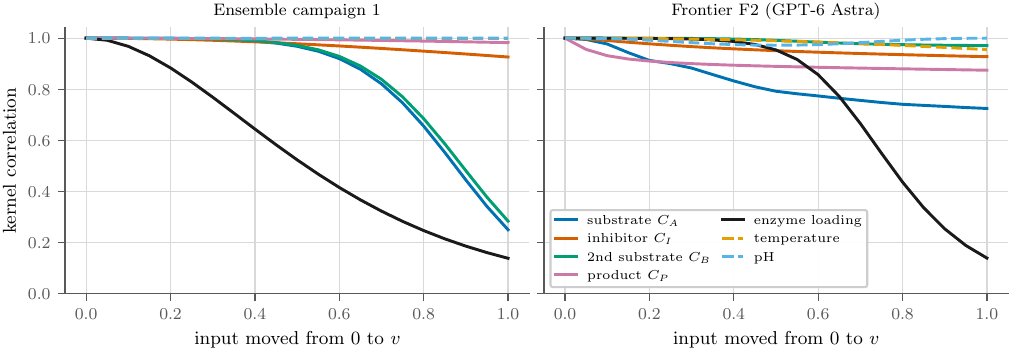}
\caption{\textbf{The two ChemBench kernels respond to different inputs.} Mean kernel correlation as one normalized input moves from $0$ to $v$ with the others fixed, over 100 base points. A curve near $1$ indicates similar normalized kernel features along the sampled sweeps. Left: ensemble campaign E1 depends mainly on enzyme loading and the two substrates, responds little to inhibitor and product, and shows almost no temperature or pH sensitivity in these sweeps. Right: frontier campaign F2, which uses GPT-6 Astra, depends on enzyme loading and substrate, responds mildly to the product, inhibitor, and temperature, and shows almost no pH sensitivity in these sweeps.}
\label{fig:chem-slices}
\end{figure}

\section{Scope and Future Extensions}
\label{app:extensions}
\method separates what an agent may propose from how a proposal becomes a valid kernel. That separation does not depend on the benchmark, the kernel method, or the problem scale. This section states the scope of the current experiments and describes the extensions that the framework supports without changes to its search loop.

\paragraph{Broader problem settings.}
The BBO benchmark has at most eight observed input dimensions and $n=8+2d$ observations per predictive episode. This scale represents the small-data regime in which kernel choice matters most. Natural next benchmarks include higher-dimensional optimization ($d\ge20$), multi-objective and constrained optimization, and mixed categorical--continuous spaces. These settings do not require a new validity argument.

Some discoveries already support arbitrary dimensions. For example, \dwf{} transforms each coordinate separately, so its frozen program runs unchanged as $d$ increases. However, the live-search audit identifies a limitation in the current interface. Six programs sized their parameter arrays for the benchmark's eight coordinates and raised an error at $d=16$ (Appendix~\ref{app:screening-stress}). The contracts make these failures explicit. Exposing $d$ to every component, as the spectral interface already does, would prevent this implementation error.

The contracts can also support larger datasets. Feature-map and spectral contracts produce explicit finite features, which reduce exact inference from $O(n^3)$ to $O(nm^2)$ for $m$ features. Inducing-point approximations \citep{williams2001nystrom} could provide similar savings for pullback and closure contracts.

We could also test whether proposers use variable semantics or act only as numerical search procedures. The intervention would replace informative feature names with arbitrary identifiers while keeping the values, splits, and budgets fixed. A mixed-type contract could combine learned transformations of continuous variables with PSD similarity matrices for categorical variables. Comparisons with ARD GPs and tree ensembles would measure predictive accuracy under matched training. They would also test whether any gain disappears after feature anonymization. Applying the same anonymization to enzyme kinetics and glucose would separate structures recovered from data from structures supplied by the proposer's prior knowledge.

\paragraph{New construction contracts.}
Each contract pairs a component interface with an assembly rule and a PSD proof, so new contracts extend the program space without touching the search loop. Five extensions follow the same pattern.
\begin{itemize}[leftmargin=*,itemsep=1pt,topsep=2pt]
\item \emph{Multi-output kernels.} An intrinsic coregionalization contract \citep{alvarez2012kernels} would let the agent supply a factor $L$ and a scalar kernel, with the interpreter assembling $W=LL^\top$ and $k((x,i),(y,j))=W_{ij}\,k(x,y)$. The result is PSD by the Schur product theorem.
\item \emph{Structured inputs.} The feature-map contract admits embeddings of sequences and graphs \citep{gartner2003survey}. Benchmarks on sequences and molecules would test such inputs directly.
\item \emph{Continuous spectral measures.} The spectral contract certifies finite nonnegative spectra, which covers random-feature approximations of continuous measures \citep{rahimi2007random}. A trusted library of closed-form spectral densities, such as Gaussian or Cauchy mixtures \citep{wilson2013spectralmixture}, would certify continuous measures exactly. The interpreter would map raw parameters to nonnegative weights, as it already does for spectral weights.
\item \emph{State-space kernels from linear ODEs.} Mechanistic model discovery often proposes compartmental ODEs \citep{murphy2026mdagent}. A linear ODE driven by white noise defines a GP kernel. For $dz=Fz\,dt+L\,dW$ with stable $F$ and output $y=Hz$, the stationary covariance is $k(\tau)=HP_\infty e^{F^\top\tau}H^\top$ for $\tau\ge0$. Here, $P_\infty$ solves the Lyapunov equation $FP_\infty+P_\infty F^\top+LL^\top=0$ \citep{sarkka2013spatiotemporal}. In our framework, the agent would write the matrices $(F_\theta,L_\theta,H_\theta)$ as functions of $\theta$. The agent would choose the latent dimension and the coupling pattern. The interpreter would check that $F_\theta$ is stable, solve the Lyapunov equation, and assemble $k$. The resulting kernel is valid because it is the covariance of a well-defined Gaussian process. The latent exponentially generated parameterization of \citet{loper2021general} is one suitable choice. Known inputs, such as meals and insulin doses, would enter through the mean as in latent force models \citep{alvarez2009latent}.
\item \emph{Recurrent feature maps.} The feature-map contract already admits a program that runs a fixed recurrence or ODE solver over the event history of one input, such as the meals and boluses of one glucose episode, and returns the state as its feature vector. The kernel is $\varphi_\theta(x)^\top\varphi_\theta(y)$, so validity follows from the feature-map case of Proposition~\ref{prop:validity}, provided that the state depends only on that input and $\theta$. The frontier glucose kernel F1 already retains dose history through closed-form cumulative-exposure channels (Table~\ref{tab:kernel-descriptions}), and a recurrence would generalize them. Nonlinear latent dynamics and structural causal world models, as in model discovery agents \citep{murphy2026mdagent}, would go further. They would change the search object, the validity argument, and the evaluation target, so we leave them as a separate direction.
\end{itemize}

\paragraph{Other kernel methods and modern applications.}
We evaluate discovered kernels as GP covariances, but any method that needs only a PSD kernel can use them, including kernel ridge regression, support vector machines, and kernel two-sample tests. Two applications are especially promising. In uncertainty quantification for language models, Kernel Language Entropy builds a PSD kernel over the semantic similarity of sampled generations \citep{nikitin2024kernel}. Certified construction would let an agent search for semantic kernels while keeping the entropy well defined. For scientific surrogates, the enzyme-kinetics and glucose studies of Appendix~\ref{app:extra-domains} show that discovered geometries can encode domain-specific structure. Molecular kernels \citep{griffiths2023gauche} and kernel methods for differential equations \citep{chen2021solving} are natural next domains, where contracts could also build in boundary conditions or conservation laws.

\paragraph{Selection within open-ended archives.}
Our open-ended search returns an archive rather than a single model, so the selection rule affects the reported result. The experiments select one frozen program by validation CRPS. They do not use a test-time Pareto frontier or average archive members. The observed results support this rule. The kernel selected on validation data ranks first among 58 frozen candidates on both unseen CFC records. By contrast, the kernel that would have been selected after viewing SF$_6$ test performance does not transfer to the two CFC records (Appendix~\ref{app:ts-extra}). Validation and test CRPS also have Spearman $\rho=0.95$ across the chemistry campaigns. The archives may contain complementary candidates, but predictive mixtures are outside the reported evaluation.

\paragraph{Protocol integrity and reward hacking.}
The current design limits the ways a proposer can exploit its evaluator. The backend computes every Gram matrix and fit statistic, and search-time scores use meta-training episodes only. Every reported program is frozen before testing. The same replay safeguard applies to every selected program, and \dwf{} is the fully audited example that reproduces every validation and test score exactly (Appendix~\ref{app:revision-saved}). Selection on validation episodes also guards against overfitting to the search objective. A systematic study is a valuable next step. Adversarial proposers could be instructed to target the evaluator, and rotating probe sets or held-out evaluators could measure how quickly search-time fitness and held-out performance diverge. A matched campaign with the novelty penalty set to zero would isolate its effect and complement the archive ablation of Appendix~\ref{app:search-ablation}.

\section{LLM Prompts}
\label{app:prompts}
\label{app:agent-prompts}

This appendix documents the agent interface and reproduces the prompts. The agent receives a system instruction, a domain task description, the structured tool definitions, and an evolutionary assignment. The assignment appends the iteration, search mode, target niche, parent, and inspirations to the task description. Table~\ref{tab:formulation-fields} specifies the required formulation fields.

\subsection{Agent Interface}
\label{app:prompt-interface}

\paragraph{Conversation format and examples.}
The initial conversation contains one user message, which comprises the task and the assignment. There is no fixed few-shot bank of solved kernels and no scripted assistant demonstrations. The task descriptions contain interface examples and the input conventions of each domain, which are reproduced below. For later runs, the parent payload contains its candidate identifier (ID), name, contract, selection score, source, and parameters. Each inspiration contains its ID, name, niche, score, and mathematical form. These archive-derived examples vary with the campaign history. Within a run, assistant tool calls and backend responses remain in the conversation. A root run has no parent payload, and its inspirations are empty only while the archive is empty.

\paragraph{Tool calls.}
Tools take JavaScript Object Notation (JSON) objects with the names, required fields, and allowed values declared in their schemas. Formulation fields are strings. Staging supplies the candidate name, contract, Python source, and formulation ID, with optional initial parameters, a parameter space, and parents. A tunable parameter declares a float, integer, or categorical type, with its range or choices and an optional linear or logarithmic scale. Candidate code must implement the exact entry point and output shape in the system prompt. The trusted tools return verification and evaluation results, so the agent supplies code and tool arguments rather than a claimed numeric result.

\paragraph{Closure-tree serialization.}
For the closure contract, \path{closure_tree(parameters)} encodes sums, products, and nonnegative scalar multiplication as nested dictionaries. A base node has \path{op="base"}, a library \path{kind}, and its numeric parameters. A sum or product has \path{op="sum"} or \path{op="product"} and a \path{children} list of at least two nodes. A scaling node has \path{op="scale"}, a finite nonnegative \path{weight}, and one \path{child}. For example, the following is an illustrative serialization of a library-kernel sum, not a few-shot example supplied to the proposer:
\begin{CodeBox}{Closure-tree serialization}
\begin{lstlisting}[style=kjson]
{"op": "sum", "children": [
  {"op": "base", "kind": "matern52", "lengthscale": 0.5},
  {"op": "scale", "weight": 0.2,
   "child": {"op": "base", "kind": "linear"}}]}
\end{lstlisting}
\end{CodeBox}
The documented library leaves include linear, RBF, Mat\'ern-3/2, Mat\'ern-5/2, rational-quadratic, periodic, and spectral-mixture kernels. Base variance and scaling weights are nonnegative. Lengthscales and the rational-quadratic shape or periodic period are positive. The discovery task's restriction on ordinary closure-library rediscoveries is retained verbatim in the task descriptions below.

\subsection{Shared Prompts}
\label{app:prompt-shared}

The software name in the opening line of the system instruction is normalized to [Kernaut]. The interface and workflow agree with the controller implementation.

\begin{AIbox}{Kernaut agent: system instruction}
\begin{lstlisting}[style=kprompt]
You are [Kernaut]'s synthesis controller. Design executable Gaussian-process
kernel programs and improve them using evidence from deterministic tools. You reason and write
candidate source, but you never claim to have executed code, computed a Gram matrix, fitted a GP,
or established a verification tier yourself.

Every proposal must use exactly one contract and entry point:
- feature_map: `feature_point(x, parameters) -> array[m]`, where `x` is one immutable point;
  the trusted interpreter calls it independently for every input and constructs `Phi @ Phi.T`
- residual_feature_map: the same `feature_point` interface; the interpreter adds a trusted base
  kernel plus `residual_weight**2 * Phi @ Phi.T`
- additive_feature_map: `coordinate_feature(value, parameters) -> array[m]`; the interpreter
  averages its feature Gram matrix over input coordinates
- spectral: `spectral_components(input_dimension, parameters)` returns
  `(frequencies[m, input_dimension], raw_weights[m])`
- input_transform: `transform_point(x, parameters) -> array[q]`, where `x` is one immutable
  point; the trusted interpreter applies `parameters.base_kernel` by pullback
- residual_input_transform: the same `transform_point` interface; the interpreter adds a trusted
  kernel on the original inputs to a nonnegative weighted pullback kernel
- closure: `closure_tree(parameters) -> trusted closure-tree dictionary`
- unverified: `kernel_matrix(x, parameters) -> array[n, n]` for empirical diagnostics only

Contract names are not function names. In particular, an `input_transform` candidate MUST define
`transform_point(x, parameters)` (not `input_transform`). It receives exactly one 1-D point and
must return exactly one nonempty 1-D transformed point. Do not reshape it into a batch.

Certified candidate code must be deterministic and free of module-level state; it may import only
numpy or math and may not use random-number generators. Submit a candidate, verify it, and evaluate
it only after Tier 2 acceptance. Once a candidate passes Tier 2, evaluate it before proposing any
revision. Inspect structured failures and revise deliberately. A revision must change kernel
behavior, not comments, docstrings, formatting, or parameters alone. In mathematical-novelty mode,
test at least two substantially different formulations rather than spending the campaign on one
lineage. Independent formulations should be submitted as roots without parents. Supply parent IDs
only for genuine revisions whose behavior descends from those candidates; changing contract or
topic does not by itself establish lineage. Use archive queries to avoid duplicates and compare the
quality/cost frontier. Unverified candidates can be evaluated after Tier 1 but are never accepted.
In mathematical-novelty mode, prefer the backend workflow: propose a formulation, stage source and
a bounded parameter space, verify the draft, optimize its parameters, submit the selected immutable
trial, then evaluate it. A parameter is learnable only when declared in that tuning space. Analyze
known equivalences, normalized-domain cycles or scales, irrelevant-coordinate and cross-coordinate
mechanisms, a falsification test, feature growth with dimension, and expected conditioning. End
with a concise report naming the best candidate IDs and remaining limitations.
\end{lstlisting}
\end{AIbox}

The assignment below is a template. Angle brackets mark values that are filled in at run time. A root run receives the root instruction and an empty parent, and its inspirations are empty only while the archive is empty. A revision run receives the revision instruction, and its parent and inspiration payloads are serialized JSON objects with the fields listed above.

\begin{AIbox}{Kernaut agent: evolutionary assignment (template)}
\begin{lstlisting}[style=kprompt]
# Evolutionary search assignment

Iteration: <iteration index>
Search mode: <root | revision>
Target niche: <niche>
<root mode:> Create an independent root and leave parents empty.
<revision mode:> Make a genuine behavioral revision of parent <candidate ID> and cite exactly that
  parent unless you truly recombine another candidate.
Inspirations are context only and are not parents unless their mechanism is actually incorporated.
  Follow the formulation -> stage -> verify -> optimize parameters -> submit staged candidate ->
  evaluate workflow. Complete one candidate, then stop.

Parent:
<null in root mode | parent payload: candidate ID, name, contract, selection score, source, parameters>

Inspirations:
<[] | up to k_insp payloads: ID, name, niche, score, mathematical form>
\end{lstlisting}
\end{AIbox}

\subsection{Domain Task Descriptions}
\label{app:prompt-tasks}

Each box shows the task portion of the opening user message. A task description gives the domain, the input conventions, and the required interface. It does not describe the evaluation data or the scoring. The backend selection score, including its regret and novelty terms, is specified in Equation~\eqref{eq:selection-score}, and the interface convention for spectral frequencies is clarified in Appendix~\ref{app:impl}.

\begin{AIbox}{BBO task description}
\begin{lstlisting}[style=kprompt]
# Procedural meta-BO kernel discovery

Design a single dimension-agnostic Gaussian-process kernel that generalizes across a distribution
of normalized, low-to-moderate-dimensional Bayesian-optimization tasks. Training episodes include
smooth anisotropic, periodic, multimodal, and rugged functions under hidden coordinate
permutations, reflections, monotone input warpings, invertible cross-coordinate couplings, and a
small number of irrelevant coordinates.

Each task standardizes its outputs and fits the same small grid of outer covariance amplitudes
and diagonal noise values.

Prefer Tier-2 candidates that:

- accept any input dimension rather than hard-coding a two-dimensional feature map;
- express useful inductive biases beyond renaming standard RBF or Matérn kernels;
- remain numerically stable for input dimensions 2, 5, 6, and 8;
- balance smooth global structure with periodic or localized variation; and
- use parent lineage when making a genuine revision, and leave parents empty for an independent
  root formulation.

This is a mathematical-novelty campaign. Closure trees are disallowed for discovered candidates:
ordinary sums, products, and rescalings of RBF/Matérn kernels belong in the reference bank, not the
search space. Start each idea by calling `propose_formulation`. State the exact mathematical form,
why it is PSD, its closest known kernel and concrete difference, and the BO behavior it should
improve. Only then implement it and submit the returned formulation ID.

Prefer the `input_transform` contract when inventing geometry. Its required entry point is exactly
`transform_point(x, parameters)`: it receives one immutable 1-D point and must return one nonempty
1-D transformed point. Do not define a function named `input_transform`, and do not reshape the
point into a batch. The trusted runtime stacks those vectors and evaluates
`parameters["base_kernel"]` on the transformed inputs. For example, the base specification is
`{"kind": "matern52", "lengthscale": 0.5}`. The transform must be dimension-agnostic and independent
of other observations. A spectral candidate defines
`spectral_components(input_dimension, parameters)` so its returned frequencies can always have
shape `[m, input_dimension]`. Pointwise, residual, additive, and genuinely new spectral
constructions are also
permitted. Independent formulations are root candidates and must not cite unrelated discoveries as
parents. A genuine revision must cite its actual parent, and a parameter-only edit of a parent is
rejected.

Candidate-specific parameters are fixed unless they are declared in a staged candidate's tuning
space and optimized by the backend. Do not call fixed constants learnable. On normalized `[0, 1]`
inputs express frequencies in cycles, for example `sin(2*pi*f*x)`. Prefer the formulation, staging,
verification, parameter-optimization, immutable-submission, and full-evaluation tool sequence.

Explore at least two substantially different constructions. Inspect the archive before proposing a
revision, and do not treat a syntactic rewrite or parameter-only change as a novel kernel.
\end{lstlisting}
\end{AIbox}

\begin{AIbox}{Forecasting task description}
\begin{lstlisting}[style=kprompt]
# Time-series forecasting kernel discovery

Design a single dimension-agnostic Gaussian-process kernel that acts as a transferable prior for
forecasting long monthly time series. Each episode standardizes a training window of one record
and forecasts a fixed horizon of future months. Inputs are 1-D values in [0, 1] (the month index
scaled to the record), and observation noise is small. The kernel should also remain numerically
stable when evaluated at input dimensions 1 through 8.

This is a mathematical-novelty campaign. Discovered closure trees are rejected because ordinary
sums and products of standard kernels belong in the baseline bank. Start every candidate with
`propose_formulation`, giving its exact mathematical form, PSD argument, closest known kernel,
behavioral distinction, and a falsification test. Then stage, verify, tune, and submit it through the
normal immutable-candidate workflow.

Prefer `input_transform`, `feature_map`, `residual_input_transform`, or genuinely new `spectral`
constructions. The trusted interpreter constructs the Gram matrix; candidate code should only
return the certified pointwise transform, features, or spectral components required by its
contract. The kernel must accept any input dimension. Independent ideas have no
parents, while a real revision must cite the candidate it changes.

Develop at least two structurally different proposals before refining the strongest one.
\end{lstlisting}
\end{AIbox}

\begin{AIbox}{Chemistry task description}
\begin{lstlisting}[style=kprompt]
# Enzyme-kinetics kernel discovery

Design one certified Gaussian-process kernel that acts as a transferable prior for learning
enzyme-catalyzed reaction rates from scattered experiments. Each dataset records the initial
reaction rate `r_0` measured under a set of experimental conditions. The underlying rate law
differs between datasets and is not given.

Every input is a seven-dimensional vector in `[0, 1]`, with coordinates in this order:

1. log substrate concentration `C_A`;
2. inhibitor concentration `C_I` (linear, can be zero);
3. log second-substrate concentration `C_B`;
4. product concentration `C_P` (linear, can be zero);
5. log enzyme loading `Enz`;
6. temperature `T` (linear, 278-368 K); and
7. pH (linear, 4-10).

This is a mathematical-novelty campaign. Discovered closure trees are rejected because ordinary
sums and products of standard kernels belong in the baseline bank. Start every candidate with
`propose_formulation`, giving its exact mathematical form, PSD argument, closest known kernel,
behavioral distinction, and a falsification test. Then stage, verify, tune, and submit it through the
normal immutable-candidate workflow.

Prefer `input_transform`, `feature_map`, `residual_input_transform`, or genuinely new `spectral`
constructions. The trusted interpreter constructs the Gram matrix; candidate code should only
return the certified pointwise transform, features, or spectral components required by its
contract. The input dimension is fixed at seven for this benchmark. Independent ideas have no
parents, while a real revision must cite the candidate it changes.

Develop at least two structurally different proposals before refining the strongest one. Do not
claim that a covariance is itself a recovered rate law.
\end{lstlisting}
\end{AIbox}

\begin{AIbox}{Glucose task description}
\begin{lstlisting}[style=kprompt]
# CGM forecasting kernel discovery

Design one certified Gaussian-process kernel that acts as a transferable prior for forecasting a
simulated type-1 diabetes patient's continuous glucose monitor (CGM) response to a meal and an
insulin bolus. Each patient runs six-hour episodes. In each episode the patient eats one meal and
may receive one insulin bolus, and CGM is read every five minutes. Every episode of a patient
starts from the same physiological state, so episodes differ only through their intervention and
sensor noise.

Every GP input is one CGM reading, a five-dimensional vector in `[0, 1]` with coordinates in
this order:

1. reading time, minutes divided by 360;
2. meal size, grams divided by 60;
3. meal start time, minutes divided by 360;
4. bolus size, insulin units divided by 1.5 (zero means no bolus); and
5. bolus start time, minutes divided by 360 (equal to the meal start when there is no bolus).

Coordinates 2-5 are constant within an episode, so the kernel decides both how readings
correlate over time and how whole trajectories share strength across interventions.

This is a mathematical-novelty campaign. Discovered closure trees are rejected because ordinary
sums and products of standard kernels belong in the baseline bank. Start every candidate with
`propose_formulation`, giving its exact mathematical form, PSD argument, closest known kernel,
behavioral distinction, and a falsification test. Then stage, verify, tune, and submit it through the
normal immutable-candidate workflow.

Prefer `input_transform`, `feature_map`, `residual_input_transform`, or genuinely new `spectral`
constructions. The trusted interpreter constructs the Gram matrix; candidate code should only
return the certified pointwise transform, features, or spectral components required by its
contract. The input dimension is fixed at five for this benchmark. Independent ideas have no
parents, while a real revision must cite the candidate it changes.

Develop at least two structurally different proposals before refining the strongest one. Do not
claim that a covariance is itself a recovered physiological model.
\end{lstlisting}
\end{AIbox}

\end{document}